\documentclass{article}

\usepackage[preprint]{neurips_2023}
\usepackage{xcolor}
\usepackage{microtype}
\usepackage{graphicx}
\usepackage{booktabs} 
\usepackage[utf8]{inputenc}
\usepackage{graphicx}
\usepackage{amsmath}
\DeclareMathOperator*{\argmin}{argmin}
\usepackage{amsthm}
\usepackage{amssymb}
\newtheorem{theorem}{Theorem}

\newtheorem{definition}{Definition}
\usepackage{mdframed}
\usepackage{lipsum}
\usepackage{url}
\usepackage{caption}
\usepackage{multicol}
\usepackage{multirow}
\usepackage{wrapfig}
\usepackage{caption}
\usepackage{subcaption}

\newmdtheoremenv{theo}{Theorem}
\newsavebox{\savepar}

\usepackage{bm}
\usepackage{enumitem}
\setitemize{noitemsep,topsep=0pt,parsep=0pt,partopsep=0pt}
\usepackage{hyperref}

\usepackage[ruled, vlined]{algorithm2e}

\title{Locally Invariant Explanations: Towards Stable and Unidirectional Explanations through Local Invariant Learning}

\author{%
  Amit Dhurandhar\thanks{Equal contribution} \\
  IBM Research\\
  Yorktown Heights, USA \\
  \texttt{adhuran@us.ibm.com} \\
  \And
  Karthikeyan Natesan Ramamurthy$^*$ \\
  IBM Research\\
  Yorktown Heights, USA \\
  \texttt{knatesa@us.ibm.com} \\
  \And
  Kartik Ahuja\\
  Mila\\
  Montreal, Canada \\
  \texttt{kartik.ahuja@mila.quebec} \\
  \And
  Vijay Arya\\
  IBM Research\\
  Bangalore, India\\
  \texttt{vijay.arya@in.ibm.com} \\
}

\begin{document}

\maketitle

\begin{abstract}
Locally interpretable model agnostic explanations (LIME) method is one of the most popular methods used to explain black-box models at a per example level. Although many variants have been proposed, few provide a simple way to produce high fidelity explanations that are also stable and intuitive. In this work, we provide a novel perspective by proposing a model agnostic local explanation method inspired by the invariant risk minimization (IRM) principle -- originally proposed for (global) out-of-distribution generalization -- to provide such high fidelity explanations that are also stable and unidirectional across nearby examples. Our method is based on a game theoretic formulation where we theoretically show that our approach has a strong tendency to eliminate features where the gradient of the black-box function abruptly changes sign in the locality of the example we want to explain, while in other cases it is more careful and will choose a more conservative (feature) attribution, a behavior which can be highly desirable for recourse. Empirically, we show on tabular, image and text data that the quality of our explanations with neighborhoods formed using random perturbations are much better than LIME and in some cases even comparable to other methods that use realistic neighbors sampled from the data manifold. This is desirable given that learning a manifold to either create realistic neighbors or to project explanations is typically expensive or may even be impossible. Moreover, our algorithm is simple and efficient to train, and can ascertain stable input features for local decisions of a black-box without access to side information such as a (partial) causal graph as has been seen in some recent works.
\end{abstract}

\section{Introduction}

Deployment and usage of neural black-box models has significantly grown in industry over the last few years creating the need for new tools to help users understand and trust models \citep{xai}. Even well-studied application domains such as image recognition require some form of prediction understanding in order for the user to incorporate the model into important decisions \citep{saliency,LRPTOOLBOX}. An example of this could be a doctor who is advised by a model of a positive cancer diagnosis based on an image scan. Since the doctor holds responsibility for the final diagnosis, the model must provide sufficient reason for its prediction. Even new text categorization tasks \citep{inputreduction} are becoming important with the growing need for social media companies to provide better monitoring of public content. Twitter was monitoring tweets related to COVID-19 in order to label tweets containing misleading information, disputed claims, or unverified claims \citep{twitter_covid19}. Laws will likely emerge requiring explanations for why red flags were or were not raised in many examples. In fact, the General Data Protection and Regulation (GDPR) \citep{gdpr} act passed in Europe already requires automated systems that make decisions affecting humans to be able to explain them. Given this acute need, a number of methods have been proposed to explain local decisions (i.e. example specific decisions) of classifiers \citep{lime,unifiedPI,saliency,LRPTOOLBOX,CEM}. Locally interpretable model-agnostic explanations (LIME) is arguably the most well-known local explanation method that requires only query (or black-box) access to the model. Although LIME is a popular method, it is known to be sensitive to certain design choices such as i) (random) sampling to create the \textit{(perturbation) neighborhood}\footnote{By \textit{perturbation neighborhood} or simply \textit{neighborhood}, we mean neighborhoods generated for local explanations. By \textit{exemplar neighborhood}, we mean closest in dataset examples.}, ii) the size of this neighborhood (number of samples) and iii) (local) fitting procedure to learn the explanation model \citep{molnarbook,stabzhang}. The first, most serious issue could lead to nearby examples having drastically different explanations making effective recourse a challenge. One possible mitigation is to increase the neighborhood size but one cannot arbitrarily do so as it not only leads to higher computational cost, but also in today's cloud computing-driven world it could have direct monetary implications where every query to a black-box model has an associated cost \citep{macem}.
\begin{wrapfigure}{r}{0.4\textwidth}
\vspace{-3mm}
\includegraphics[width=0.35\textwidth]{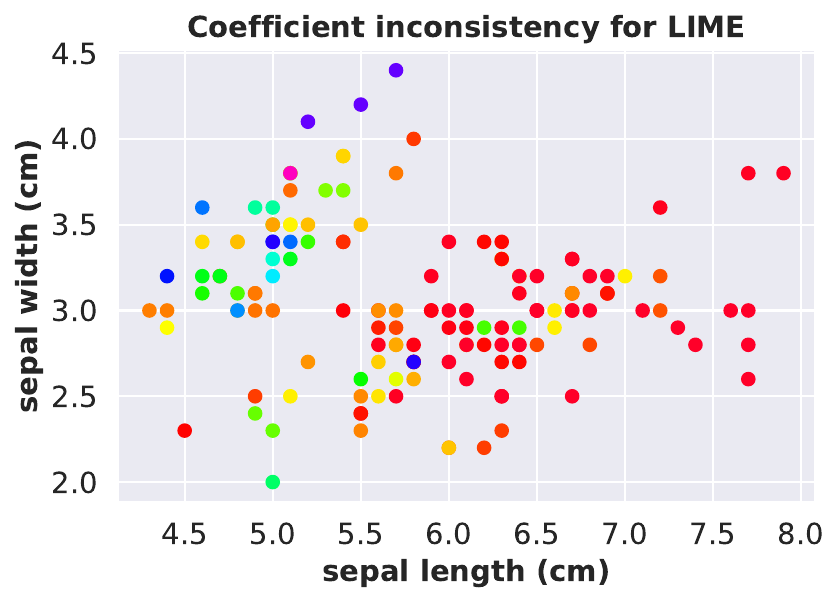} \\
 \includegraphics[width=0.35\textwidth]{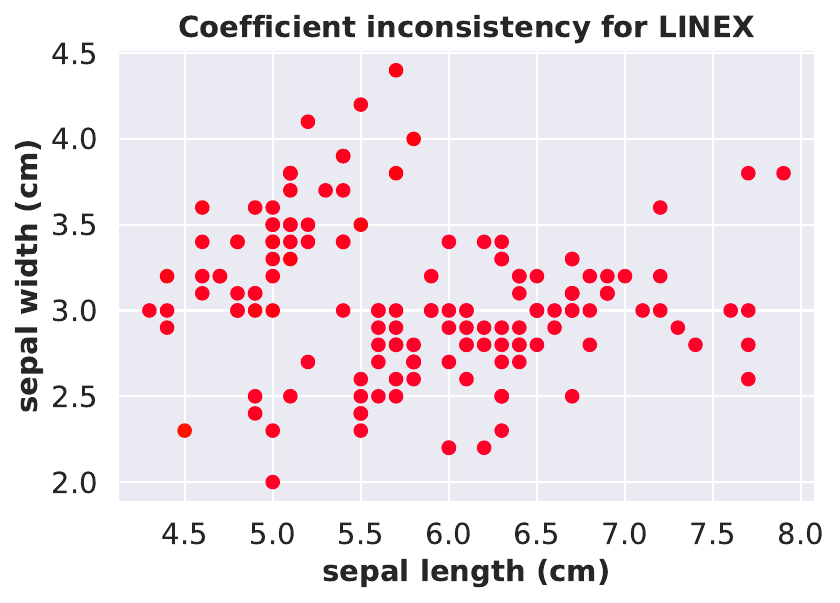} 
 \caption{Above we visualize for the IRIS dataset the Coefficient Inconsistency (CI) (see Section \ref{sec:exp} for exact definition and setup details)  between the explanation (top two features) for an example and its nearest neighbor in the dataset. Each circle denotes an example and a \emph{rainbow} colormap depicts the degree of inconsistency w.r.t. its nearest neighbor where red implies least inconsistency, while violet implies the most. As can be seen LINEX explanations are much more consistent than LIME's. 
 }
 \label{fig:intro}
\vspace{-1.5cm}
\end{wrapfigure}
There have been variants suggested to overcome these limitations \citep{melime, cdexp, maple,ans} primarily through mechanisms that create realistic neighborhoods or through adversarial training \citep{hima}, however, their efficacy is restricted to certain settings and modalities based on their assumptions and training strategies.

In this paper we introduce a new method called Locally INvariant EXplanations (LINEX) inspired by the invariant risk minimization (IRM) principle \citep{arjovsky2019invariant}, that produces explanations in the form of feature attributions that are robust to  neighborhood sampling and can recover faithful (i.e. mimic black-box behavior), stable (i.e. similar for closeby examples) and unidirectional (i.e. same sign attributions a.k.a. feature importances) for closeby examples, see section \ref{sec:desp}) explanations across tabular, image, and text modalities. In particular, we show that our method performs better than the competitors for random as well as realistic neighborhood generation, where in some cases even with the prior strategy our explanation quality is close to methods that employ the latter. 
Qualitatively, our method highlights (local) features as important that in the particular locality i) have consistently high gradient with respect to (w.r.t.) the black-box function and ii) where the gradient does not change significantly, especially in sign. Such stable behavior for LINEX is visualized in Figure \ref{fig:intro}, where we get similar explanations for nearby examples in the IRIS dataset. The (in)fidelity of LINEX is still similar to LIME (see Table \ref{tab:results}), but of course our explanations are much more stable. 
\section{Related Work}
\label{sec:relw}

Posthoc explanations can typically be partitioned into two broad categories global and local. Global explainability avers to trying to understand a black-box model at a holistic level where the typical tact is knowledge transfer \citep{distill,ProfWeight,Sratio} where (soft/hard) labels of the black-box model are used to train an interpretable model such as a decision tree or rule list \citep{Rudin2019}. Local explanations on the other hand avers to understanding individual decisions. These explanations are typically in two forms, either exemplar based or feature based. For exemplar based as the name suggests similar but diverse examples \citep{l2c,proto} are provided as explanations for the input in question. While for feature based \citep{lime,unifiedPI,CEM,LRPTOOLBOX,baylime}, which is the focus of this work, important features are returned as being important for the decision made for the input. There are some methods that do both \citep{maple}. Moreover, there are methods which provide explanations that are local, global as well as at a group level \citep{mame}. All of these methods though may not still provide stable and robust local feature based explanations which can be desirable in practice \citep{fragileint}.

Given this there have been more recent works that try to learn either robust or even causal explanations. In \citep{hima} the authors try to learn robust and stable local explanations relative to distribution shifts and adversarial attacks. However, the distribution shifts they consider are linear shifts and adversarial training is performed which can be slow and sometimes unstable \citep{limitsadvtr}. Moreover, the method seems to be applicable primarily to tabular data. There are also works \citep{dom1,dom2} which try to robustify gradient based explanations assuming white box access to the model. Works on causal explanations \citep{shapasym,shapcausal} mainly modify SHAP and assume access to a partial causal graph. Some others \citep{CMA} assume white-box access. In this work we do not assume availability of such additional information. There are also works which show that creating realistic  neighborhoods by learning the data manifold for LIME \citep{melime,cdexp} can lead to better quality explanations, where in a particular work \citep{tsp} it is suggested that projecting explanations themselves on to the manifold can also make them more robust. The need for stability in a exemplar neighborhood for LIME like methods has been highlighted in \citep{stabzhang}, with the general desire for stable explanations being also expressed in \citep{yeh2019fidelity,stabvisani}. Furthermore, it was recently surmised through expert and crowd worker user studies that stability is a key factor when it comes to assessing capability of a model or when learning a new domain \citep{hcomp}.

Given that our approach is inspired from IRM we now describe, how it is novel w.r.t. to it. 
It is important to realize that IRM approaches such as \cite{ahuja2020LRG,ahuja2020invariant} are designed for the out-of-distribution (OOD) generalization, which learn global models directly from the data. The main similarity of these works to ours is only that they also are game theory based approaches, but with the details being quite different. For one, they assume accessibility to environments which (ideally) correspond to different interventional distributions and with assumptions on the structural causal model derive results on how the true causal factors could be divulged. In our case, we propose ways to \emph{generate environments} as they are not given, and \emph{have $l_1$ and $l_{\infty}$ constraints on the entire and environment specific parts} of the model respectively, which is not the case with these prior works. As such those algorithms do not produce \emph{sparse unidirectional models} that are also consumable. Moreover, the \emph{perspective we provide is novel} in the context of local posthoc explanations where a priori it is not obvious that approaches from OOD generalization could be extended and adapted. 
Additionally, we propose a novel metric \emph{Unidirectionality} which is not part of any of these works, but as we have argued it is a desirable property for explanations.

\section{Preliminaries}
\label{sec:prelim}

\textbf{Invariant Risk Minimization:} Given a collection of training datasets $D = \{D_e\}_{e\in \mathcal{E}_{tr}}$ gathered from a set of environments $\mathcal{E}_{tr}$, where $D_e=\{\bm{x}^{i}_e, y^{i}_e\}_{i=1}^{n_e}$ is the dataset gathered from environment $e\in \mathcal{E}_{tr}$ and $n_e$ is the number of points in environment $e$. 
The feature value for data point $i$ is $\bm{x}_e^{i} \in \mathcal{X}$ and the corresponding label is $y_e^{i}\in \mathcal{Y}$, where $\mathcal{X} \subseteq \mathbb{R}^{d}$ and $\mathcal{Y}\subseteq \mathbb{R}$. Each point $(\bm{x}_e^{i},y_e^{i})$ in environment $e$ is drawn i.i.d from  a distribution $\mathbb{P}_e$.
Define a predictor $f:\mathcal{X} \rightarrow \mathbb{R}$.

The goal of IRM is to use these collection of datasets $D$ to construct a predictor $f$ that performs well across many unseen environments $\mathcal{E}_{all}$, where $ \mathcal{E}_{all} \supseteq \mathcal{E}_{tr}$. Define the risk achieved by $f$ in environment $e$ as $R_e(f) = \mathbb{E}_{e}\big[\ell(f(\bm{X}_e), Y_e)\big]$, where $\ell$ is the square loss when $f(\bm{X}_{e})$ is the predicted value and $Y_{e}$ is the corresponding label, $(\bm{X}_e,Y_e)\sim \mathbb{P}_e$ and the expectation $\mathbb{E}_{e}$ is defined w.r.t. the distribution of points in environment $e$.

An invariant predictor is composed of two parts a representation $\bm{\Phi} \in \mathbb{R}^{ d \times n}$ and a predictor (with the constant term) $\bm{w} \in \mathbb{R}^{d\times 1}$. We say that a data representation $\bm{\Phi}$ elicits an invariant predictor $\bm{w}^{\mathsf{T}}\bm{\Phi}$ across the set of environments $\mathcal{E}_{tr}$ if there is a predictor $\bm{w}$ that achieves the minimum risk for all the environments 
$\bm{w} \in \argmin_{\tilde{\bm{w}} \in \mathbb{R}^{d\times 1}} R_{e}(\tilde{\bm{w}}^{\mathsf{T}}\bm{\Phi}), \; \forall e \in \mathcal{E}_{tr}$.  IRM may be phrased as the following constrained optimization problem \citep{arjovsky2019invariant}:
\begin{align}
    &\min_{\bm{\Phi} \in \mathbb{R}^{d\times n},\bm{w}  \in \mathbb{R}^{d\times 1}} \sum_{e \in \mathcal{E}_{tr}}R_{e}(\bm{w}^{\mathsf{T}}\bm{\Phi})
    \text{~~~s.t.}\;\bm{w} \in \argmin_{\tilde{\bm{w}} \in \mathbb{R}^{d\times 1}} R_{e}(\tilde{\bm{w}}^{\mathsf{T}}\bm{\Phi}),\;\forall e \in \mathcal{E}_{tr}
    \label{eqn: IRM}
\end{align}
If $\bm{w}^{\mathsf{T}} \bm{\Phi}$ solves the above, then it is an invariant  predictor across the training environments $\mathcal{E}_{tr}$.

\textbf{Nash Equilibrium (NE):}
To understand how certain key aspects of our method function let us revisit the notion of Nash Equilibrium \citep{NE}. A standard normal form game is written as a tuple $\Omega = (\mathcal{N}, \{u_i\}_{i \in \mathcal{N}},\{\mathcal{S}_i\}_{i\in \mathcal{N}})$, where $\mathcal{N}$ is a finite set of players.  Player $i \in \mathcal{N}$ takes actions from a strategy set $\mathcal{S}_i$. The utility of player $i$ is $u_i:\mathcal{S} \rightarrow \mathbb{R}$, where we write the joint set of actions of all the players as  $\mathcal{S} = \Pi_{i\in \mathcal{N}} \mathcal{S}_i$. The joint strategy of all the players is given as $\bm{s} \in \mathcal{S}$,  the  strategy of player $i$ is $\bm{s}_i$ and the strategy of the rest of players is $\bm{s}_{-i} = (\bm{s}_{i^{'}})_{i^{'} \not = i}$.
\begin{definition}
A strategy $\bm{s}^{\dagger}\in \mathcal{S}$ is said to be a pure strategy Nash equilibrium (NE) if it satisfies,
$u_i(\bm{s}_{i}^{\dagger},\bm{s}_{-i}^{\dagger}) \geq u_i(k,\bm{s}_{-i}^{\dagger}), \forall k \in \mathcal{S}_{i}, \forall i \in \mathcal{N}$, where $u_i(\bm{s}_{i}^{\dagger},\bm{s}_{-i}^{\dagger})= u_i(\bm{s}_{1}^{\dagger},\bm{s}_{2}^{\dagger},...,\bm{s}_{\mathcal{N}}^{\dagger})=u_i(\bm{s}^{\dagger})$.
\end{definition}
NE thus identifies a state where each player is using the best possible strategy in response to the rest of the players leaving no incentive for any player to alter their strategy. In seminal work by \citep{debreu1952social} it was shown that for a special class of games called concave games such a pure NE always exists. This is relevant because the game implied by Algorithm \ref{algo:LINEX} falls in this category.

\section{Methodology}

We first define desirable properties for our explanation methods. The first three have been seen in previous works, while the last \textit{Unidirectionality} is new. We then describe our method where the goal is to explain a black-box model $f:\mathcal{X} \rightarrow \mathbb{R}$ for individual inputs $\bm{x}$ based on predictors $\bm{w}$ by looking at their corresponding components, also termed as feature attributions. 

We take inspiration from IRM since, our goal here too is to extract robust features that are ideally stable and unidirectional. The main difference is that we do not learn a new (possibly invariant) representation since, we desire interpretability and this new representation may not be interpretable. We hence, are restricted to the provided input or some other interpretable representation. Thus, given that $\bm{\Phi}\subseteq \mathcal{X}$ where $n=1$ (since local explanations) in our setup, our goal is to find the best predictor $\bm{w}$ (viz. high fidelity) for an input that will eliminate or at least mitigate the effect of unstable features. In other words, we want to identify features in the input space that will (roughly) have the same importance (i.e. are invariant) in the neighborhood of the example we want to explain. Our approach as we will see is similar in spirit to IRM games \cite{ahuja2020invariant}, where we adopt a game theoretic strategy to obtain such explanations. The differences with IRM games are mentioned in the last paragraph of Section \ref{sec:relw}.

\subsection{Desirable Properties}
\label{sec:desp}
We now discuss certain properties we would like our explainability method to have in order to provide robust explanations that could potentially be used for recourse. Let $D_t$ denote a (test) dataset with examples $(x,y)$  where $y_b(x)$ is the black-box models prediction on $x$ and $y_e^{x'}(x)$ is the prediction on $x$ ($\in \mathcal{X}$) using the explanation model at $x'$. The feature attributions (or coefficients) for the explanation model at $x$ are denoted by $c_e^{x}$, $\mathcal{N}_x$ denotes the exemplar neighborhood of $x$ with $|.|_{\text{card}}$ denoting cardinality and $|.|$ denoting absolute value.

\noindent\textbf{Fidelity:} This is the most standard property which all proxy model based explanation methods are evaluated against \citep{lime,unifiedPI,hima} as it measures how well the proxy model simulates the behavior of the black-box (i.e. faithfulness to the black box) it is attempting to explain. Here we define inverse of it, that is \textit{Infidelity (INFD)}, as the MAE between the black-box and explanation model predictions across all the test points: 
\begin{align}
\text{INFD}=\frac{1}{|D_t|_{\text{card}}}\sum_{(x,y)\in D_t} |y_b(x)-y_e^x(x)|.
\end{align}

We also define another metric here called \textit{Generalized Infidelity (GI)}, which also been used in previous works \citep{mame} to measure the generalizability of local explanations to neighboring test points. It is defined as: 
\begin{align}
\text{GI}=\frac{1}{|D_t|_{\text{card}}}\sum_{(x,y)\in D_t} \frac{1}{|\mathcal{N}_x|_{\text{card}}}\sum_{x'\in \mathcal{N}_x} |y_b(x)-y_e^{x'}(x)|.
\end{align}

\noindent\textbf{Stability:} This is also a popular notion \citep{robustexpl,mame,yeh2019fidelity} to evaluate robustness of explanations.
Largely, stability can be measured at three levels. One is prediction stability, which measures how much the predictions of an explanation model change for the same example subject to different randomizations within the method or across close by examples. The second is the variance in the feature attributions again for the same or close by examples. It is good for a method to showcase stability w.r.t. both even though in many cases the latter might imply the former. An interesting third notion of stability is the correlation between the feature attributions of an explanation model and average feature values of examples belonging to a particular class. This measures how much does the explanation method pick features that are important for the class, rather than spurious ones that seem important for just the example. We thus define two stability metrics.

\noindent\textit{Coefficient Inconsistency (CI):}  This notion has been used before \citep{robustexpl} to measure an explanation methods robustness. It can be defined as the MAE between the attributions of the test points and their respective neighbors: 
\begin{align}
\text{CI} = \frac{1}{|D_t|_{\text{card}}}\sum_{(x,y)\in D_t}\frac{1}{|\mathcal{N}_x|_{\text{card}}}\sum_{x'\in \mathcal{N}_x} |c_e^{x}-c_e^{x'}|_1.
\end{align}

\noindent\textit{Class-Attribution Consistency (CAC):} For local explanations of classification black-boxes, we expect certain important features to be highlighted across most of the explanations of a class. This is codified by this metric which is defined as follows:
\begin{align}
\text{CAC} =\frac{1}{|\mathcal{Y}|_{\text{card}}}\sum_{y\in \mathcal{Y}} r(\mu_e^{y},\mu_y),
\end{align}
where $\mathcal{Y}$ denotes the set of class labels in the dataset, $\mu_y$
the mean (vector) of all inputs in class $y\in \mathcal{Y}$, $\mu_e^{y}$ the mean explanation for class $y$ and $r$ the Pearson's correlation coefficient. This metric quantifies the consistency between the important features for a class and attributions provided by the explanations.

\noindent\textbf{Black-box Invariance:} This is the same as implementation invariance defined in \citep{intgrad}. Essentially, if two models have exactly the same behavior on all inputs then their explanations should also be the same. Since, our method is model agnostic with only query access to the model it is easy to see that it satisfies this property if the same environments are created.

\noindent\textbf{Unidirectionality:} This is a new property, but as we argue that this is a natural one to have. Loosely speaking, unidirectionality would measure how consistently the sign of the predictor for a feature is maintained for the same or close by examples by an explanation method. This is a natural metric \citep{miller2018explanation}, which from an algorithmic recourse \citep{algorec} perspective is also highly desirable. For instance, recommending a person to increase their salary to get a loan and then recommending to another person with a very similar profile to decrease their salary for the same outcome makes little sense.

We define the unidirectionality $\Upsilon$ as a measure of how consistent the sign of the attribution for a particular feature in a local explanation is when varying neighborhoods for the same example or when considering different close by examples. As such, given $m$ attributions for each of $d$ features denoted by $w_1^{(1)},...,w_{m}^{(d)}$ the metric for an example is:
\begin{align}
     \label{eq:unid}
     \Upsilon =\frac{1}{md} \sum_{i=1}^d\left|\sum_{j=1}^m \mathsf{sgn}\left(w_j^{(i)}\right)\right|
\end{align}
where $|.|$ stands for absolute value. Clearly, the more consistent the signs for the attribution of a particular feature across $m$ attributions the higher the value, where the maximum value can be one. If equal number of attributions have different signs for all features then $\Upsilon$ will be zero, the lowest possible value. This property thus measures how intuitively consistent (ignoring magnitude) the explanations are. Given its sole focus on the sign of the attributions it compliments the above metrics along with attributional robustness metrics \citep{attrobust,attrrobustenh}.

\begin{algorithm}[htbp]
\SetAlgoLined

\textbf{Input:} example $\bm{x}$, black-box predictor $f(.)$, number of environments to be created $k$,  ($l_{\infty}$) threshold $\gamma>0$, ($l_1$) threshold $t>0$ and convergence threshold $\epsilon>0$ 

  \textbf{Initialize:} $\forall i\in \{1,...,k\}$ $\tilde{\bm{w}}_i= \bm{0}$ and $\Delta=0$
  
  Let $\xi_1(.),...,\xi_k(.)$ be $k$ environment creation functions as described in section \ref{sec:env}
  
 \SetKwRepeat{Do}{do}{while}
 \Do{$\Delta \ge \epsilon$}
 {
 $\Delta=0$
 
 \For{$i=1$ to $k$}
 {
 $\tilde{\bm{w}}^+_{-i} = \sum_{j\in\{1,...,k\},j\neq i}\tilde{\bm{w}}_j$
 
 $\tilde{\bm{w}}_i^{\mathsf{prev}} = \tilde{\bm{w}}_i $
 
 $\tilde{\bm{w}}_i=\argmin\limits_{\tilde{\bm{w}}}\sum_{\tilde{\bm{x}}\in \xi_i(\bm{x})}\left(f(\tilde{\bm{x}})-\tilde{\bm{w}}^{+^{\mathsf{T}}}_{-i}\tilde{\bm{x}}-\tilde{\bm{w}}^{\mathsf{T}}\tilde{\bm{x}}\right)^2$ s.t. $|\tilde{\bm{w}}^+_{-i}+\tilde{\bm{w}}|_1\le t$ and $|\tilde{\bm{w}}|_{\infty}\le \gamma$
 
 $\Delta = \max\left(|\tilde{\bm{w}}_i^{\mathsf{prev}} - \tilde{\bm{w}}_i|_2, \Delta\right)$
 }
 }
 \KwOut{$\bm{w}=\sum_{i\in\{1,...,k\}}\tilde{\bm{w}}_i$}
 \caption{Locally Invariant EXplanations (LINEX).}
  \label{algo:LINEX}
\end{algorithm}
\subsection{Method}
\label{sec:meth}
\subsubsection{Description}
In Algorithm \ref{algo:LINEX}, we show the steps of our method LINEX. The input is the example we want to explain $\bm{x}$, the black-box predictor, a few thresholds that we describe next and $k$ (local) environments whose creation is described in Section \ref{sec:env}. In the algorithm we iteratively learn a constrained least squares predictor for each environment, where the final (local) linear predictor is the sum of these individual predictors. In each iteration when computing the contribution of environment $e_i$ to the final summed predictor, the most recent contributions of the other predictors are summed and the residual is optimized subject to the constraints. The first constraint is a standard lasso type constraint which tries to keep the final predictor sparse as in LIME.

\noindent\textbf{Why $l_{\infty}$ constraint?} The second constraint is more unique and is a $l_{\infty}$ constraint on the predictor of just the current environment. This constraint as we prove in Section \ref{sec:theory} is essential for obtaining robust predictors. To intuitively understand why this is the case consider we have two environments. In this case if the optimal predictors for a feature in each environment have opposite signs, then the Nash equilibrium (NE) is when each predictor takes $+\gamma$ or $-\gamma$ values as they try to force the sum to have the same sign as them. \emph{In other words, features that have a disagreement in even the direction of their impact are eliminated by our method}. LIME type methods on the other hand would simply choose some form of average value of the predictors which may be a risky choice especially for actionability/recourse given that the directions change so abruptly. On the other hand, if the optimal predictors for a feature in the two environments have the same sign, the lower absolute valued predictor would be chosen (assuming $\gamma$ is greater) making it a careful choice. The reasoning for this and a discussion involving more than two environments is given in Section \ref{sec:theory}.

The overall algorithm resembles a (simultaneous) game where each environment is a player trying to find the best predictor for its environment given all other predictors and constraints. Formally, for $i\in\{1,...,k\}$ the players are $\mathcal{N}=\{\xi_i\}$, their strategy space is $\mathcal{S}_i=[-\gamma,\gamma]^d$ and their utility $u_i\left(\tilde{\bm{w}}_i,\tilde{\bm{w}}^{+}_{-i}\right)=-\sum_{\tilde{\bm{x}}\in \xi_i(\bm{x})}\left(f(\tilde{\bm{x}})-\tilde{\bm{w}}^{+^{\mathsf{T}}}_{-i}\tilde{\bm{x}}-\tilde{\bm{w}}_i^{\mathsf{T}}\tilde{\bm{x}}\right)^2$. The optimization problem solved by each player is convex as norms are convex.
\subsubsection{Creating  Local Environments}
\label{sec:env}

In standard IRM, environments are assumed to be given. In our case of local explainability we have to decide how to produce them. We offer a few options for the environment creation functions $\xi_i$ $\forall i\{1,...,k\}$ in Algorithm \ref{algo:LINEX}.

\noindent\textbf{Random Perturbation:} This simple approach is similar to what LIME employs. We could perturb the input example by adding zero mean gaussian noise to create the base environment (used by LIME) and then perform bootstrap sampling to create the $k$ different environments. This will efficiently create neighbors in each environment, although they may be unrealistic in the sense that they could correspond to low probability points w.r.t. the underlying distribution.

\noindent\textbf{Realistic Generation/Selection:} One could also create neighbors using data generators such as done in MeLIME \citep{melime} or select neighboring examples from the training set as done in MAPLE \citep{maple} to create the base environment following which bootstrap sampling could be done to form the $k$ different environments. This could provide more realistic neighbors than the previous one, but may be much more computationally expensive. Other than bootstrapping one could also oversample and try to find the optimal hard/soft partition through various clustering type objectives \citep{charubook,envinf}.

\subsection{Theoretical Results}
\label{sec:theory}

In this section, we analyze the output of Algorithm 1 with two environments. The extension to multiple environments is discussed following this result, where the general intuition is still maintained but some special cases arise depending on whether there are an even or odd number of environments. To prove our main result we make two assumptions.

\noindent\textbf{Assumption 1} \emph{The features of the samples in the local environments are independent.} 

This assumption is satisfied by the most standard way of creating neighborhoods/environments, where gaussian noise is used to create them as described in Section \ref{sec:env}.

\noindent\textbf{Assumption 2} \emph{$t\geq \gamma d$, where $d $ is the dimensionality of the feature vector.}

Here $t$ is the parameter in the $\ell_1$ penalty and $\gamma$ in the $\ell_\infty$ as noted in Algorithm \ref{algo:LINEX}. Making this assumption ensures that we closely analyze the role of the $\ell_{\infty}$ penalty, which is one of our main novelties. 

\noindent\textbf{Definition 2} Let the explanation that each environment $\xi_{i}$ arrives at for an example $\bm{x}$ based on unconstrained least squares minimization be $\bm{w}_{i}^{*}$ where,
\begin{equation}
\label{eq:twoexpl}
    \bm{w}_{i}^{*} \in \argmin_{\tilde{w}\in \mathbb{R}^d} \mathbb{E}_{\tilde{\bm{x}}\in \xi_{i}(x)}[(f(\tilde{\bm{x}})-\tilde{w}^{\mathsf{T}}\tilde{\bm{x}})^2]
\end{equation}
The expectation is taken w.r.t the environment generation distribution. 
\begin{theorem}
\label{thm1}
The output of Algorithm \ref{algo:LINEX} under Assumptions 1, 2 and equation \ref{eq:twoexpl} is given by:
\begin{equation}
    \bm{w}=\Big(\bm{w}_1^{*} \odot \bm{1}_{|\bm{w}_2^{*}|\geq |\bm{w}_1^{*}|} + \bm{w}_2^{*} \odot \bm{1}_{|\bm{w}_1^{*}|>|\bm{w}_2^{*}|} \Big) \bm{1}_{ \bm{w}_1^{*}\odot \bm{w}_2^{*} \ge \bm{0}}
    \label{ne_exp}
\end{equation}
where $\odot$ is element wise product and $\bm{1}$ is the indicator function.
\end{theorem}
\begin{proof}[Proof Sketch]
The above expression describes the NE of the game played between the two local environments each trying to move $\bm{w}$ towards their least squares optimal solution. Given assumptions 1 and 2, we witness the following behavior of our method. Let the $i^{th}$ feature of the predictors $\tilde{w}_{1}$ and $\tilde{w}_{2}$ from Algorithm \ref{algo:LINEX} be $\tilde{w}_{1i}$ and $\tilde{w}_{2i}$ respectively. Let the corresponding least squares optimal predictors for the $i^{th}$ feature have the following relation: $w_{1i}^{*}>w_{2i}^{*}$ and $|w_{1i}^{*}|>|w_{2i}^{*}|$. Then the two environments will push the ensemble predictor, $\tilde{w}_{1i}+ \tilde{w}_{2i}$, in opposite directions during their turns, with the first environment increasing its weight, $\tilde{w}_{1i}$, and the second environment decreasing its weight, $\tilde{w}_{2i}$. Eventually, the environment with a higher absolute value ($\xi_1=1$ since $|w_{1i}^{*}|>|w_{2i}^{*}|$) reaches the boundary ($\tilde{w}_{1i}=\gamma$) and cannot move any further due to the $l_{\infty}$ constraint. The other environment $\xi_2$ best responds, where it either hits the other end of the boundary ($\tilde{w}_{2i}=-\gamma$), in which case the weight of the ensemble for component $i$ is zero, a case which occurs if $w_{1i}^{*}$ and $w_{2i}^{*}$ have opposite signs; or gets close to the other boundary while staying in the interior ($\tilde{w}_{2i}=w_{2i}^{*}-\gamma$), in which case the weight of the ensemble for feature $i$ is $w_{2i}^{*}$, a situation which occurs if $w_{1i}^{*}$ and $w_{2i}^{*}$ have the same sign.
\end{proof}
\noindent\textbf{Implications of the Theorem \ref{thm1}:} The following are the main takeaways from Theorem \ref{thm1}: (1) If the signs of the explanations for unconstrained least squares for the two environments differ for some feature, then the algorithm outputs a zero for that feature attribution. (2) If the signs of the explanations for the two environments are the same, then the algorithm outputs the lesser magnitude of the two. These two properties are highly desirable from an algorithmic recourse or actionability perspective, where the first biases us to not rely on features where the black-box function changes direction rapidly (unidirectionality). The second, provides a reserved estimate so that we do not incorrectly over rely on the particular feature (stability). Based on similar logic presented in the proof sketch the behavior for more than two environments for LINEX is discussed in Suppl. C.

\begin{figure}[htbp]
\includegraphics[height=1.02in,width=0.45\textwidth]{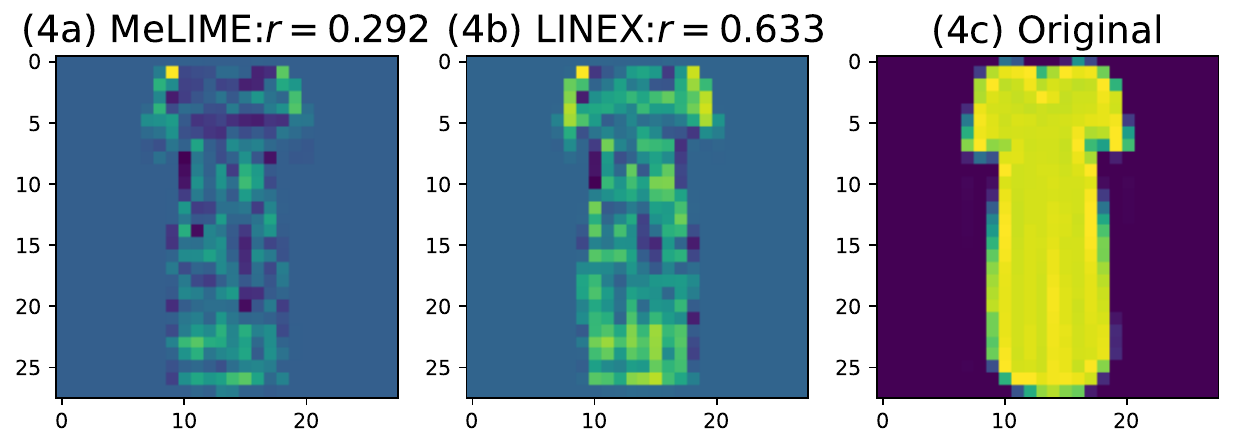}
\includegraphics[height=1in, 
width=0.45\textwidth]{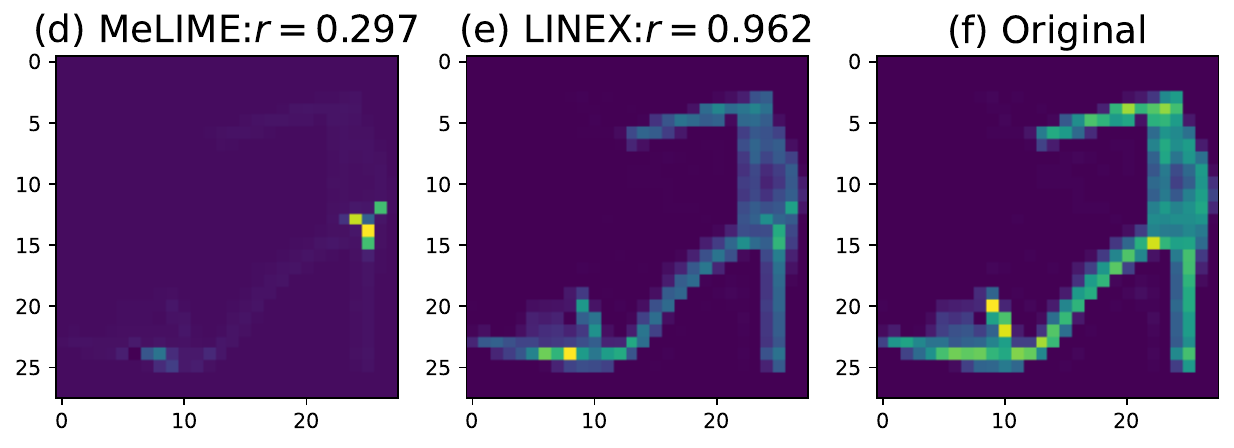}
\caption{Sample results using FMNIST dataset for two classes. (a-c): Class \textit{Dress}, (d-f): Class \textit{Sandal}. (a, d): MeLIME explanations. (b, d): LINEX explanations.  (c, f): Original images. 
We observe that LINEX explanations capture important artifacts and thus exhibit significantly higher correlation with the original images for the same level of sparsity, where in aggregate too the correlations are high w.r.t. images belonging to a particular class, thus showcasing higher stability (i.e. high CAC) as is seen in Table \ref{tab:results}. More examples are shown in Suppl. G.}
\label{fig:fmnist}
\end{figure}

\begin{table}[htbp]
\small
\caption{Below are three example positive sentiment sentences from the Rotten Tomatoes dataset. Green and red indicate the most important word highlighted by \textcolor{green}{MeLIME} and \textcolor{red}{LINEX} respectively. As seen LINEX highlights stronger positive sentiment words. More examples in Suppl. F.} 
\label{tab:rot}
\vspace{3mm}
\centering
\begin{tabular}{|l|p{5cm}|l|} 
\hline
Example 1 & Example 2 & Example 3\\
\hline
\multirow{2}{*}{one-of-a-\textcolor{green}{kind} near-\textcolor{red}{masterpiece}} &  \textcolor{red}{moving} tale of love and destruction in \textcolor{green}{unexpected} places , unexamined lives & \multirow{2}{*}{\textcolor{green}{spare} yet \textcolor{red}{audacious} . . .}\\
\hline
\end{tabular}
\end{table}

\section{Experiments}
\label{sec:exp}
We test our method on five real world datasets covering all three modalities: IRIS (Tabular) \citep{uci}, Medical Expenditure Panel Survey (Tabular) \citep{MEPS}, Fashion MNIST (Image) \citep{fmnist}, CIFAR10 (Image) \cite{cifar} and Rotten Tomatoes reviews (Text) \citep{rotten} with LIME-like random (rand) and MeLIME-like realistic neighborhood generation (real) or MAPLE-like realistic neighborhood selection (mpl). 
The summary of black-box classifier accuracies, and type of realistic perturbation used for the datasets are provided in Table 3
in the Supplement. In other cases except FMNIST and CIFAR10 which come with their own test partition we randomly split the datasets into 80/20\% train/test partition and average results for the local explanations over this test partition. 
For LINEX we produce two environments where the two environments are formed by performing bootstrap sampling on the base environment which is created either by rand, real or mpl type neighborhood generation. Thus in all cases the union of the environments is the same as a single neighborhood used to produce explanations for the competitors making it a fair comparison. Behavior with more environments is in Suppl. E.

\begin{table}[htbp]
\small
\caption{Comparison of the different methods based on infidelity (INFD), generalized infidelity (GI), coefficient inconsistency (CI), class attribution consistency (CAC) and unidirectionality ($\Upsilon$). $\uparrow$ indicates higher value for the metric is better, and $\downarrow$ indicates lower is better. Statistically significant results based on paired t-test are bolded. LINEX is better than baselines in 21 out of 40 cases, and worse only in 5 cases. Plots showing behavior with varying neighborhood size, number of environments and kernel width are in Suppl. E.}
\addtolength{\tabcolsep}{-4pt}
\vspace{3mm}
\centering
\begin{tabular}{|c | c | c | c | c | c | c |}
\hline
\textit{Dataset} & \textit{Method} & INFD $\downarrow$ & GI $\downarrow$ & CI $\downarrow$ & $\Upsilon$ $\uparrow$ & CAC $\uparrow$ \\
\hline
\hline
\multirow{7}{*}{\textit{IRIS}} 

& {LIME}  &  $0.015 \pm 0.011$ & $0.132 \pm 0.042$ & $0.319 \pm 0.132$ & $0.646 \pm 0.040$ & $0.667 \pm 0.167$ \\
& {S-LIME}  &  $0.015 \pm 0.010$ & $0.077 \pm 0.011$ & $0.143 \pm 0.045$ & $0.704 \pm 0.037$ & $0.878 \pm 0.034$ \\
& {LINEX/rand} &  $0.013 \pm 0.009$ & $\mathbf{0.052 \pm 0.008}$ & $\mathbf{0.044 \pm 0.013}$ & $\mathbf{0.802 \pm 0.043}$ & $\mathbf{0.921 \pm 0.042}$ \\

\cline{2-7}

& {MeLIME}  & $0.008 \pm 0.003$ & $0.049 \pm 0.018$ & $0.219 \pm 0.108$ & $0.629 \pm 0.013$ & $0.464 \pm 0.100$\\
& {LINEX/real}  & $0.009 \pm 0.003$ & $\mathbf{0.029 \pm 0.003}$ & $\mathbf{0.024 \pm 0.002}$ & $\mathbf{0.744 \pm 0.044}$ & $\mathbf{0.942 \pm 0.023}$\\

\cline{2-7}
& {MAPLE}  & $0.009 \pm 0.001$ & $0.038 \pm 0.004$ & $0.261 \pm 0.033$ & $0.458 \pm 0.032$ & $0.586 \pm 0.035$\\
& {LINEX/mpl}  & $0.013 \pm 0.000$ & $\mathbf{0.020 \pm 0.000}$ & $\mathbf{0.026 \pm 0.002}$ & $\mathbf{0.694 \pm 0.008}$ & $\mathbf{0.929 \pm 0.004}$\\
\hline
\hline
\multirow{5}{*}{\textit{MEPS}} 
& {LIME} & $0.158 \pm 0.066$ & $0.214 \pm 0.041$ & $0.005 \pm 0.001$ & $0.981 \pm 0.006$ & \multirow{3}{*}{NA}\\
& {S-LIME} & $0.158 \pm 0.066$ & $0.214 \pm 0.042$ & $0.005 \pm 0.001$ & $0.974 \pm 0.008$ & \\
& {LINEX/rand} & $\mathbf{0.130 \pm 0.052}$ & $\mathbf{0.164 \pm 0.021}$ & $0.003 \pm 0.001$ & $0.979 \pm 0.006$ & \\

\cline{2-7}
& {MAPLE} & $\mathbf{0.063 \pm 0.000}$ & $\mathbf{0.067 \pm 0.000}$ & $0.007 \pm 0.000$ & $0.957 \pm 0.000$ & \multirow{2}{*}{NA}\\
& {LINEX/mpl}  & $0.098 \pm 0.001$ & $0.094 \pm 0.001$ & $0.007 \pm 0.000$ & $0.950 \pm 0.000$	& \\
\hline
\hline
\multirow{5}{*}{\textit{FMNIST}} 
& {LIME}  &  $0.162 \pm 0.003$ & \multirow{3}{*}{NA} & \multirow{3}{*}{NA} &	\multirow{3}{*}{NA} & \multirow{3}{*}{NA} \\
& {S-LIME}  & $0.142 \pm 0.003$ &  &  &	&  \\
& {LINEX/rand} & $0.149 \pm 0.002$ &  &	 &	& \\
                             
\cline{2-7}
& {MeLIME}  & $\mathbf{0.001 \pm 0.000}$ & $\mathbf{0.277 \pm 0.000}$ & $0.007 \pm 0.000$ & $0.769 \pm 0.000$ & $0.327 \pm 0.000$ \\
& {LINEX/real}  & $0.100 \pm 0.002$ & $0.304 \pm 0.001$ & $0.002 \pm 0.000$ & $\mathbf{0.780 \pm 0.000}$ & $\mathbf{0.649 \pm 0.001}$ \\
\hline
\hline
\multirow{5}{*}{\textit{CIFAR10}} 
& {LIME}  &  $0.191 \pm 0.005$ & \multirow{3}{*}{NA} & \multirow{3}{*}{NA} &	\multirow{3}{*}{NA} & \multirow{3}{*}{NA} \\
& {S-LIME}  & $0.185 \pm 0.002$ &  &  &	&  \\
& {LINEX/rand} & $0.186 \pm 0.002$ &  &	 &	& \\
                             
\cline{2-7}
& {MeLIME}  & $0.100 \pm 0.003$ & $0.412 \pm 0.007$ & $0.014 \pm 0.000$ & $0.546 \pm 0.003$ & \multirow{2}{*}{NA} \\
& {LINEX/real}  & $0.090 \pm 0.005$ & $\mathbf{0.279 \pm 0.001}$ & $\mathbf{0.006 \pm 0.000}$ & $\mathbf{0.679 \pm 0.004}$ &  \\
\hline
\hline
\multirow{5}{4em}{\textit{Rotten Tomatoes}}
& {LIME}  & $0.079 \pm 0.036$ & \multirow{3}{*}{NA} & \multirow{3}{*}{NA} &	\multirow{3}{*}{NA} & \multirow{3}{*}{NA} \\
& {S-LIME}  & $0.075 \pm 0.035$ &  &  &	&  \\
& {LINEX/rand} & $0.069 \pm 0.032$ &  &	 &	& \\

\cline{2-7}
& {MeLIME}  & $\mathbf{0.029 \pm 0.001}$ & $0.391 \pm 0.000$ & $0.000 \pm 0.000$ & $0.999 \pm 0.000$ & $0.909 \pm 0.000$\\
& {LINEX/real}  & $0.053 \pm 0.000$ & $\mathbf{0.361 \pm 0.000}$ & $0.000 \pm 0.000$ & $1.000 \pm 0.000$ & $\mathbf{0.953 \pm 0.001}$ \\
\hline

\end{tabular}
\label{tab:results}
\end{table}

Given the neighborhood generation schemes we compare LINEX with LIME, Smoothed LIME (S-LIME), MeLIME and MAPLE, where for S-LIME we average the explanations of LIME across the LINEX environments. SHAP's results are in Suppl. H, since it is not a natural fit here. Nor are methods such as saliency maps, gradcam, integrated gradients as they are \emph{white-box}
methods requiring access to a differentiable model.

\noindent\textbf{Metrics:} We evaluate using five simple metrics: Infidelity (INFD), Generalized Infidelity (GI), Coefficient Inconsistency (CI), Class Attribution Consistency (CAC) and Unidirectionality ($\Upsilon$), which are defined in section \ref{sec:desp}. The first two evaluate \textit{faithfulness}, the next two \textit{stability} and the last \textit{goodness for recourse}.
\begin{wrapfigure}{r}{0.6\textwidth}
\vspace{-.15cm}
\includegraphics[width=0.6\textwidth]{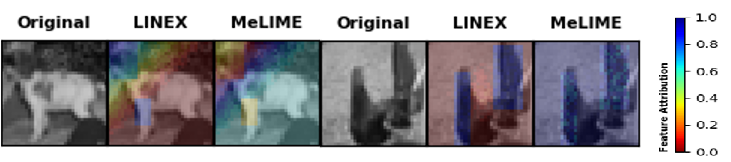}
\caption{Sample results using CIFAR10 dataset for dog and bird class. As can be seen LINEX focuses more on salient features such as head and legs for the dog, and wings for the bird (rather than also the background). More examples are shown in Suppl. G.}
\vspace{-0.25cm}
\label{fig:cifar}
\end{wrapfigure}

We report the above metrics in Table \ref{tab:results}.
Each result in Table \ref{tab:results} is mean $\pm$ standard error of the mean over five kernel sizes $\tau \sqrt{d}$ generally, where $\tau = \{0.05, 0.1, 0.25, 0.5, 0.75\}$. Test neighborhoods do not make sense for random perturbations with FMNIST, CIFAR10 and Rotten Tomatoes because the features (viz. superpixels) used by neighboring test examples are different. Also, we do not use realistic perturbations with MEPS since KDE and VAE generators do not work well with categorical data. In addition, since MEPS data uses regression black-box, CAC cannot be computed. Also for CIFAR10 images in a class are not aligned so CAC is inapplicable. All these justify the missing entries in Table \ref{tab:results}. The results were generated on Linux machines with $56$ cores and $242$ GB RAM. 
More details regarding the exact perturbation schemes for LIME/MeLIME/MAPLE, the perturbation neighborhood sizes and the time taken by the methods are in Suppl. A
and Suppl. D.

\noindent\textbf{Observations:} \emph{Quantitatively}, we see that in terms of CAC, LINEX is better than baselines in all cases which indicates that on average the LINEX explanations highlight the important features characterizing the entire class, making them more stable. This is also verified by looking at $\Upsilon$ and CI metrics where LINEX is similar or better than others. For GI and INFD, the results are more evenly spread which implies that LINEX's key advantage is obtaining stable and unidirectional explanations that are faithful to a similar degree. Ablation studies showing superiority of LINEX over MeLIME on the FMNIST dataset where we have significantly higher INFD than MeLIME are given in Suppl. J.

An interesting observation is that when it comes to the stability metrics (CI and CAC) and unidirectionality LINEX with even random perturbation model is better than MeLIME in some cases. This is very promising as it means LINEX could be potentially be trusted without the need to generate realistic perturbations which may be computationally expensive or not even possible.

\emph{Qualitatively}, we see in Figures \ref{fig:fmnist} and \ref{fig:cifar}, that LINEX explanations are more coherent and highlight more salient features compared to MeLIME. Even on the text data we see more reasonable attributions in Table \ref{tab:rot}, where ``masterpiece", ``moving" and ``audacious" are highlighted as the most important words indicative of positive sentiment in the three examples.
We also performed qualitative error analysis on FMNIST where our INFD is much worse than MeLIME and is described in Suppl. I.
We see that even where LINEX has high infidelity it invariably still focuses on salient features ignoring superfluous features which may result in lower fidelity but may not be critical for correct identification. The goodness of these features identified by LINEX can be further verified by looking at other metrics such as GI, CAC, CI and $\Upsilon$ in Table \ref{tab:results} where it is either comparable or better than MeLIME.

\section{Discussion}
\label{sec:disc}
In this paper we have provided a method based on a game theoretic formulation and inspired by the invariant risk minimization principle to provide faithful, stable and unidirectional explanations. We have defined the latter property and argued that it is somewhat of a necessity (may not be sufficient) for recourse. We have theoretically shown that our method has a strong tendency to be stable and unidirectional as we will mostly eliminate features where the black-box models gradient changes abruptly and in other cases choose a conservative value. Empirically, we have verified this where we outperform competitors in majority of the cases on these metrics. Interestingly, in some cases our method provides more stable and unidirectional explanations with just a random perturbation model relative to more expensive methods that use realistic neighbors.

We now discuss a real world use case we tested our method on. We worked with a large financial institution to explain the fraud detection model they had built. The Association of Certified Fraud Examiners (ACFE) claims that roughly 5\% of a companies revenue is lost to fraud every year. Thus, catching fraud or even non-compliance is extremely important for any organization. Their model (fraud $=1$
 else $0$) had 
 $\approx$ 91\% accuracy. The inputs to the model were (transactional) invoices and details corresponding to those invoices such as vendor name, invoice amount, purchase order (PO) or not, vendor address, commodity code, country perception indices (CPI), etc. Since, one of the focuses is to reduce false positives accurate explanations are important. We applied LINEX to this setting to explain why certain invoices were classified as fraudulent. The experts found that in majority of the cases (913 out of 1000) the attributions of LINEX especially in terms of sign made sense. For instance, low CPI implies high risk and so LINEX gave a negative coefficient for this feature for most examples, while LIME gave a positive coefficient for many instances. Going forward their plan is to incorporate such capabilities into their workflow to further improve fraud detection precision.

In the future, it would be worth experimenting with more varied strategies to form environments and if possible find the optimal ones \citep{envinf}, which may lead to picking even more relevant features that are ``causal" to the local decision.

\section{Summary of the Supplement}
Information about black-box classifier accuracies and realistic perturbation methods used for the datasets are provided in Table 3.
Suppl. A has run time comparisons. Suppl. B has proof of Theorem 1. Suppl. C discusses theoretical behavior of LINEX for more than two environments. Suppl. D has dataset details and hyperparameter specifications. Suppl. E has experiments with different hyperparameter combinations (including more than 2 environments). Suppl. F has additional examples of text data attributions. Suppl. G has example feature attributions with image data. Suppl. H has SHAP results. Suppl. I, J and K has error analysis and ablation studies. Suppl. L has additional synthetic experiments. Suppl. M discusses sensitivity to $\gamma$. Suppl. N demonstrates convergence of LINEX. Suppl. O discusses limitations of LINEX. Figure \ref{fig:IRIS_mean_median_mom_smoothing} depicts SLIME variants using median and median of means which turn out to be worse than using the (typical) mean.

\bibliography{ExAbsent}
\bibliographystyle{plainnat}

\clearpage
\appendix

\begin{table*}
\small
\caption{Datasets, models and neighborhoods used in experiments. RF$\rightarrow$ Random Forest, NN$\rightarrow$ Neural Network, ResNet$\rightarrow$ Residual Network and NB$\rightarrow$ Naive Bayes.}
\label{tab:datasets}
\centering
\begin{tabular}{lllp{5.5cm}} 
\toprule
Dataset & Modality & Black-box model acc/$R^2$, & Realistic neighborhood creation methods\\
\midrule
IRIS~
& tabular & RF classifier, 93\% & KDEGen \citep{melime}, RF \citep{maple}\\  
MEPS~
& tabular & RF regressor, $0.325$ & \citep{maple} \\  
FMNIST~
& image & NN classifier, 87\% & VAEGen \citep{melime}\\ 
CIFAR10~& image & ResNet18, 95\% & VAEGen \citep{melime}\\
Rotten Tomatoes~
& text & NB classifier, 75\% & Word2VecGen \citep{melime}\\ 
\bottomrule
\end{tabular}
\vspace{-0.2cm}
\end{table*}

\section{Efficiency of LINEX}
\label{app:linex_efficiency}

 It is important to note that the query complexity (i.e. number of times we query the black box to obtain an explanation) of LINEX is the same as that of LIME since the union of the environments is the same as a LIME perturbation neighborhood. This is important in todays cloud-driven world where models may exist on different cloud platforms and posthoc explanations are an independent service where each call to the model has an associated cost. In terms of running time for two environments, convergence was fast and running time was approximately 2.5 times that of LIME (LINEX took 2.5 seconds on IRIS for 30 examples as opposed to 1 second by LIME, LINEX took 47 seconds on MEPS for 500 examples as opposed to 18 seconds by LIME), which is very similar to Smoothed LIME (S-LIME) (took 2.3 seconds on IRIS and 40 seconds on MEPS) that we still outperform in majority of the cases. 
 
 Realistic neighborhood generation can be time consuming especially for MeLIME since generators have to be trained which may take up to an hour using a single GPU for datasets such as FMNIST. After the generator is trained and neighborhood sampled MeLIME takes the same amount of time as LIME since the model fitting procedure is the same. MAPLE took 1.5 seconds for the IRIS dataset for 30 examples and 27 seconds for 500 MEPS examples.
 
 A way to further speed up LINEX would be to implement it through \emph{embarrassing parallelism} which can easily be done across explanations. This will prevent scaling of the running time in the number of examples when many explanations are needed. The setting with many explanations is anyway where we would need efficiency because if only few explanations were desired the slightly higher running time of LINEX would not be an issue.


\section{{Proof of Theorem 1}}

{Expanding on the proof sketch provided in the main paper we now provide a case wise analysis to prove Theorem 1.}

{$\bullet\; \bm{w}_1^{*}=\bm{w}_2^{*}$: If the optimal solutions to both environments in the convex set $[-\gamma,\gamma]^d$ are the same, then in the first iteration itself where we fit to the first environment we would have reached the optimal solution to our problem where $\tilde{\bm{w}}_1=\bm{w}_1^{*}$. This is because in the second iteration where we fit the second environment to the residual from the previous fit $\tilde{\bm{w}}_2=\bm{0}$ and the algorithm would terminate. This would imply the output of algorithm 1 would be $\bm{w}=\bm{w}_1^{*}$.}

{$\bullet$  $\bm{w}_1^{*}\not=\bm{w}_2^{*}$: When the optimal solutions for the two environments are not equal we consider the following two cases:
\begin{itemize}
    \item \textbf{Opposite  sign attributions:} If  the $i^{th}$ component of $\bm{w}_1^{*}$ and $\bm{w}_2^{*}$ have opposite signs, then the $i^{th}$ components of the  ensemble predictor, $\tilde{w}_{1i}$ and $\tilde{w}_{2i}$ are both at the boundary   $\gamma$ and $-\gamma$ respectively if $\tilde{w}_{1i}>0$. This is because both try to push the ensemble (i.e. their sum) towards the sign they have where eventually they reach the boundary $\pm\gamma$ and have no incentive to deviate. Any deviation from these values will lead to a higher least squares error in their environment, thus making this a NE.
    \item \textbf{Same  sign attributions:} If  the $i^{th}$ component of $\bm{w}_1^{*}$ and $\bm{w}_2^{*}$ have same signs, then the $i^{th}$ component of ensemble predictor constructed from the NE is set to the least squares attribution with a smaller absolute value, i.e., ${w}_{i} =w_{1i}^{*}$, where $|w_{1i}^{*}|\leq |w_{2i}^{*}|$.  Without loss of generality assume $0<w_{1i}^{*}<w_{2i}^{*}$, the attribution of the environments' predictors in NE, then $\tilde{w}_{1i}$ and $\tilde{w}_{2i}$ have opposite signs, i.e., $\tilde{w}_{2i} = \gamma$ and $\tilde{w}_{1i} = w_{1i}^{*}- \gamma$ where the ensemble predictor for the $i^{\text{th}}$ component would be $w_i=\tilde{w}_{1i}+\tilde{w}_{2i}=w_{1i}^{*}- \gamma+\gamma=w_{1i}^{*}$, since any deviation from this would lead to a worse least squares loss for the corresponding environment.  This shows that ensemble predictor is conservative and selects the smaller least squares attribution. 
\end{itemize}}

\section{Behavior for More than Two Environments}
\label{sec:m2envs}
Given Assumptions 1 and 2 we now discuss the behavior of our method for more than two environments. If the number of environments is odd, then using similar logic to that discussed in the proof sketch one can see that the feature attribution would be equal to the median of the feature attributions across all the environments. Essentially, all environments with optimal least squares attributions above the median would be at $+\gamma$, while those below it would be at $-\gamma$. The one at the median would remain so with no incentive for any environment to alter its attribution making it a NE. This is a stable choice that is also likely to be faithful as we have no more information to decide otherwise. On the other hand if we have an even number of environments the final attribution in this case depends on the middle two environments in the same manner as the two environment case proved in Theorem \ref{thm1}. Thus, if the optimal least squares attributions of the middle two environments have opposite sign, then the final attribution is zero, else its the lower of the two attributions in terms of the numerical value. This happens because the NE for the other environments is $\pm\gamma$ depending on if their optimal least squares attributions are above/below those of the middle two environments. This again is a stable and likely to be faithful choice, where also unidirectionality is preferred.

\section{Experimental Details}
\label{app:exp_details}


\subsection{Dataset Details and Hyperparameter Specifications}
We describe the datasets and the hyperparameters used for each. We set perturbation neighborhood sizes 10 (IRIS), 500 (MEPS), 100 (FMNIST-random), 500 (FMNIST-realistic), 100 (CIFAR10-random), 500 (CIFAR10-realistic), 100 (Rotten tomatoes) for generating local explanations. We also use 3, 10, 10, 10, 5 as exemplar neighborhood sizes to compute GI, CI and $\Upsilon$ metrics for the five datasets respectively. We also use $5-$sparse explanations for all cases except FMNIST and CIFAR10 with realistic perturbations where we follow MeLIME and generate a dense explanation using ridge penalty with penalty multiplier value of $0.001$. The $\ell_\infty$ bound $\gamma$ in Algorithm \ref{algo:LINEX} is set as the maximum absolute value of linear coefficient computed by running LIME/MeLIME in the two individual environments. Please look at IRIS dataset first since it contains some of the common details used across others.

\paragraph{IRIS (Tabular):} This dataset has 150 instances with four numerical features representing the sepal and petal width and length in centimeters. The task is to classify instances of Iris flowers into three species: \textit{setosa}, \textit{versicolor}, and \textit{virginica}. A random forest classifier was trained with a train/test split of 0.8/0.2 and yielded a test accuracy of 93\%. We provide local explanations for the prediction probabilities for class \textit{setosa}. For both random and realistic perturbations, we use a perturbation neighborhood size of $n$. For random perturbations, we used the same approach followed by LIME and sample from a Gaussian around each data point.
Realistic perturbations (with the same number $n$) were generated using KDEGen~\cite{melime}, a kernel density estimator (KDE) with the Gaussian kernel fitted on the training dataset to sample data around a sample point. For both random and realistic perturbations, we weight the neighborhood using a Gaussian kernel of width $\tau\sqrt{d}$, where $d$ is the dimension of the feature vector and $\tau=\{0.05, 0.1, 0.25, 0.5, 0.75\}$, and this corresponded to kernel widths $\{0.1,0.2,0.5,1.0,1.5\}$. We also perform a weighted version of realistic selection where we use MAPLE~\cite{maple} to assign weights to all the test examples and pick the top $n$ weighted examples to use as the perturbation neighborhood. For random/realistic perturbations and realistic selection, the corresponding environments (of size $n$ each) for LINEX are created by drawing $k$ bootstrap samples where $k=\{2,3,4,5\}$ in our experiments. We test for $n=\{10,20,30,40,50\}$ with this dataset.

\paragraph{Medical Expenditure Panel Survey (Tabular):}
The Medical Expenditure Panel Survey (MEPS) dataset is produced by the US Department of Health and Human Services. It is a collection of surveys of families of individuals, medical providers, and employers across the country. We choose \textit{Panel  19}  of the  survey which consists of a cohort that started in 2014 and consisted of data collected over $5$ rounds of interviews over $2014-2015$. The outcome variable was a composite utilization feature that quantified the total number of healthcare visits of a patient. The features used included demographic features, perceived  health status, various diagnosis, limitations, and socioeconomic factors. We filter  out records that had a utilization (outcome) of 0, and log-transformed the outcome for modeling. These pre-processing steps resulted in a dataset with $11136$ examples and $32$ categorical features. We train a random forest regressor that has a test $R^2$ of $0.325$ in this dataset. We provide local explanations of the predictions. With MEPS, we do not use realistic perturbations since KDE and VAE generators do not work well with categorical data. Otherwise the setting is similar as IRIS data, except that we use $n=\{50,100,200,300,400,500\}$. The kernel widths in this case were $\{0.28,0.57,1.41,2.83,4.24\}$. We use $k=\{2,3,4,5\}$ for this dataset.

\paragraph{Fashion MNIST (Images):}
This dataset has $28\times28$ grayscale images of fashion articles with 60,000 train and 10,000 test samples. 
The task is to classify these into 10 classes corresponding to coat, shoe, and so on. 
A neural network trained with test accuracy of 87\%. Explanations are generated for the prediction probabilities corresponding to the predicted class for each example. We choose 1000 test examples to generate explanations.
Realistic perturbations were generated using VAEGen~\cite{melime}, a Variational Auto Encoder (VAE) fitted on the training dataset. For random perturbations, we chose $n$ from $\{50,100,200,300,400,500\}$ and kernel sizes were $\{0.43, 0.85, 2.14, 4.27, 6.41\}$. For realistic perturbations we chose $n$ from $\{250, 500, 750, 1000\}$ and the kernel widths were  $\{1.4,2.8,7.0,14.0,21.0\}$. We use $k=\{2,3,4,5\}$ for this dataset.

\paragraph{CIFAR10 (Images):} This dataset has 32 × 32 colored images belonging to 10 different classes. The dataset has
50,000 train and 10,000 test samples. The task is to classify these into 10 classes corresponding to dog, bird, and so on. A residual network with 18 units (ResNet18) was trained with test accuracy of $\sim$ 95\%. Explanations
are generated for the prediction probabilities corresponding to the predicted class for each example. We choose 1000 test examples to generate explanations. Realistic perturbations were generated using VAEGen~\cite{melime}, a Variational Auto Encoder (VAE) fitted on the training dataset. For random perturbations, we chose $n$ from $\{50,100,200,300,400,500\}$ and kernel sizes were $\{0.43, 0.85, 2.14, 4.27, 6.41\}$. For realistic perturbations we chose $n$ from $\{250, 500, 750, 1000\}$ and the kernel widths were  $\{1.4,2.8,7.0,14.0,21.0\}$. We use $k=\{2,3,4,5\}$ for this dataset.

\paragraph{Rotten Tomatoes (Text):} 
This dataset contains 10662 movie reviews from rotten tomatoes website along with their sentiment polarity, i.e., positive or negative reviews and the task is to classify the sentiment of the reviews into positive or negative. 
The review sentences were vectorized using CountVectorizer and TfidfTransformer and a sklearn Naive Bayes classifier was fitted on training dataset which yielded a test accuracy of 75\%. Explanations are generated for the prediction probabilities corresponding to the predicted class for each example.
Realistic perturbations were generated using Word2VecGen~\cite{melime}, wherein  word2vec embeddings are first trained using the training corpus and new sentences are generated by randomly replacing a sentence word whose distance in the embedding space lies within the radius of the neighbourhood. For both random and realistic perturbations, $n$ was chosen from $\{25, 50, 75, 100\}$. The kernel sizes were $\{0.42, 1.06, 2.12, 3.18\}$ for random perturbations (kernel size $0.21$ resulted in numerical issues), and $\{0.21, 0.42, 1.06, 2.12, 3.18\}$ for realistic perturbations. We use $k=\{2,3,4,5\}$ for this dataset.

\section{Results with All Datasets and Hyperparameter Combinations for Random and Realistic Perturbations}
\label{app:all_res}
We present results with all hyperparameter combinations for random and realistic perturbations. Results for LIME with random perturbations (LIME), smoothed LIME (S-LIME), LINEX with random perturbations (LINEX/rand), MeLIME (MeLIME), LINEX with MeLIME-like realistic neighborhoods (LINEX/real), MAPLE (MAPLE), LINEX with MAPLE-like realistic neighborhoods (LINEX/mpl) are presented in figures \ref{fig:CI_cnt_all}-\ref{fig:INFD_kernel_width_all}. The legend for these figures are given in Figure \ref{fig:legend}. 

For the five datasets, we perform ablations by varying one of perturbation neighborhood size (Figures \ref{fig:CI_cnt_all}-\ref{fig:INFD_cnt_all}), number of environments (Figures \ref{fig:CI_num_envs_all}-\ref{fig:INFD_num_envs_all}), and kernel width (Figures \ref{fig:Upsilon_kernel_width_all}-\ref{fig:INFD_kernel_width_all}). Each point in these figures are averaged over all possible values for the two parameters that are not ablated. For example, each point in Figure \ref{fig:CI_cnt_all} is averaged over all possible values for kernel widths and number of environments for a given perturbation neighborhood size. Standard errors of the mean are also plotted in the same color with lesser opacity. Lower values of Infidelity (INFD), Generalized Infidelity (GI), Coefficient Inconsistency (CI) are better whereas for Unidirectionality ($\Upsilon$) and Class Attribution Consistency (CAC) higher values are better.

Figures \ref{fig:CI_cnt_all}-\ref{fig:INFD_cnt_all} show ablations with respect to perturbation neighborhood sizes. Considering all datasets, the stability/recourse metrics (CI, $\Upsilon$, CAC) are clearly better for LINEX compared to its counterparts. For LINEX methods (LINEX/rand, LINEX/real, LINEX/mpl), the metrics get better or stays approximately the same generally as perturbation neighborhood size increases keeping with the intuition that larger perturbation neighborhood sizes should produce explanations that are more stable in the exemplar neighborhood. $\Upsilon$ for FMNIST and CIFAR10 are already good for small perturbation neighborhood sizes possibly because of the quality of MeLIME perturbations. 

Turning to the fidelity metrics (INFD and GI) in tabular datasets, we see that the results still favor LINEX, but less heavily compared to the stability/recourse metrics. This is in line with what we observe in Table \ref{tab:results}. In IRIS and MEPS, LINEX is close to or outperforms the corresponding baselines in the GI measure (except for LINEX/mpl with MEPS). This gap closes a bit with INFD, but we note that GI is a better measure since it estimates how faithful explanations are in a exemplar neighborhood. With the text dataset, LINEX variants are slightly more favored, whereas with the image dataset, the baselines have an edge.

Considering Figures \ref{fig:CI_num_envs_all}-\ref{fig:INFD_num_envs_all}, we see that variations are less stark with respect to number of environments overall for LINEX variants. Note that except for S-LIME, other baselines do not use multiple environments, and hence stay constant. The slight variations in MAPLE are due to the effect of random seeds. In the stability/recourse metrics, again LINEX variants emerge as the clear winner across datssets. With the faithfulness metrics (GI and INFD), in the text dataset, LINEX variants generally perform better, whereas the baselines have a better performance in the image dataset.

Finally, we study the variation of the performance measures with respect to kernel width in Figures \ref{fig:Upsilon_kernel_width_all}-\ref{fig:INFD_kernel_width_all}. We see that the stability/recourse metrics flatten out in all cases with large kernel widths. This behaviour holds true for faithfulness metrics (GI and INFD) as well except in some cases. GI and INFD measures also increase before they flatten out since the fit becomes poorer at larger kernel widths. The stability/recourse metrics become better or remain approximately the same since explanations generally improve or preserve their stability properties as kernel widths increase. Note that very small kernel widths can lead to unexpected behavior that does not fit the trend as seen with the tabular datasets since explanations can become hyper-local. MAPLE and LINEX/mpl stay the same at different kernel widths since they use a different weighting scheme. As with other ablations, we see that LINEX variants are similar or better in stability/recourse metrics overall, while with the faithfulness metrics the results are more mixed.

Note that we do not compute MeLIME perturbations with MEPS since KDE and VAE generators do not work well with categorical data, and do not use compute CAC since the task is regression. Further, the features used in explanations for different test examples are not comparable for random perturbations with FMNIST, CIFAR10 and Rotten Tomatoes, hence we cannot compute CAC for those cases as well. This explains the missing curves/plots.

\begin{figure}[ht]
\centering
\includegraphics[width=0.25\columnwidth]{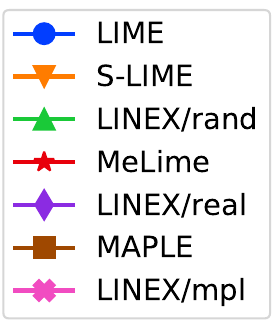}
\caption{Legend for figures \ref{fig:CI_cnt_all}-\ref{fig:INFD_kernel_width_all}}
\label{fig:legend}
\end{figure}


\begin{figure}[!htb]
\centering
    \begin{subfigure}[b]{0.24\textwidth}
        \vspace{0pt}
        \centering
        \captionsetup{justification=centering}
        \includegraphics[width=\textwidth]{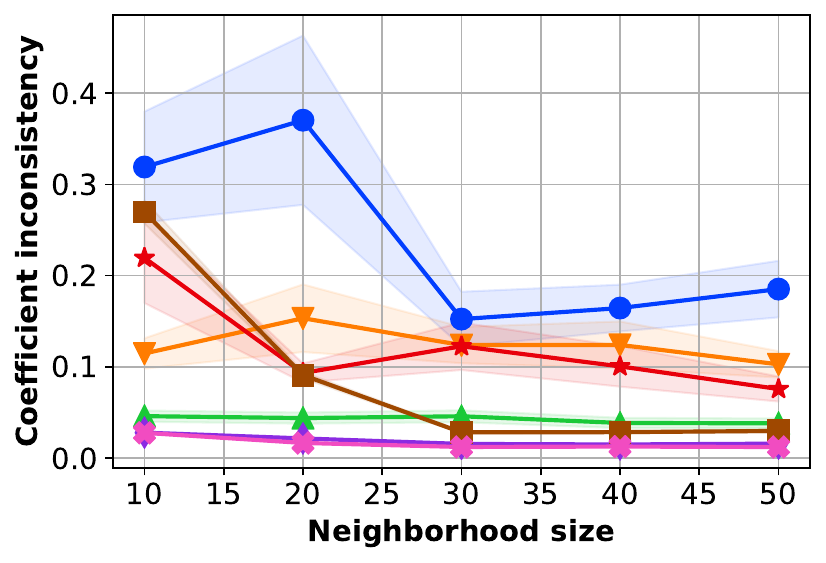}
        \caption{Iris}
        \label{fig:iris_rfc_cnt_CI}
    \end{subfigure}
    \begin{subfigure}[b]{0.24\textwidth}
        \vspace{0pt}
        \centering
        \captionsetup{justification=centering}
        \includegraphics[width=\textwidth]{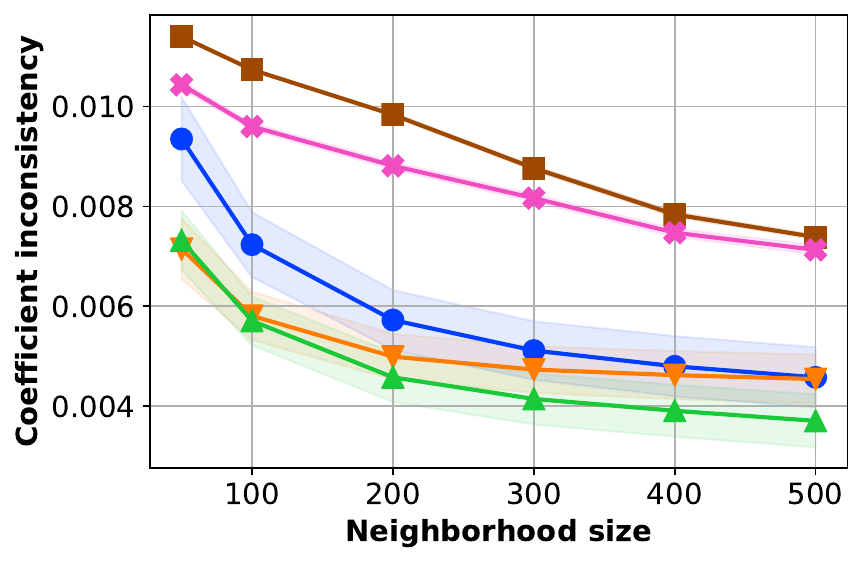}
        \caption{MEPS}
        \label{fig:MEPS_rfr_cnt_CI}
    \end{subfigure}
    \begin{subfigure}[b]{0.24\textwidth}
        \vspace{0pt}
        \centering
        \captionsetup{justification=centering}
        \includegraphics[width=\textwidth]{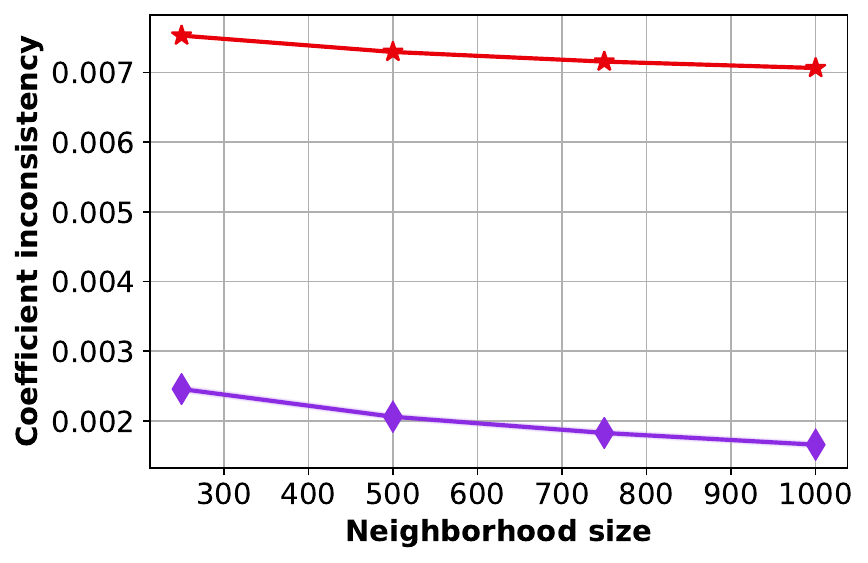}
        \caption{FMNIST}
        \label{fig:fmnist_nn_mp_cnt_CI}
    \end{subfigure}
    \begin{subfigure}[b]{0.24\textwidth}
        \vspace{0pt}
        \centering
        \captionsetup{justification=centering}
        \includegraphics[width=\textwidth]{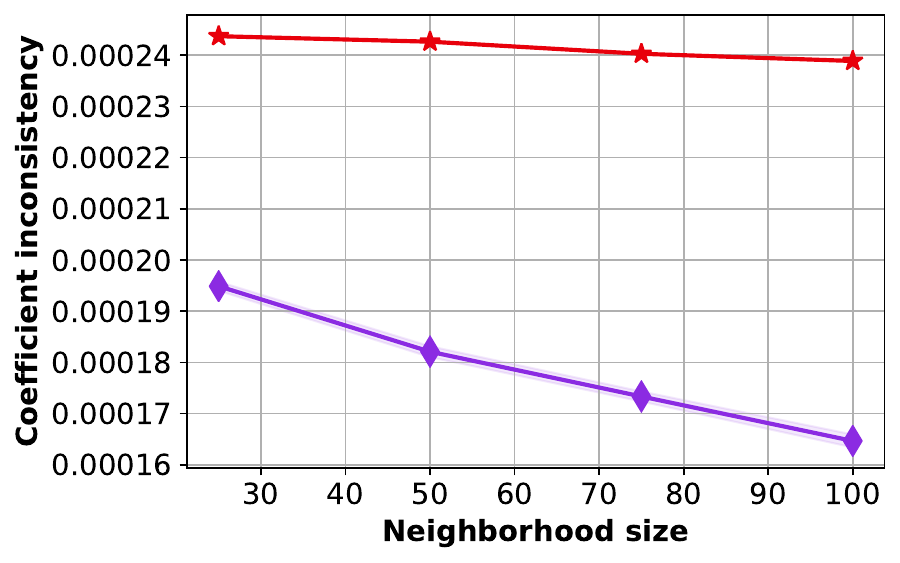}
        \caption{Rotten Tomatoes}
        \label{fig:rotten_mnb_mp_cnt_CI}
    \end{subfigure}
    \begin{subfigure}[b]{0.24\textwidth}
        \vspace{0pt}
        \centering
        \captionsetup{justification=centering}
        \includegraphics[width=\textwidth]{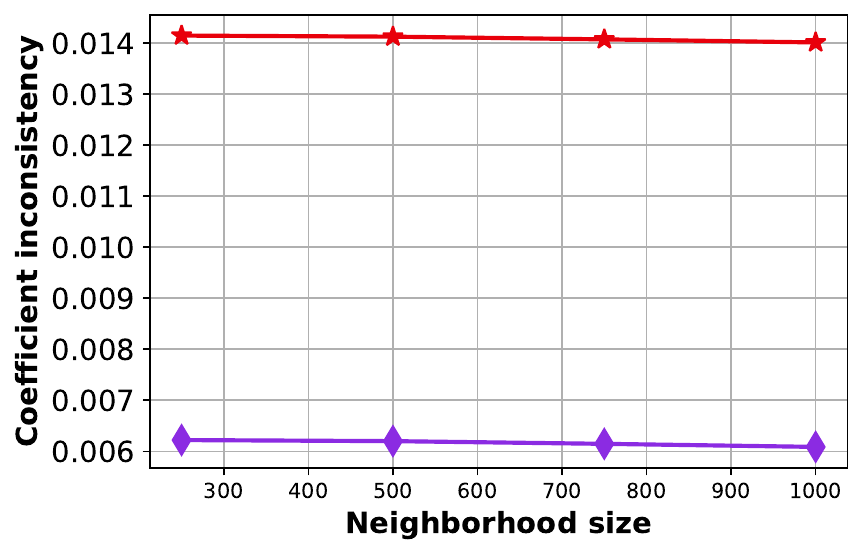}
        \caption{CIFAR10}
        \label{fig:cifar_mp_cnt_CI}
    \end{subfigure}
\caption{Coefficient inconsistency (CI) vs. Perturbation neighborhood size.}
\vspace{-0.5cm}
\label{fig:CI_cnt_all}
\end{figure}

\begin{figure}[!htb]
\centering
    \begin{subfigure}[b]{0.24\textwidth}
        \vspace{0pt}
        \centering
        \captionsetup{justification=centering}
        \includegraphics[width=\textwidth]{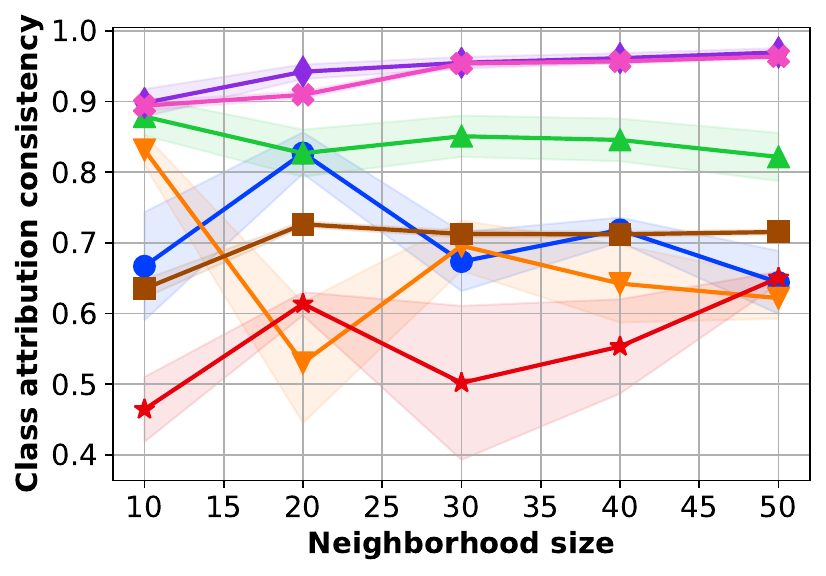}
        \caption{Iris}
        \label{fig:iris_rfc_cnt_CAC}
    \end{subfigure}
    \begin{subfigure}[b]{0.24\textwidth}
        \vspace{0pt}
        \centering
        \captionsetup{justification=centering}
        \includegraphics[width=\textwidth]{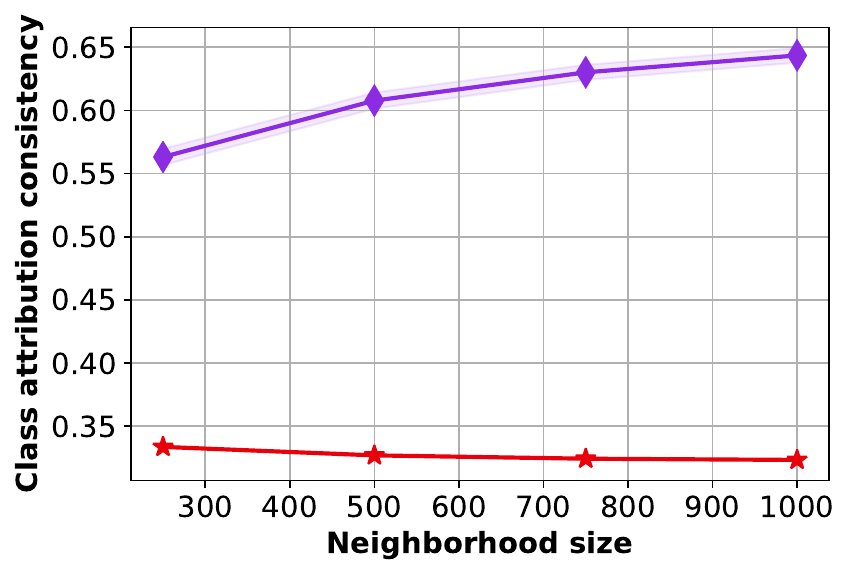}
        \caption{FMNIST}
        \label{fig:fmnist_nn_mp_cnt_CAC}
    \end{subfigure}
    \begin{subfigure}[b]{0.24\textwidth}
        \vspace{0pt}
        \centering
        \captionsetup{justification=centering}
        \includegraphics[width=\textwidth]{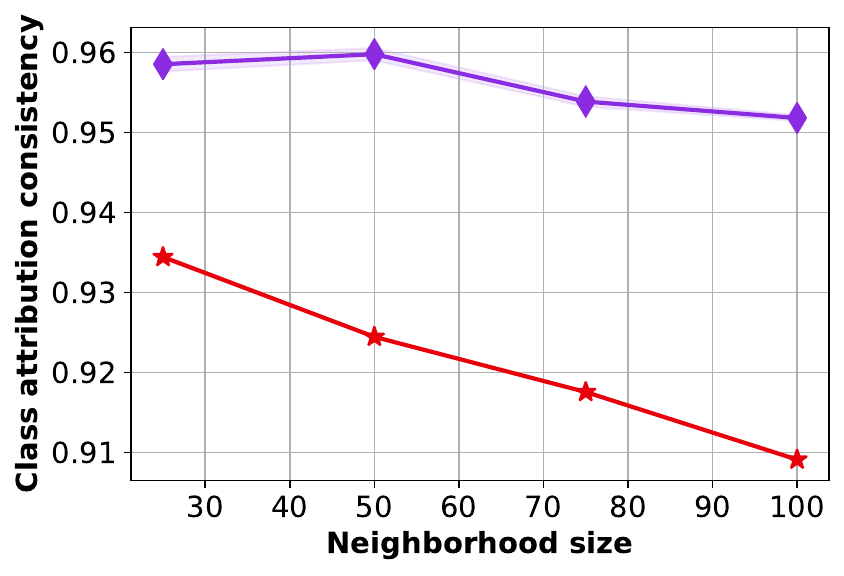}
        \caption{Rotten Tomatoes}
        \label{fig:rotten_mnb_mp_cnt_CAC}
    \end{subfigure}
\caption{Class attribution consistency (CAC) vs. Perturbation neighborhood size.}
\vspace{-0.5cm}
\label{fig:CAC_cnt_all}
\end{figure}

\begin{figure}[!htb]
\centering
    \begin{subfigure}[b]{0.24\textwidth}
        \vspace{0pt}
        \centering
        \captionsetup{justification=centering}
        \includegraphics[width=\textwidth]{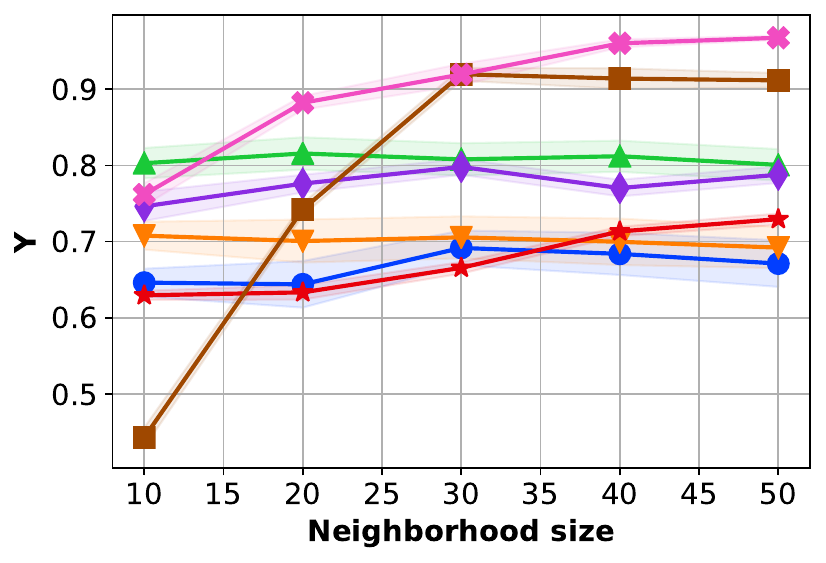}
        \caption{Iris}
        \label{fig:iris_rfc_cnt_Upsilon}
    \end{subfigure}
    \begin{subfigure}[b]{0.24\textwidth}
        \vspace{0pt}
        \centering
        \captionsetup{justification=centering}
        \includegraphics[width=\textwidth]{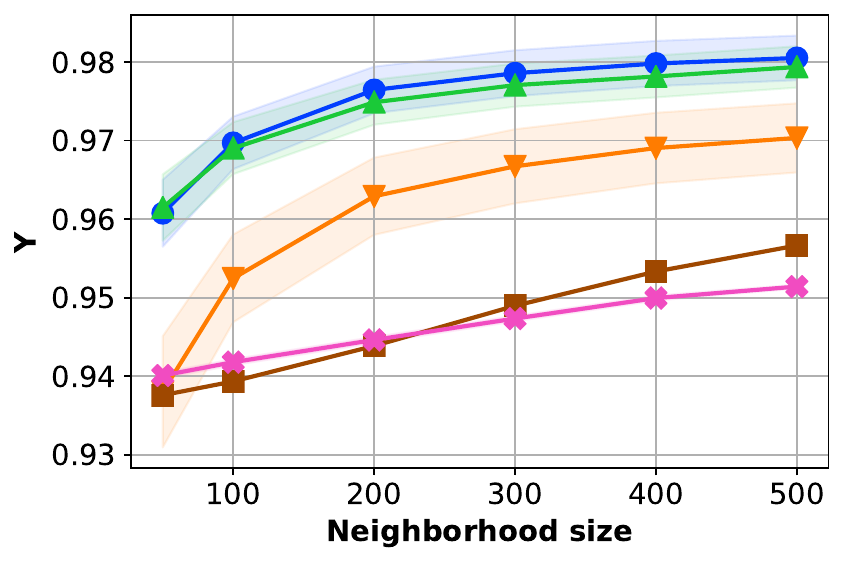}
        \caption{MEPS}
        \label{fig:MEPS_rfr_cnt_Upsilon}
    \end{subfigure}
    \begin{subfigure}[b]{0.24\textwidth}
        \vspace{0pt}
        \centering
        \captionsetup{justification=centering}
        \includegraphics[width=\textwidth]{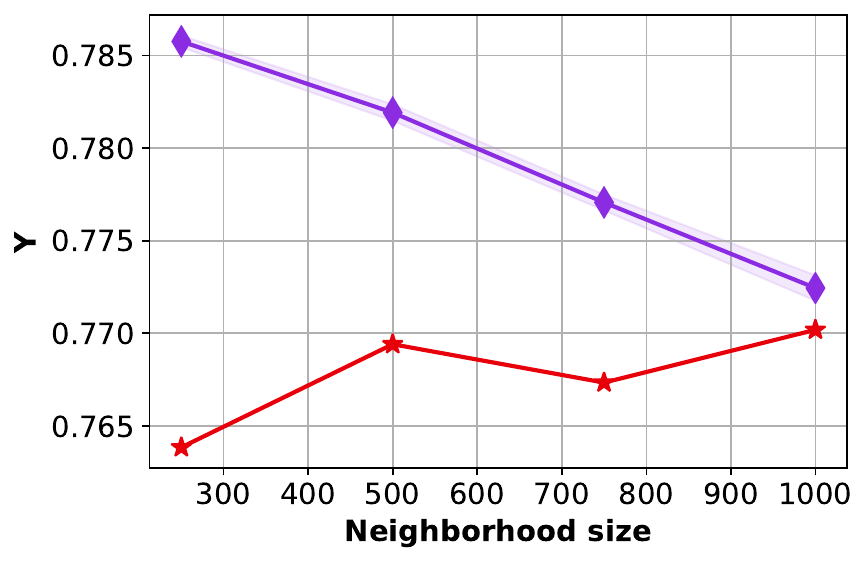}
        \caption{FMNIST}
        \label{fig:fmnist_nn_mp_cnt_Upsilon}
    \end{subfigure}
    \begin{subfigure}[b]{0.24\textwidth}
        \vspace{0pt}
        \centering
        \captionsetup{justification=centering}
        \includegraphics[width=\textwidth]{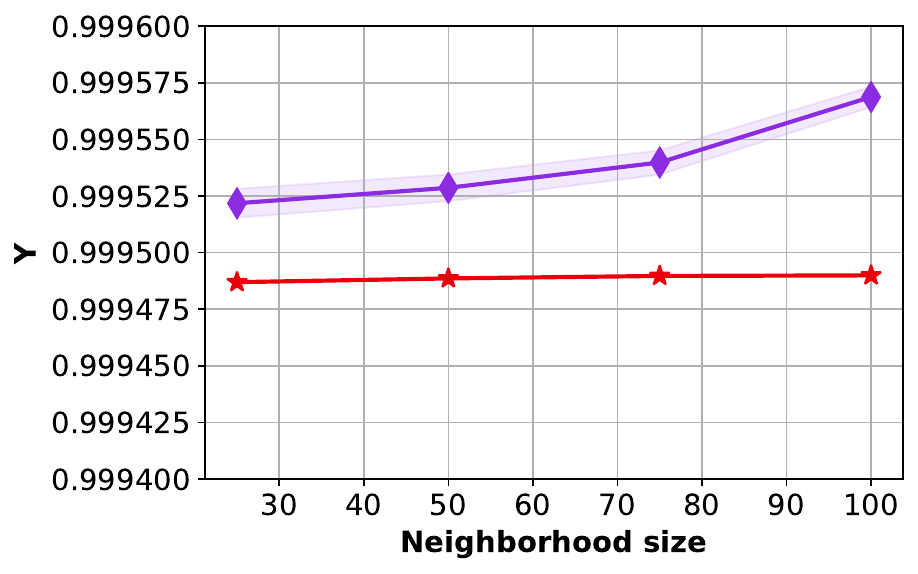}
        \caption{Rotten Tomatoes}
        \label{fig:rotten_mnb_mp_cnt_Upsilon}
    \end{subfigure}
    \begin{subfigure}[b]{0.24\textwidth}
        \vspace{0pt}
        \centering
        \captionsetup{justification=centering}
        \includegraphics[width=\textwidth]{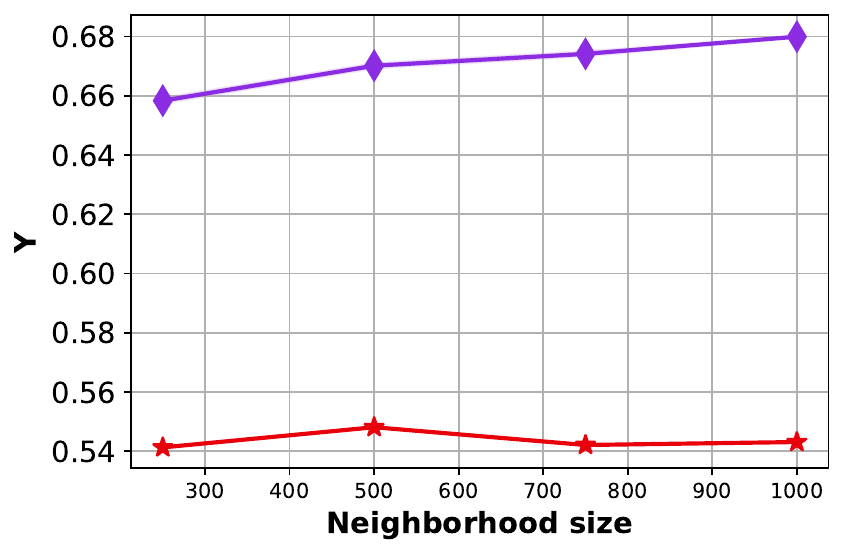}
        \caption{CIFAR10}
        \label{fig:cifar_mp_cnt_Upsilon}
    \end{subfigure}
\caption{Unidirectionality ($\Upsilon$) vs. Perturbation neighborhood size.}
\vspace{-0.5cm}
\label{fig:Upsilon_cnt_all}
\end{figure}
\begin{figure}[!htb]
\centering
    \begin{subfigure}[b]{0.24\textwidth}
        \vspace{0pt}
        \centering
        \captionsetup{justification=centering}
        \includegraphics[width=\textwidth]{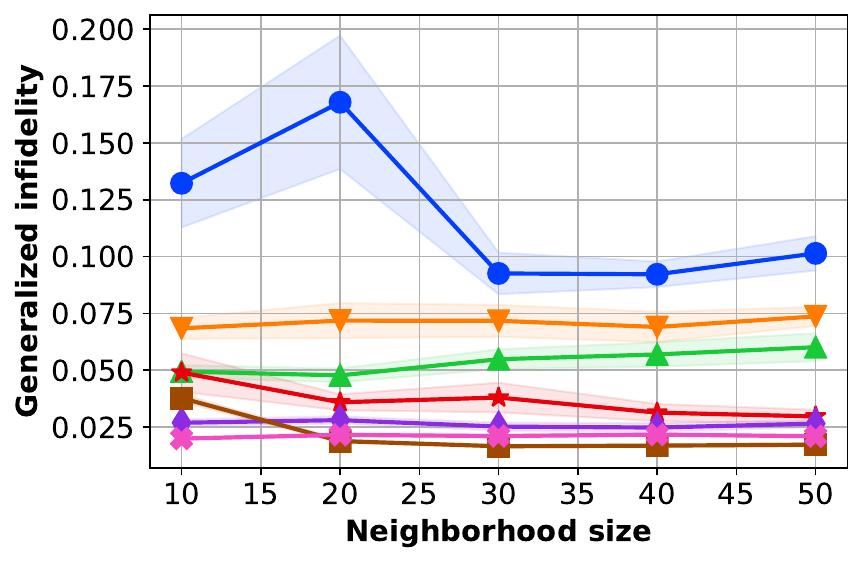}
        \caption{Iris}
        \label{fig:iris_rfc_cnt_GI}
    \end{subfigure}
    \begin{subfigure}[b]{0.24\textwidth}
        \vspace{0pt}
        \centering
        \captionsetup{justification=centering}
        \includegraphics[width=\textwidth]{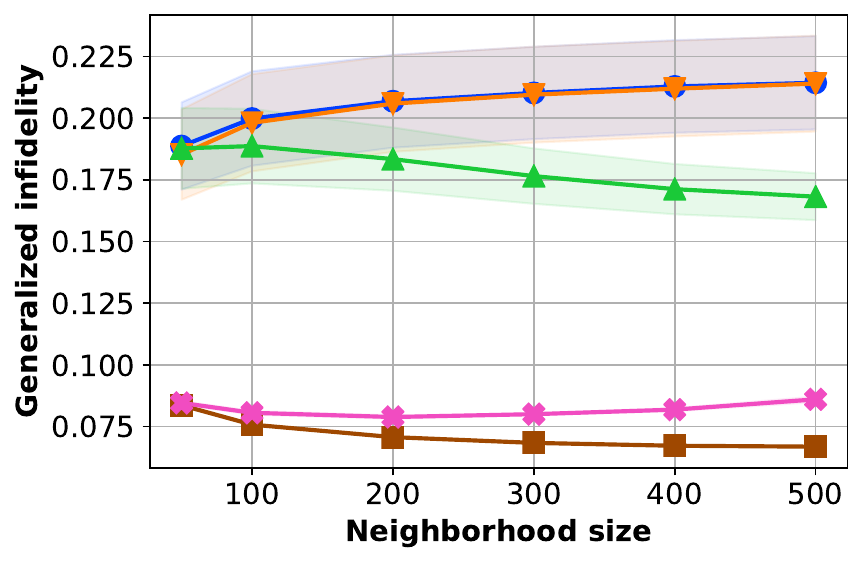}
        \caption{MEPS}
        \label{fig:MEPS_rfr_cnt_GI}
    \end{subfigure}
    \begin{subfigure}[b]{0.24\textwidth}
        \vspace{0pt}
        \centering
        \captionsetup{justification=centering}
        \includegraphics[width=\textwidth]{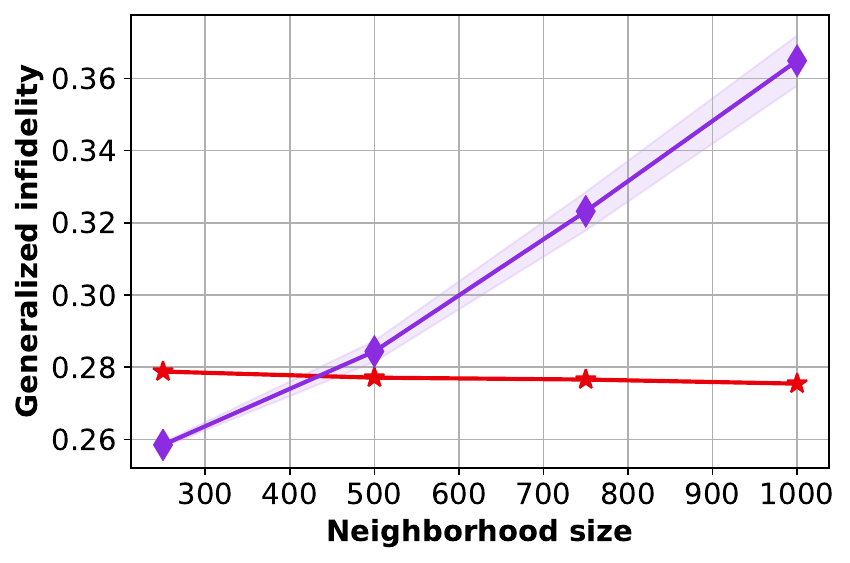}
        \caption{FMNIST}
        \label{fig:fmnist_nn_mp_cnt_GI}
    \end{subfigure}
    \begin{subfigure}[b]{0.24\textwidth}
        \vspace{0pt}
        \centering
        \captionsetup{justification=centering}
        \includegraphics[width=\textwidth]{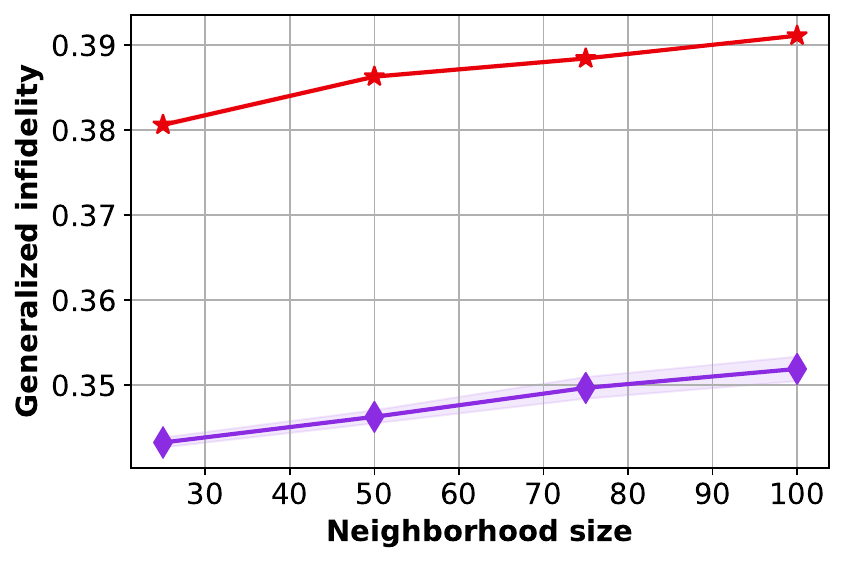}
        \caption{Rotten Tomatoes}
        \label{fig:rotten_mnb_mp_cnt_GI}
    \end{subfigure}
    \begin{subfigure}[b]{0.24\textwidth}
        \vspace{0pt}
        \centering
        \captionsetup{justification=centering}
        \includegraphics[width=\textwidth]{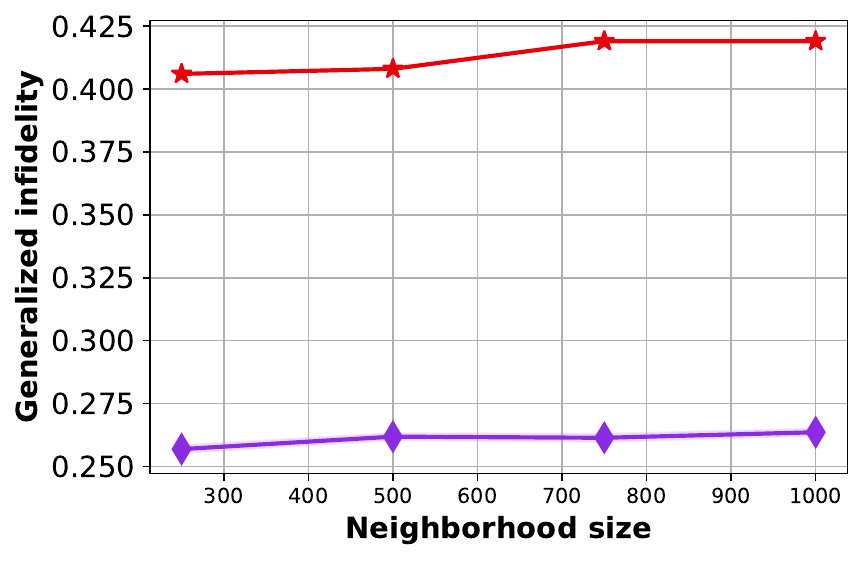}
        \caption{CIFAR10}
        \label{fig:cifar_mp_cnt_GI}
    \end{subfigure}

\caption{Generalized infidelity (GI) vs. Perturbation neighborhood size.}
\vspace{-0.5cm}
\label{fig:GI_cnt_all}
\end{figure}

\begin{figure}[!htb]
\centering
    \begin{subfigure}[b]{0.24\textwidth}
        \vspace{0pt}
        \centering
        \captionsetup{justification=centering}
        \includegraphics[width=\textwidth]{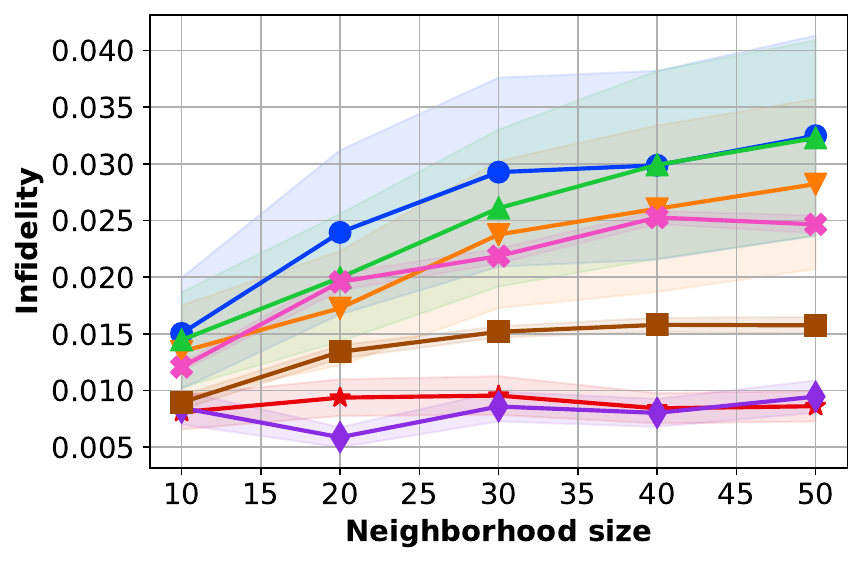}
        \caption{Iris}
        \label{fig:iris_rfc_cnt_INFD}
    \end{subfigure}
    \begin{subfigure}[b]{0.24\textwidth}
        \vspace{0pt}
        \centering
        \captionsetup{justification=centering}
        \includegraphics[width=\textwidth]{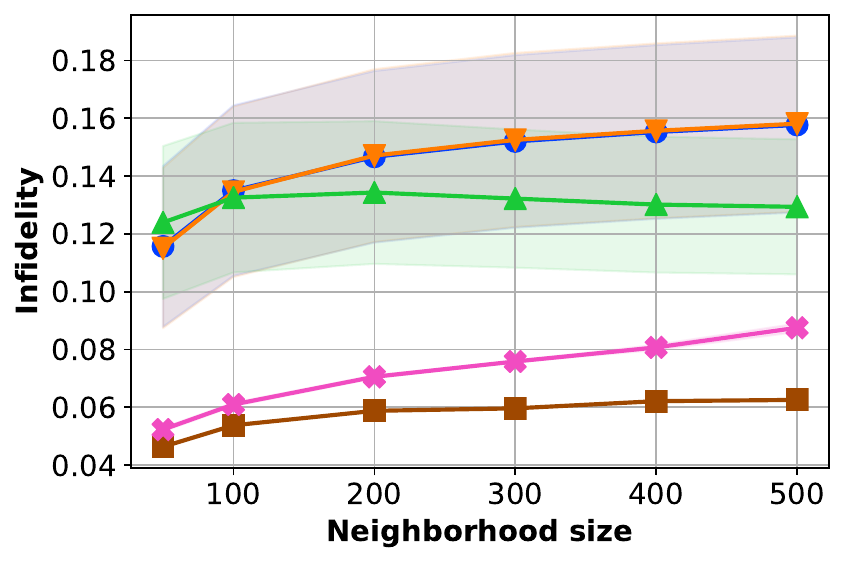}
        \caption{MEPS}
        \label{fig:MEPS_rfr_cnt_INFD}
    \end{subfigure}
    \begin{subfigure}[b]{0.24\textwidth}
        \vspace{0pt}
        \centering
        \captionsetup{justification=centering}
        \includegraphics[width=\textwidth]{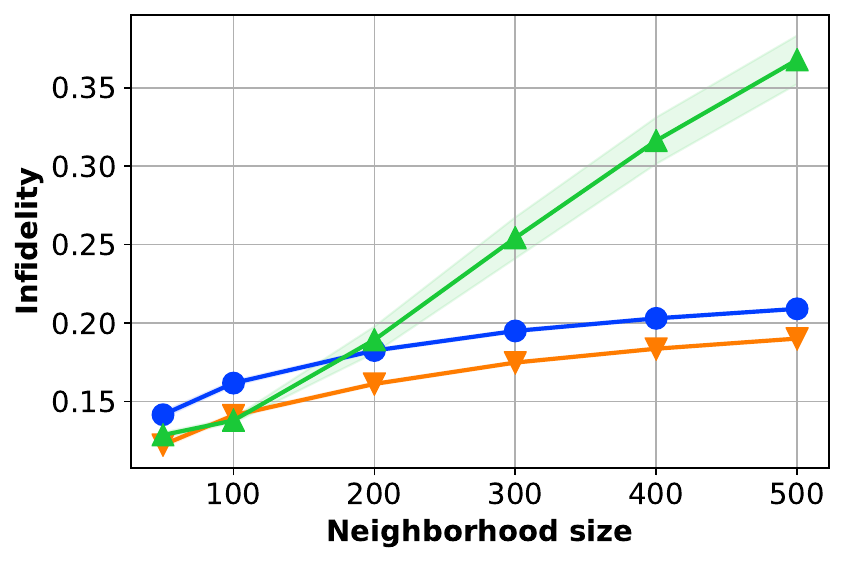}
        \caption{FMNIST (random)}
        \label{fig:fmnist_nn_bp_cnt_INFD}
    \end{subfigure}
    \begin{subfigure}[b]{0.24\textwidth}
        \vspace{0pt}
        \centering
        \captionsetup{justification=centering}
        \includegraphics[width=\textwidth]{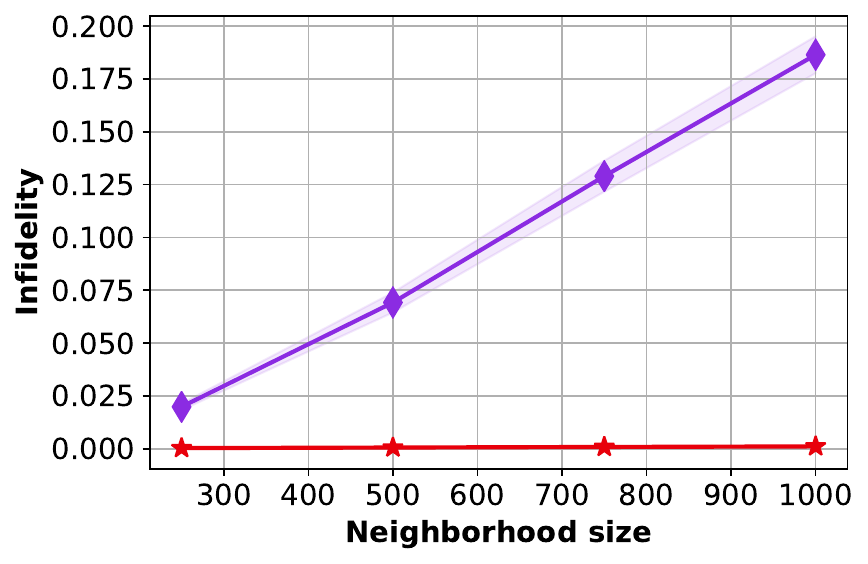}
        \caption{FMNIST (realistic)}
        \label{fig:fmnist_nn_mp_cnt_INFD}
    \end{subfigure}
    \begin{subfigure}[b]{0.24\textwidth}
        \vspace{0pt}
        \centering
        \captionsetup{justification=centering}
        \includegraphics[width=\textwidth]{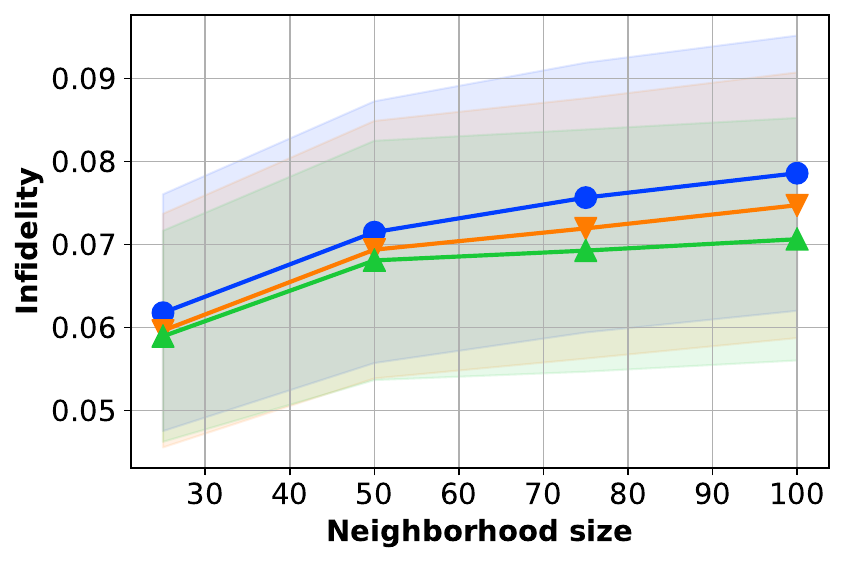}
        \caption{Rotten Tomatoes (random)}
        \label{fig:rotten_mnb_bp_cnt_INFD}
    \end{subfigure}
    \begin{subfigure}[b]{0.24\textwidth}
    \vspace{0pt}
    \centering
    \captionsetup{justification=centering}
    \includegraphics[width=\textwidth]{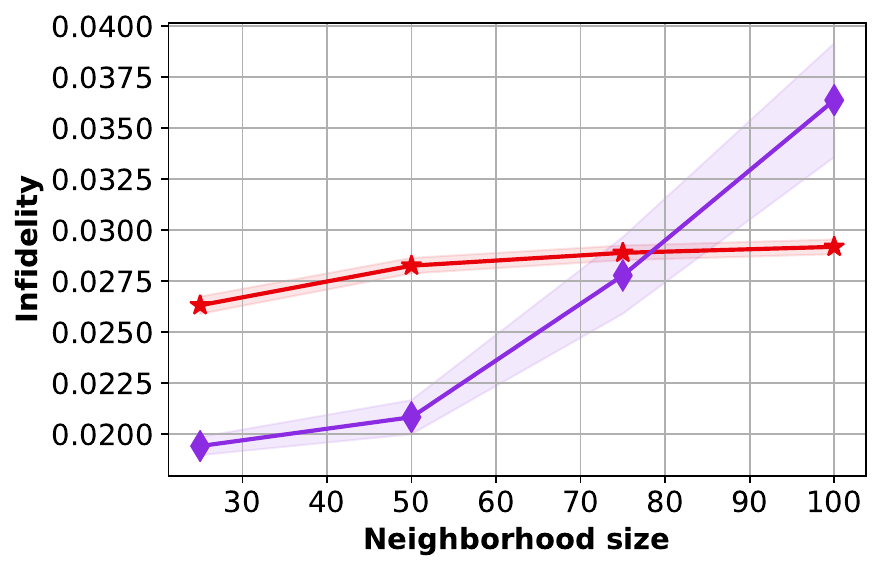}
    \caption{Rotten Tomatoes (realistic)}
    \label{fig:rotten_mnb_mp_cnt_INFD}
    \end{subfigure}
    \begin{subfigure}[b]{0.24\textwidth}
    \vspace{0pt}
    \centering
    \captionsetup{justification=centering}
    \includegraphics[width=\textwidth]{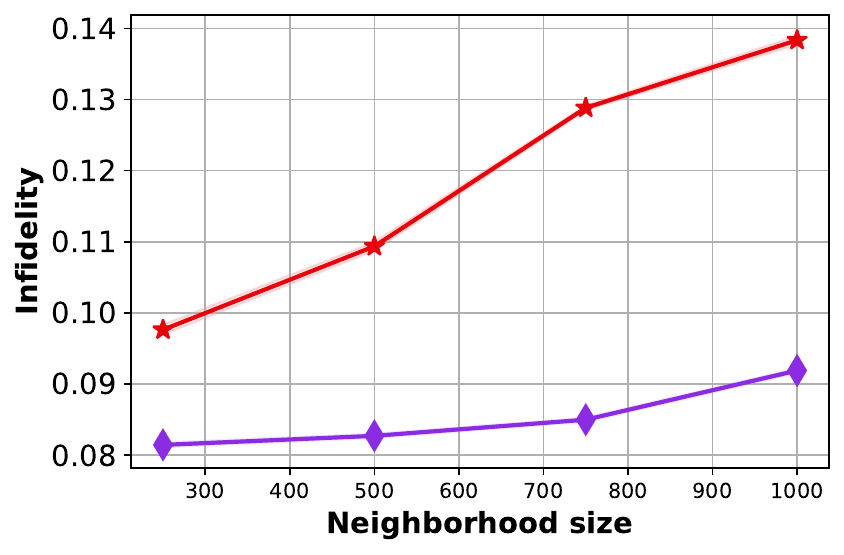}
    \caption{CIFAR10\\\phantom{test}}
    \label{fig:cifar_mp_cnt_INFD}
    \end{subfigure}
\caption{Infidelity (INFD) vs. Perturbation neighborhood size.}
\label{fig:INFD_cnt_all}
\end{figure}

\newpage


\begin{figure}[!htb]
\centering
    \begin{subfigure}[b]{0.24\textwidth}
        \vspace{0pt}
        \centering
        \captionsetup{justification=centering}
        \includegraphics[width=\textwidth]{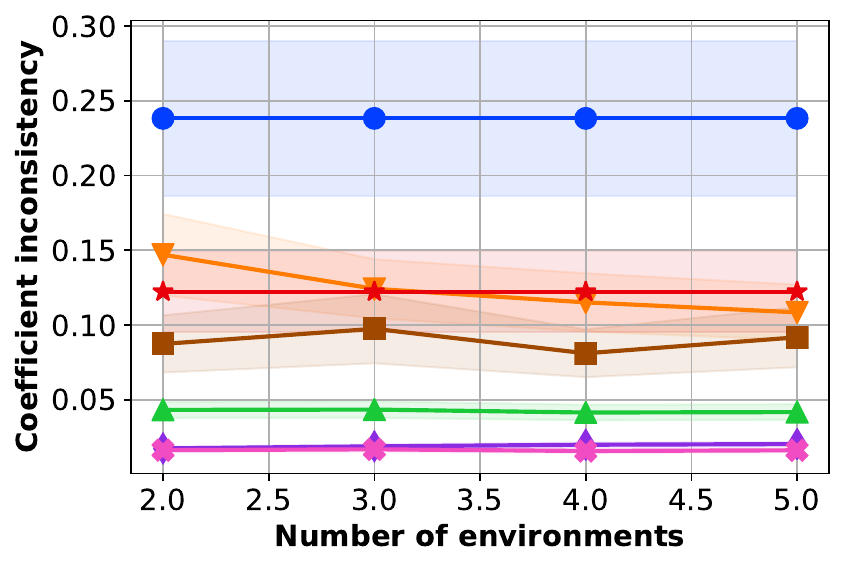}
        \caption{Iris}
        \label{fig:iris_rfc_num_envs_CI}
    \end{subfigure}
    \begin{subfigure}[b]{0.24\textwidth}
        \vspace{0pt}
        \centering
        \captionsetup{justification=centering}
        \includegraphics[width=\textwidth]{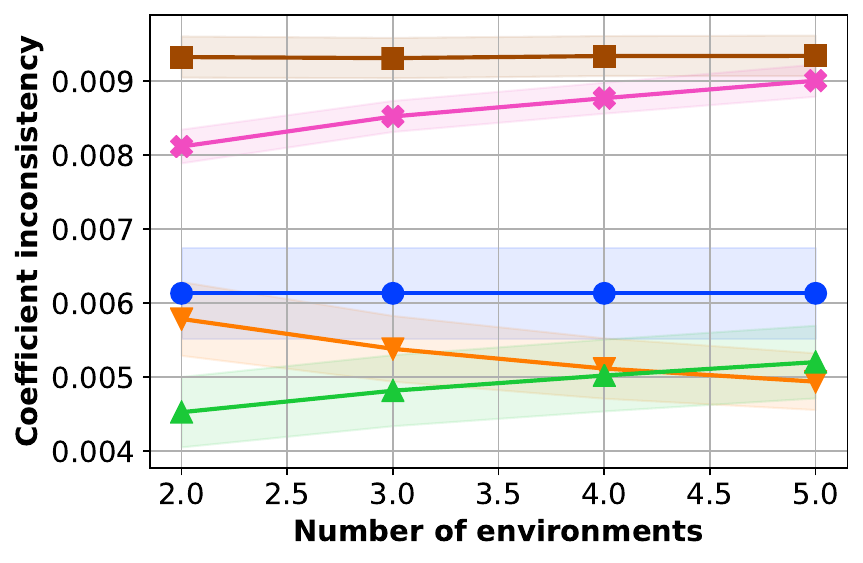}
        \caption{MEPS}
        \label{fig:MEPS_rfr_num_envs_CI}
    \end{subfigure}
    \begin{subfigure}[b]{0.24\textwidth}
        \vspace{0pt}
        \centering
        \captionsetup{justification=centering}
        \includegraphics[width=\textwidth]{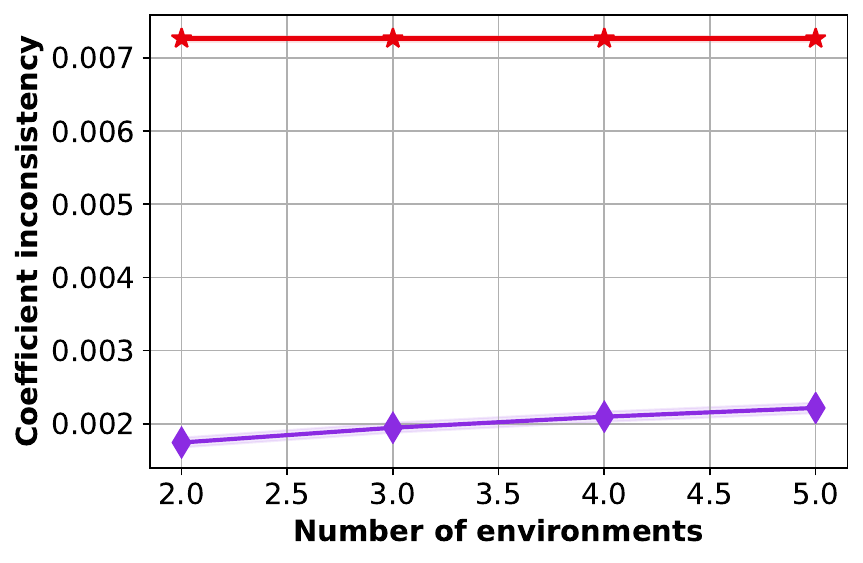}
        \caption{FMNIST}
        \label{fig:fmnist_nn_mp_num_envs_CI}
    \end{subfigure}
    \begin{subfigure}[b]{0.24\textwidth}
        \vspace{0pt}
        \centering
        \captionsetup{justification=centering}
        \includegraphics[width=\textwidth]{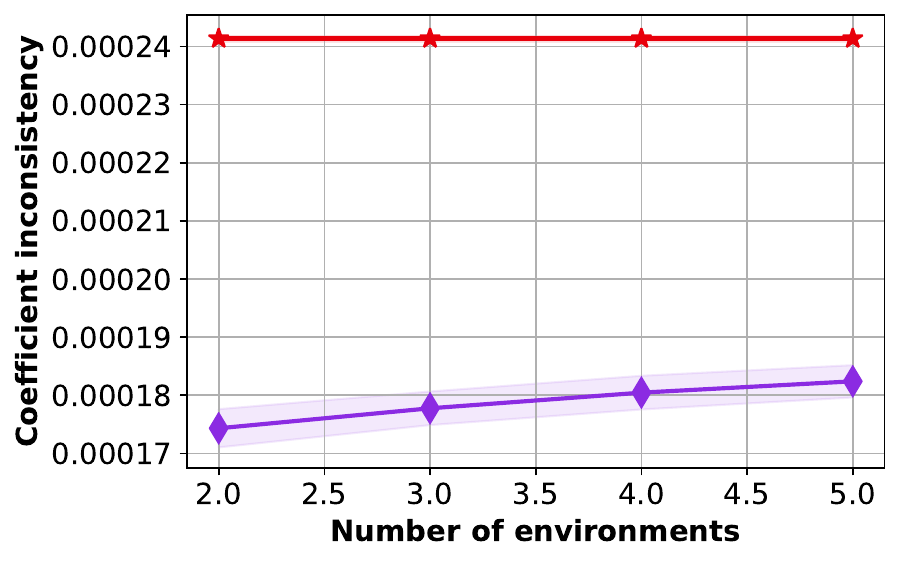}
        \caption{Rotten Tomatoes}
        \label{fig:rotten_mnb_mp_num_envs_CI}
    \end{subfigure}
    \begin{subfigure}[b]{0.24\textwidth}
        \vspace{0pt}
        \centering
        \captionsetup{justification=centering}
        \includegraphics[width=\textwidth]{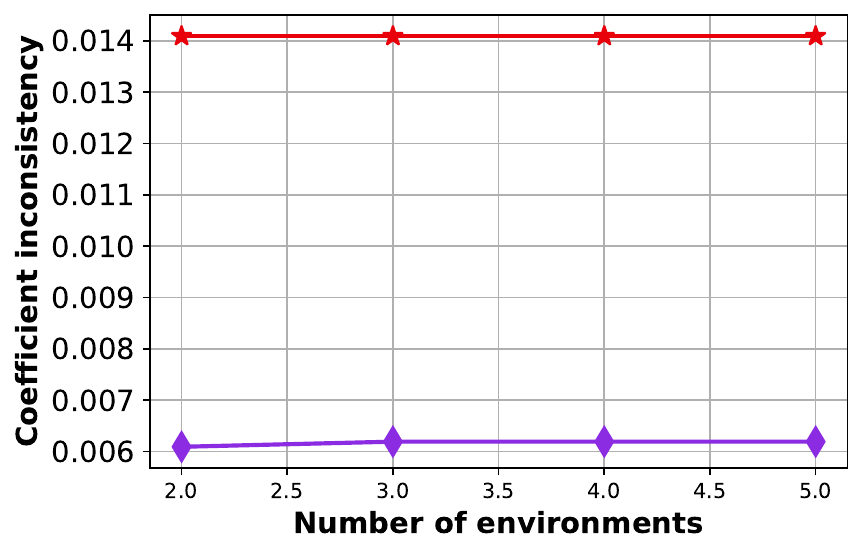}
        \caption{CIFAR10}
        \label{fig:cifar_mp_num_envs_CI}
    \end{subfigure}
\caption{Coefficient inconsistency (CI) vs. Number of environments.}
\label{fig:CI_num_envs_all}
\end{figure}

\begin{figure}[!htb]
\centering
    \begin{subfigure}[b]{0.24\textwidth}
        \vspace{0pt}
        \centering
        \captionsetup{justification=centering}
        \includegraphics[width=\textwidth]{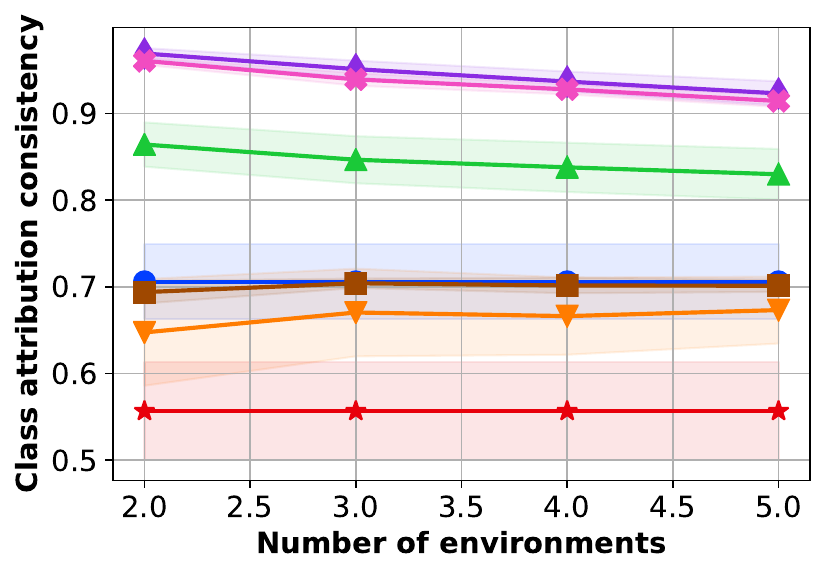}
        \caption{Iris}
        \label{fig:iris_rfc_num_envs_CAC}
    \end{subfigure}
    \begin{subfigure}[b]{0.24\textwidth}
        \vspace{0pt}
        \centering
        \captionsetup{justification=centering}
        \includegraphics[width=\textwidth]{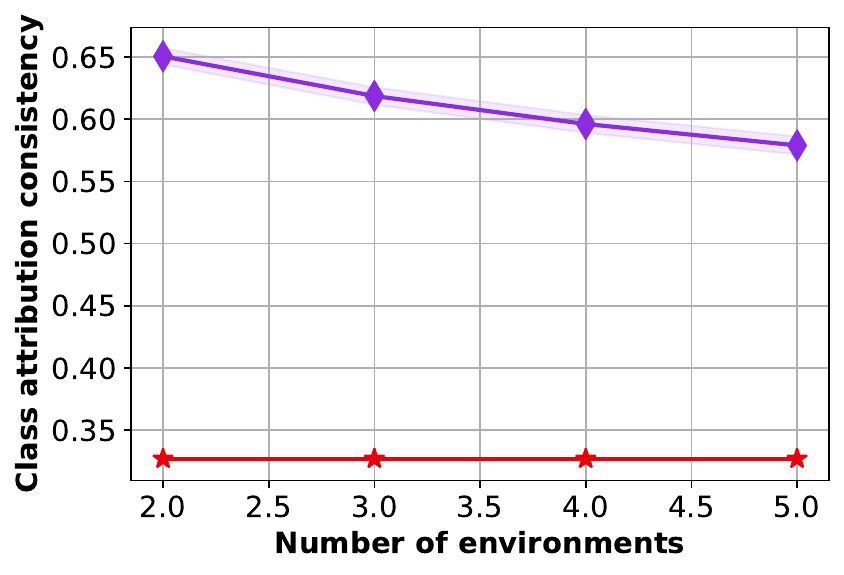}
        \caption{FMNIST}
        \label{fig:fmnist_nn_mp_num_envs_CAC}
    \end{subfigure}
    \begin{subfigure}[b]{0.24\textwidth}
        \vspace{0pt}
        \centering
        \captionsetup{justification=centering}
        \includegraphics[width=\textwidth]{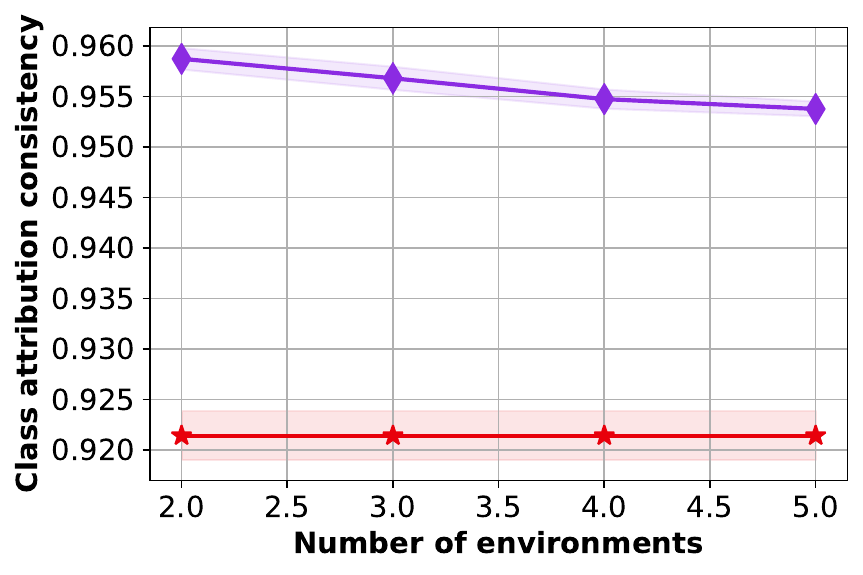}
        \caption{Rotten Tomatoes}
        \label{fig:rotten_mnb_mp_num_envs_CAC}
    \end{subfigure}
\caption{Class attribution consistency (CAC) vs. Number of environments.}
\label{fig:CAC_num_envs_all}
\end{figure}

\begin{figure}[!htb]
\centering
    \begin{subfigure}[b]{0.24\textwidth}
        \vspace{0pt}
        \centering
        \captionsetup{justification=centering}
        \includegraphics[width=\textwidth]{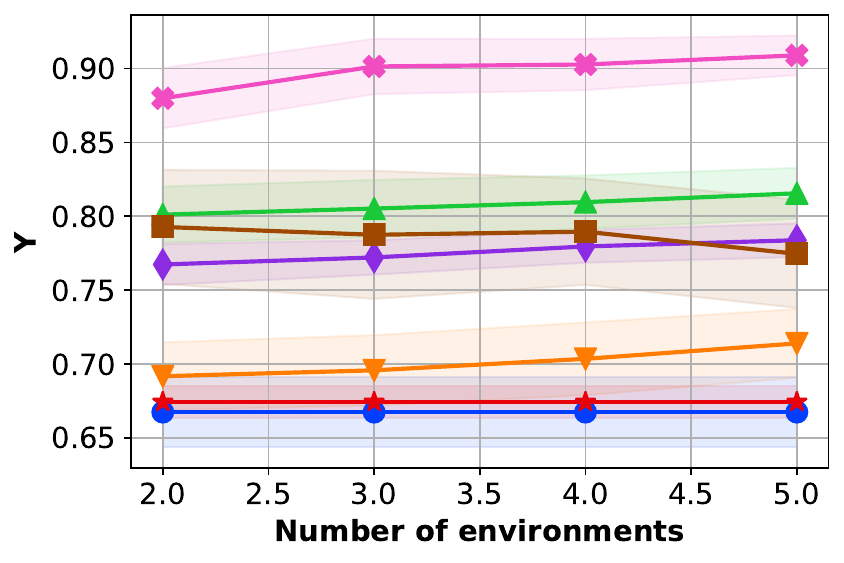}
        \caption{Iris}
        \label{fig:iris_rfc_num_envs_Upsilon}
    \end{subfigure}
    \begin{subfigure}[b]{0.24\textwidth}
        \vspace{0pt}
        \centering
        \captionsetup{justification=centering}
        \includegraphics[width=\textwidth]{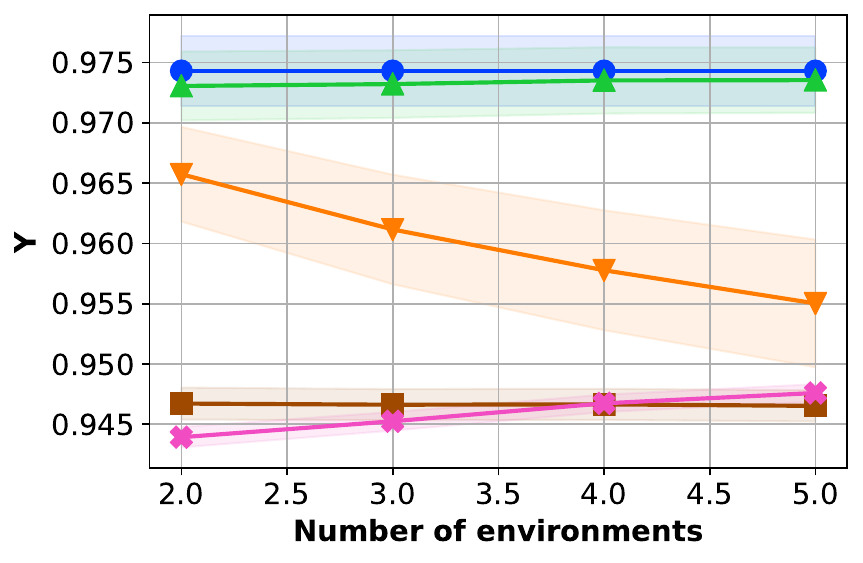}
        \caption{MEPS}
        \label{fig:MEPS_rfr_num_envs_Upsilon}
    \end{subfigure}
    \begin{subfigure}[b]{0.24\textwidth}
        \vspace{0pt}
        \centering
        \captionsetup{justification=centering}
        \includegraphics[width=\textwidth]{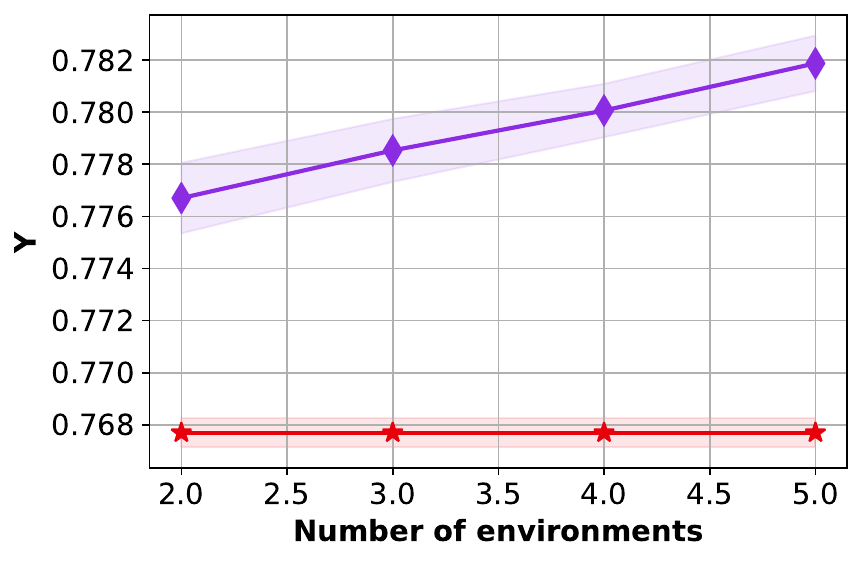}
        \caption{FMNIST}
        \label{fig:fmnist_nn_mp_num_envs_Upsilon}
    \end{subfigure}
    \begin{subfigure}[b]{0.24\textwidth}
        \vspace{0pt}
        \centering
        \captionsetup{justification=centering}
        \includegraphics[width=\textwidth]{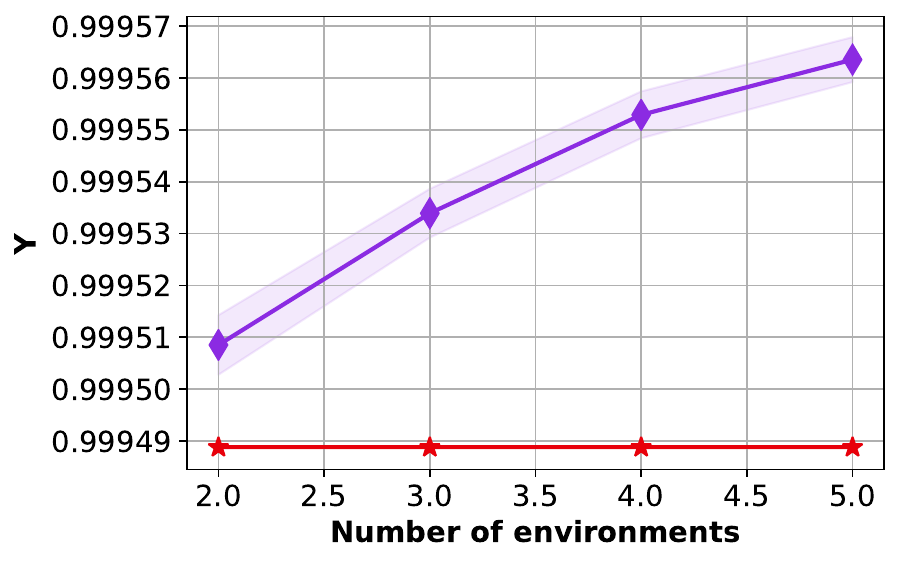}
        \caption{Rotten Tomatoes}
        \label{fig:rotten_mnb_mp_num_envs_Upsilon}
    \end{subfigure}
    \begin{subfigure}[b]{0.24\textwidth}
        \vspace{0pt}
        \centering
        \captionsetup{justification=centering}
        \includegraphics[width=\textwidth]{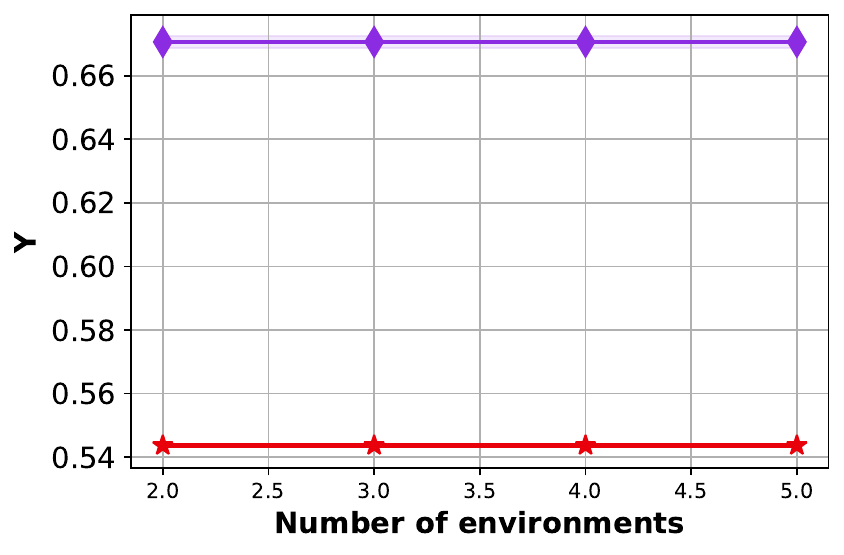}
        \caption{CIFAR10}
        \label{fig:cifar_mp_num_envs_Upsilon}
    \end{subfigure}
\caption{Unidirectionality ($\Upsilon$) vs. Number of environments.}
\vspace{-0.5cm}
\label{fig:Upsilon_num_envs_all}
\end{figure}

\begin{figure}[!htb]
\centering
    \begin{subfigure}[b]{0.24\textwidth}
        \vspace{0pt}
        \centering
        \captionsetup{justification=centering}
        \includegraphics[width=\textwidth]{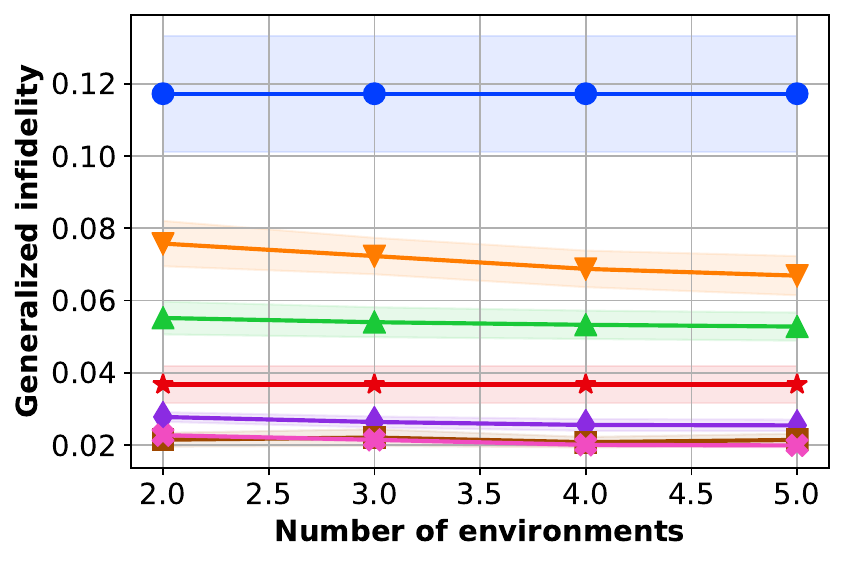}
        \caption{Iris}
        \label{fig:iris_rfc_num_envs_GI}
    \end{subfigure}
    \begin{subfigure}[b]{0.24\textwidth}
        \vspace{0pt}
        \centering
        \captionsetup{justification=centering}
        \includegraphics[width=\textwidth]{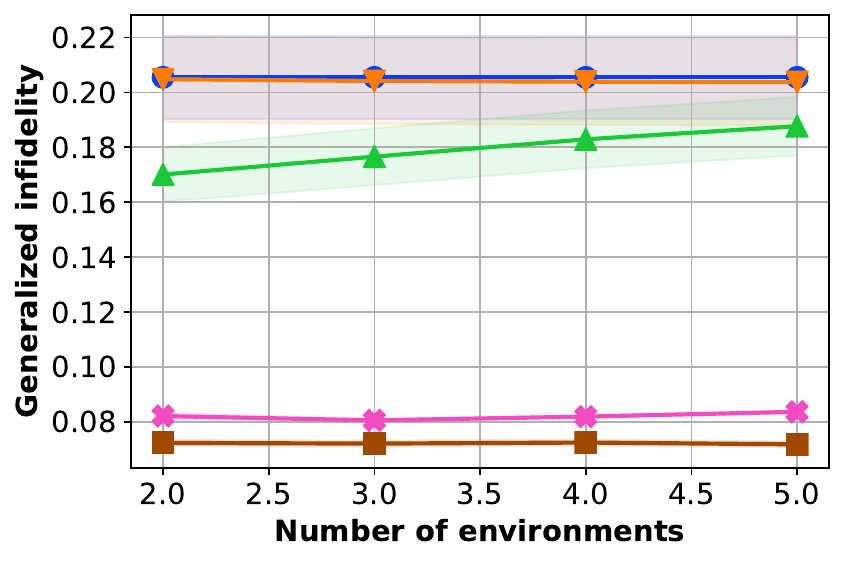}
        \caption{MEPS}
        \label{fig:MEPS_rfr_num_envs_GI}
    \end{subfigure}
    \begin{subfigure}[b]{0.24\textwidth}
        \vspace{0pt}
        \centering
        \captionsetup{justification=centering}
        \includegraphics[width=\textwidth]{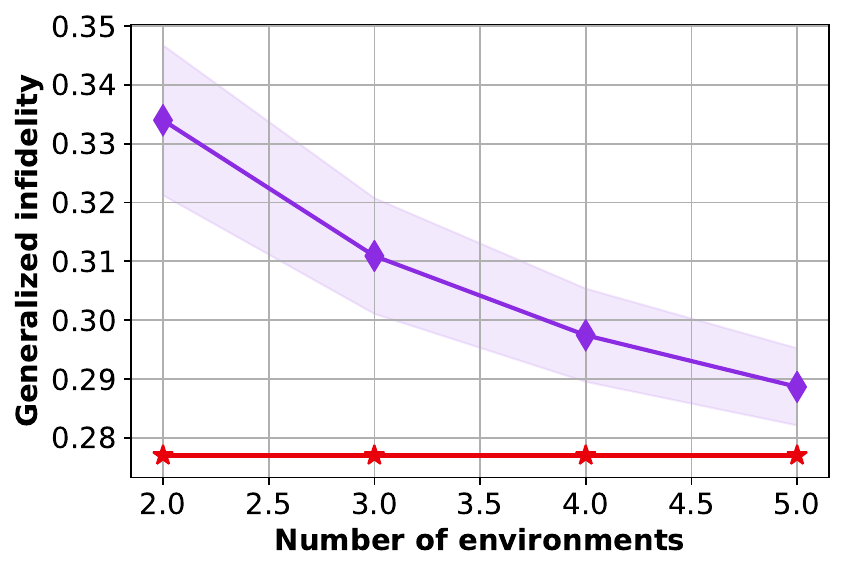}
        \caption{FMNIST}
        \label{fig:fmnist_nn_mp_num_envs_GI}
    \end{subfigure}
    \begin{subfigure}[b]{0.24\textwidth}
        \vspace{0pt}
        \centering
        \captionsetup{justification=centering}
        \includegraphics[width=\textwidth]{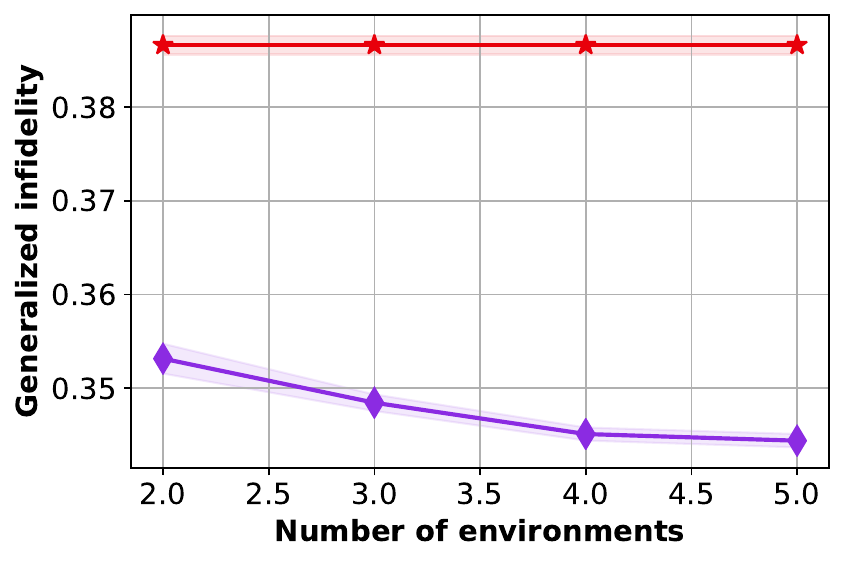}
        \caption{Rotten Tomatoes}
        \label{fig:rotten_mnb_mp_num_envs_GI}
    \end{subfigure}
    \begin{subfigure}[b]{0.24\textwidth}
        \vspace{0pt}
        \centering
        \captionsetup{justification=centering}
        \includegraphics[width=\textwidth]{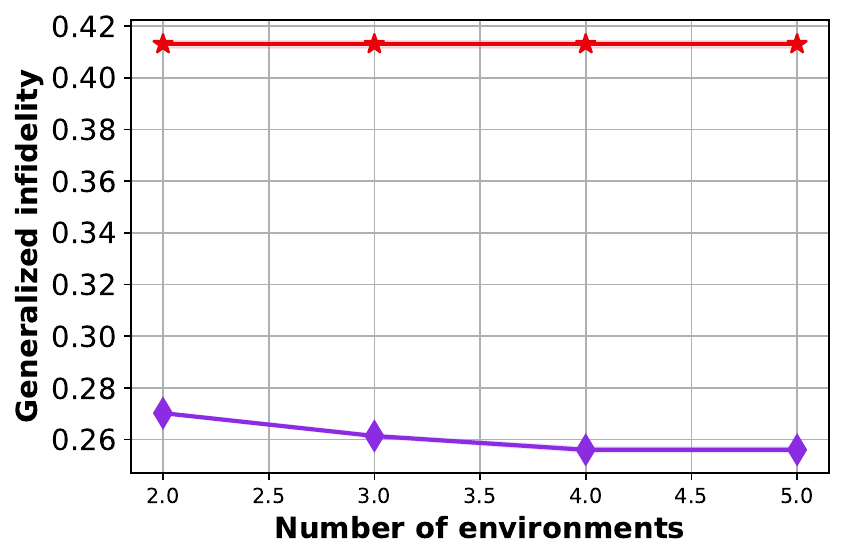}
        \caption{CIFAR10}
        \label{fig:cifar_mp_num_envs_GI}
    \end{subfigure}
    
\caption{Generalized infidelity (GI) vs. Number of environments.}
\vspace{-0.5cm}
\label{fig:GI_num_envs_all}
\end{figure}

\begin{figure}[!htb]
\centering
    \begin{subfigure}[b]{0.24\textwidth}
        \vspace{0pt}
        \centering
        \captionsetup{justification=centering}
        \includegraphics[width=\textwidth]{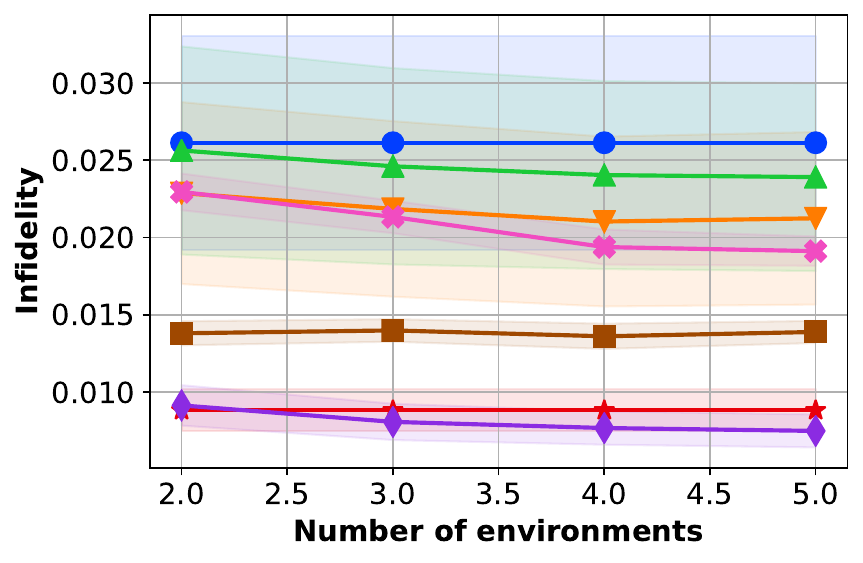}
        \caption{Iris}
        \label{fig:iris_rfc_num_envs_INFD}
    \end{subfigure}
    \begin{subfigure}[b]{0.24\textwidth}
        \vspace{0pt}
        \centering
        \captionsetup{justification=centering}
        \includegraphics[width=\textwidth]{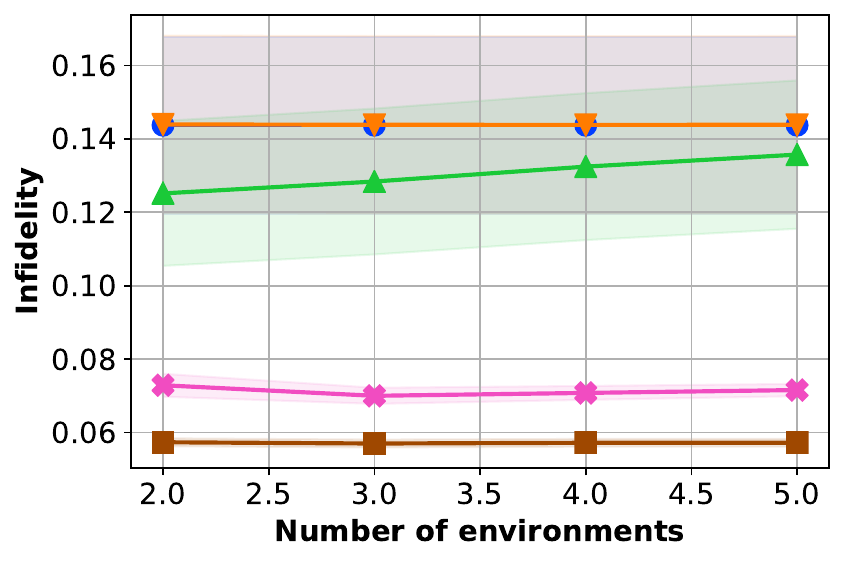}
        \caption{MEPS}
        \label{fig:MEPS_rfr_num_envs_INFD}
    \end{subfigure}
    \begin{subfigure}[b]{0.24\textwidth}
        \vspace{0pt}
        \centering
        \captionsetup{justification=centering}
        \includegraphics[width=\textwidth]{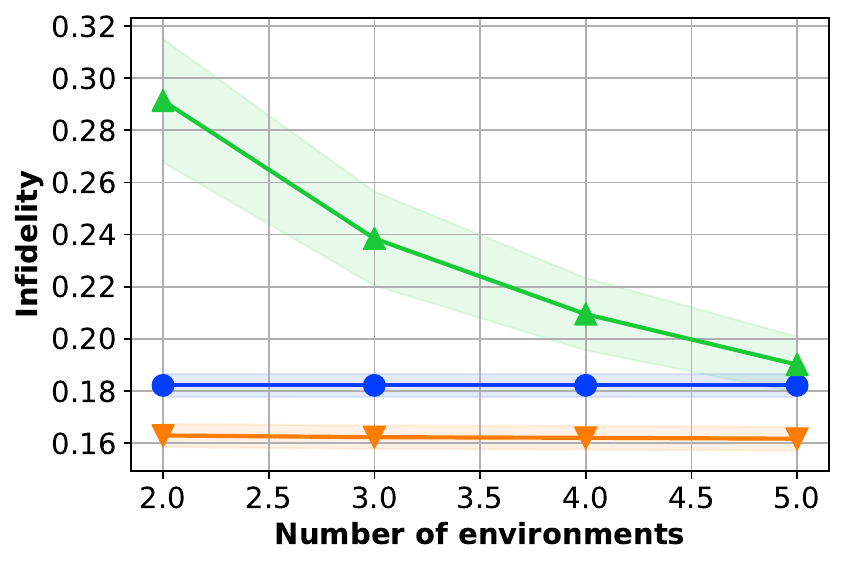}
        \caption{FMNIST (random)}
        \label{fig:fmnist_nn_bp_num_envs_INFD}
    \end{subfigure}
    \begin{subfigure}[b]{0.24\textwidth}
        \vspace{0pt}
        \centering
        \captionsetup{justification=centering}
        \includegraphics[width=\textwidth]{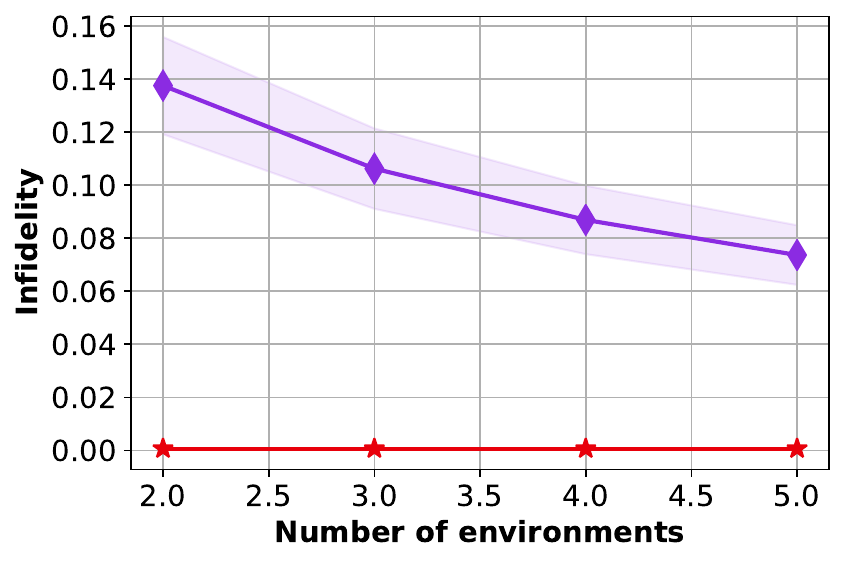}
        \caption{FMNIST (realistic)}
        \label{fig:fmnist_nn_mp_num_envs_INFD}
    \end{subfigure}
    \begin{subfigure}[b]{0.24\textwidth}
        \vspace{0pt}
        \centering
        \captionsetup{justification=centering}
        \includegraphics[width=\textwidth]{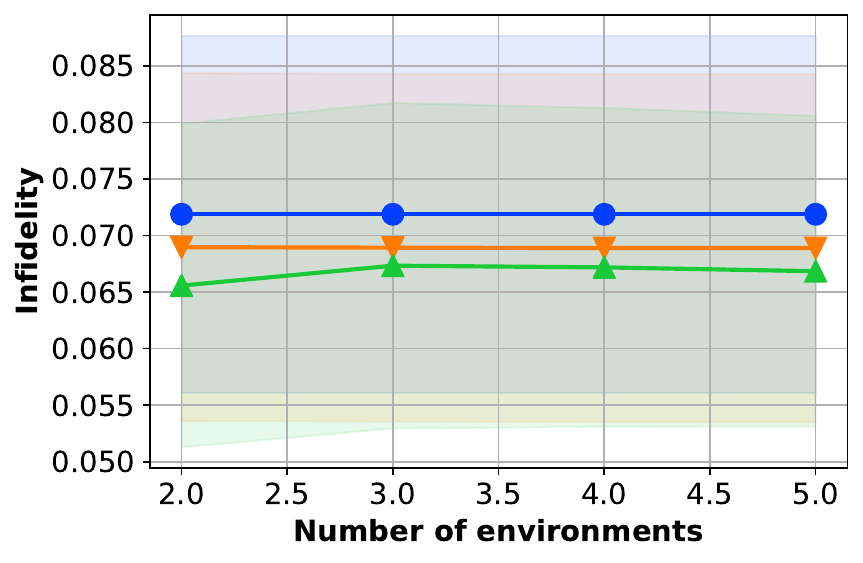}
        \caption{Rotten Tomatoes (random)}
        \label{fig:rotten_mnb_bp_num_envs_INFD}
    \end{subfigure}
    \begin{subfigure}[b]{0.24\textwidth}
    \vspace{0pt}
    \centering
    \captionsetup{justification=centering}
    \includegraphics[width=\textwidth]{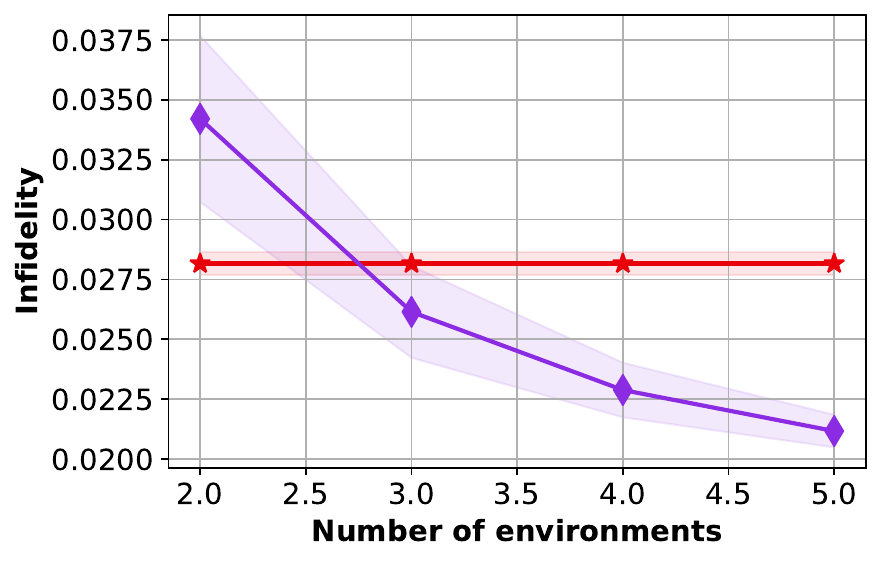}
    \caption{Rotten Tomatoes (realistic)}
    \label{fig:rotten_mnb_mp_num_envs_INFD}
    \end{subfigure}
    \begin{subfigure}[b]{0.24\textwidth}
    \vspace{0pt}
    \centering
    \captionsetup{justification=centering}
    \includegraphics[width=\textwidth]{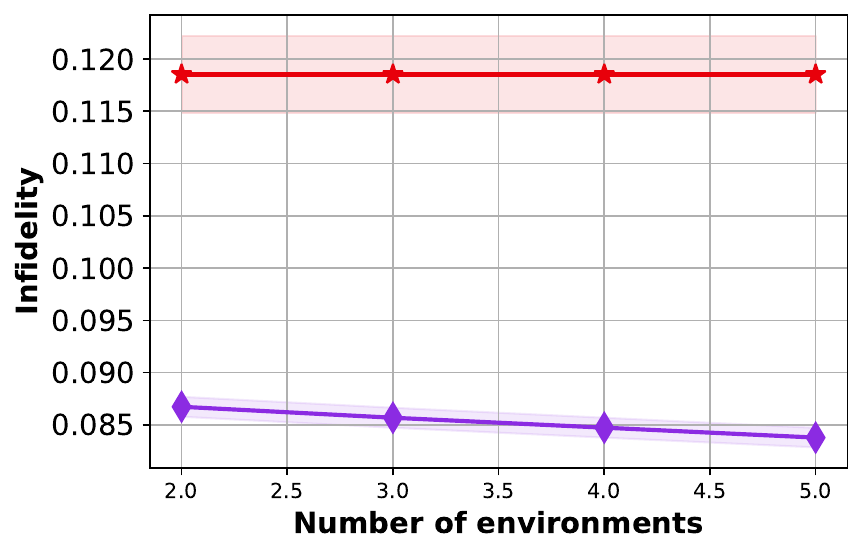}
    \caption{CIFAR10\\\phantom{test}}
    \label{fig:cifar_mp_num_envs_INFD}
    \end{subfigure}
\caption{Infidelity (INFD) vs. Number of environments.}
\label{fig:INFD_num_envs_all}
\end{figure}

\newpage


\begin{figure}[!htb]
\centering
    \begin{subfigure}[b]{0.24\textwidth}
        \vspace{0pt}
        \centering
        \captionsetup{justification=centering}
        \includegraphics[width=\textwidth]{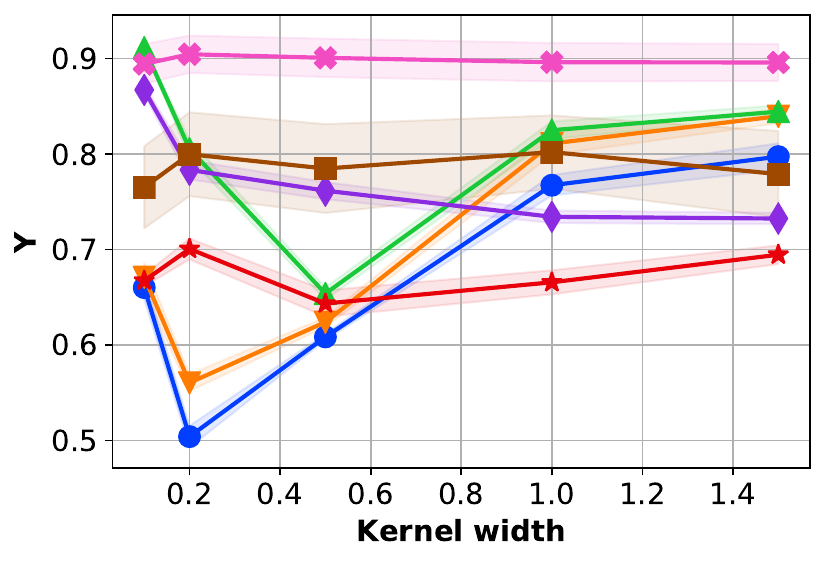}
        \caption{Iris}
        \label{fig:iris_rfc_kernel_width_Upsilon}
    \end{subfigure}
    \begin{subfigure}[b]{0.24\textwidth}
        \vspace{0pt}
        \centering
        \captionsetup{justification=centering}
        \includegraphics[width=\textwidth]{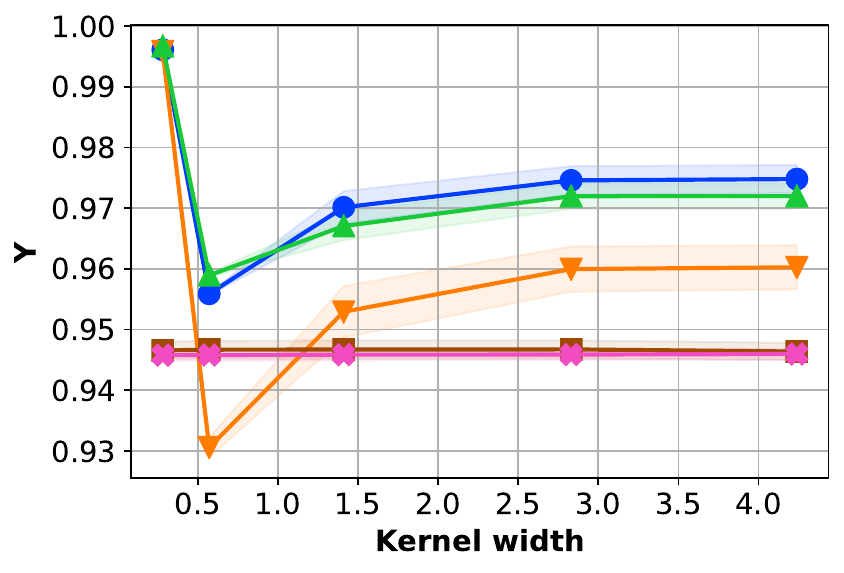}
        \caption{MEPS}
        \label{fig:MEPS_rfr_kernel_width_Upsilon}
    \end{subfigure}
    \begin{subfigure}[b]{0.24\textwidth}
        \vspace{0pt}
        \centering
        \captionsetup{justification=centering}
        \includegraphics[width=\textwidth]{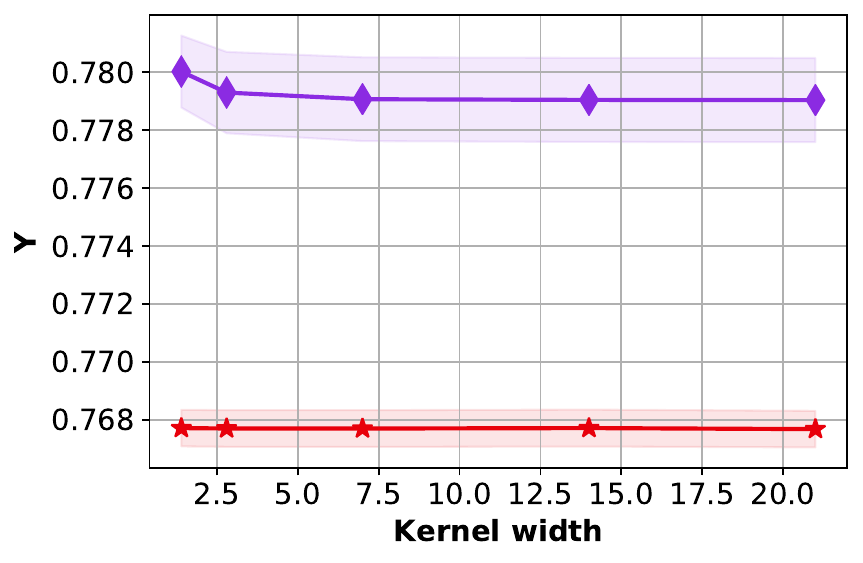}
        \caption{FMNIST}
        \label{fig:fmnist_nn_mp_kernel_width_Upsilon}
    \end{subfigure}
    \begin{subfigure}[b]{0.24\textwidth}
        \vspace{0pt}
        \centering
        \captionsetup{justification=centering}
        \includegraphics[width=\textwidth]{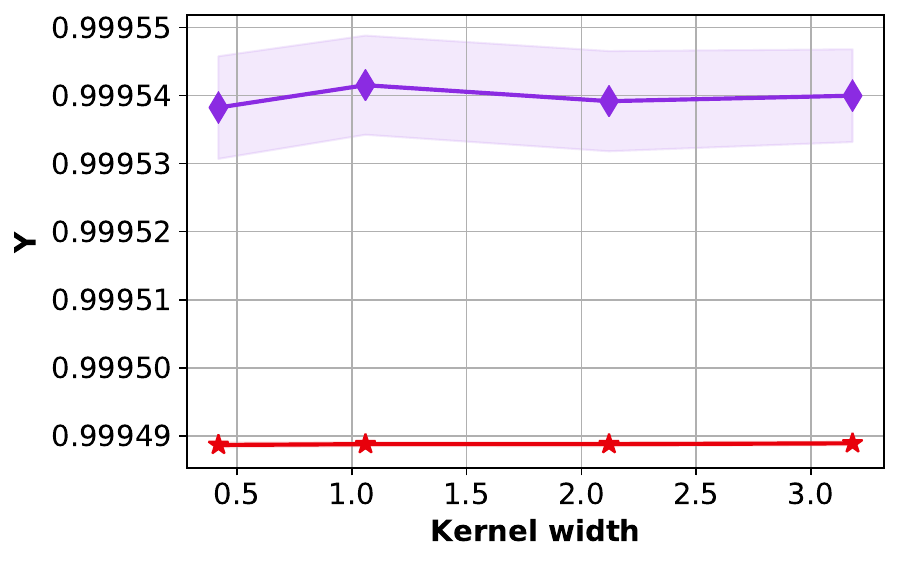}
        \caption{Rotten Tomatoes}
        \label{fig:rotten_mnb_mp_kernel_width_Upsilon}
    \end{subfigure}
    \begin{subfigure}[b]{0.24\textwidth}
        \vspace{0pt}
        \centering
        \captionsetup{justification=centering}
        \includegraphics[width=\textwidth]{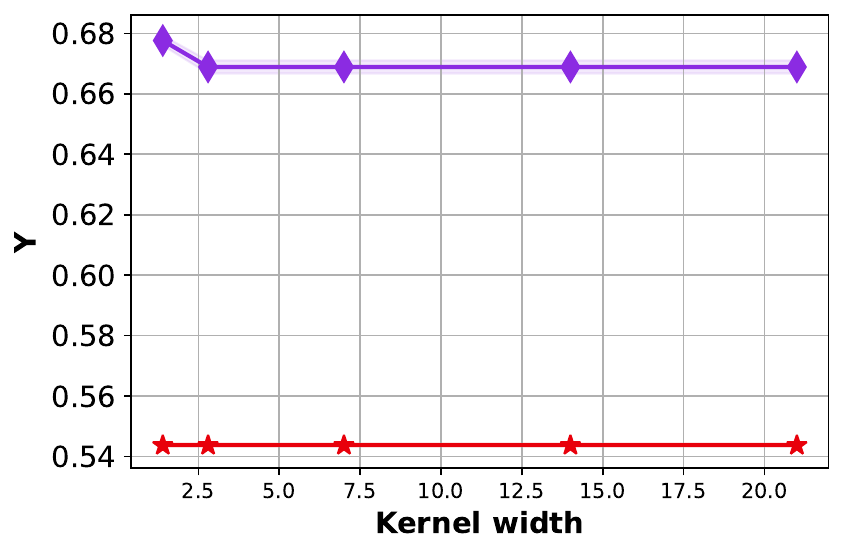}
        \caption{CIFAR10}
        \label{fig:cifar_mp_kernel_width_Upsilon}
    \end{subfigure}
\caption{Unidirectionality ($\Upsilon$) vs. Kernel width.}
\label{fig:Upsilon_kernel_width_all}
\end{figure}

\begin{figure}[!htb]
\centering
    \begin{subfigure}[b]{0.24\textwidth}
        \vspace{0pt}
        \centering
        \captionsetup{justification=centering}
        \includegraphics[width=\textwidth]{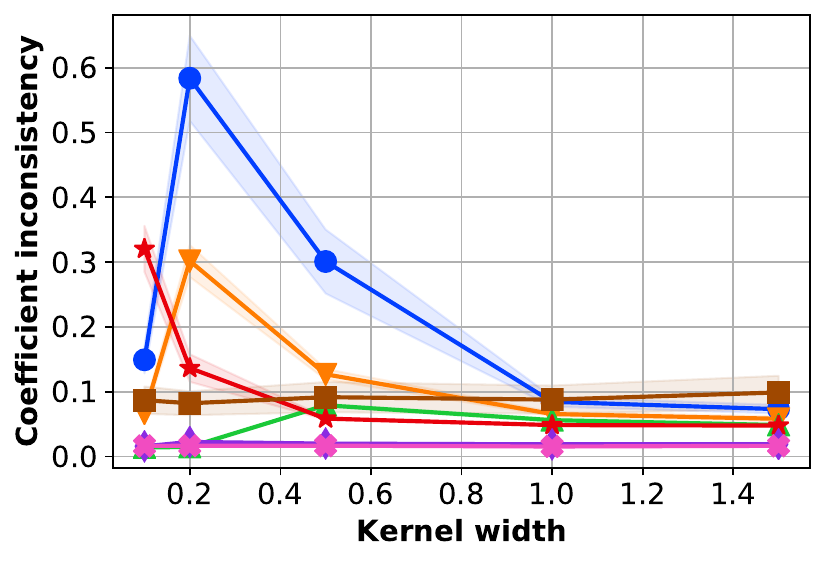}
        \caption{Iris}
        \label{fig:iris_rfc_kernel_width_CI}
    \end{subfigure}
    \begin{subfigure}[b]{0.24\textwidth}
        \vspace{0pt}
        \centering
        \captionsetup{justification=centering}
        \includegraphics[width=\textwidth]{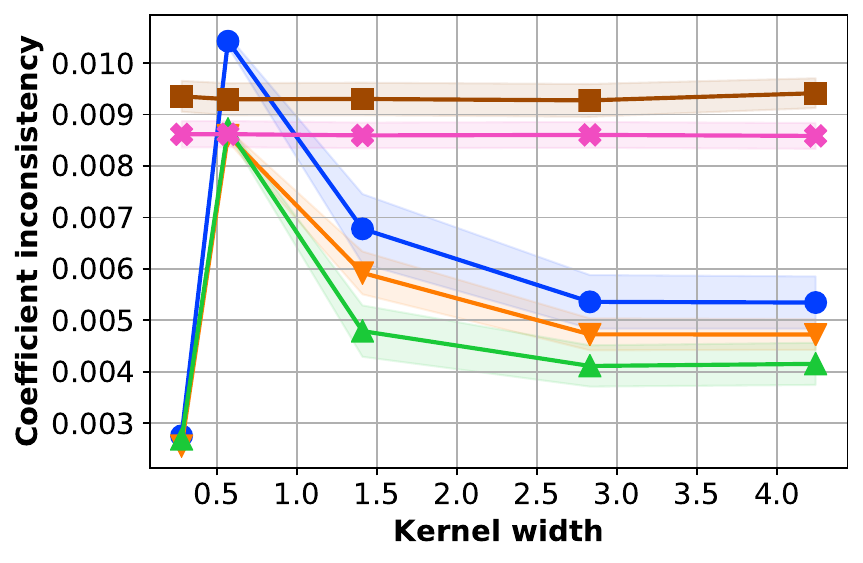}
        \caption{MEPS}
        \label{fig:MEPS_rfr_kernel_width_CI}
    \end{subfigure}
    \begin{subfigure}[b]{0.24\textwidth}
        \vspace{0pt}
        \centering
        \captionsetup{justification=centering}
        \includegraphics[width=\textwidth]{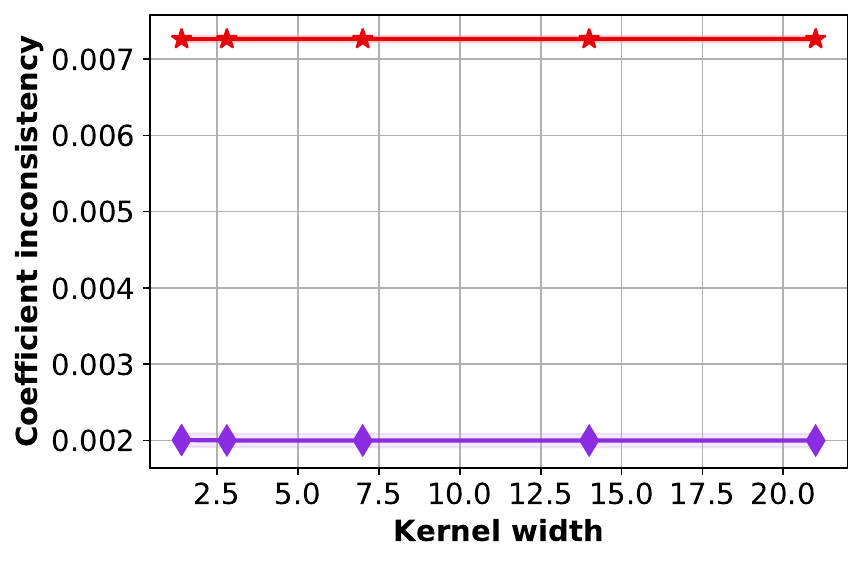}
        \caption{FMNIST}
        \label{fig:fmnist_nn_mp_kernel_width_CI}
    \end{subfigure}
    \begin{subfigure}[b]{0.24\textwidth}
        \vspace{0pt}
        \centering
        \captionsetup{justification=centering}
        \includegraphics[width=\textwidth]{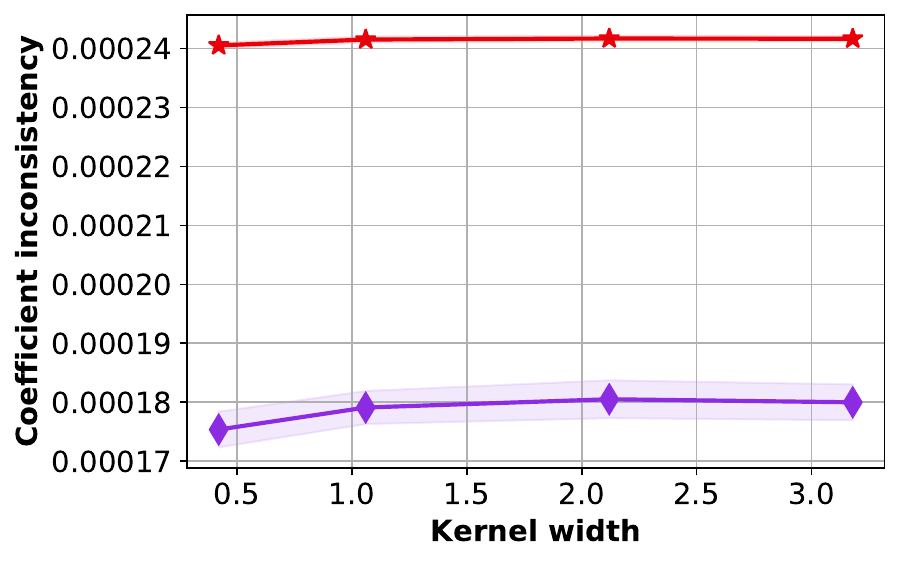}
        \caption{Rotten Tomatoes}
        \label{fig:rotten_mnb_mp_kernel_width_CI}
    \end{subfigure}
    \begin{subfigure}[b]{0.24\textwidth}
        \vspace{0pt}
        \centering
        \captionsetup{justification=centering}
        \includegraphics[width=\textwidth]{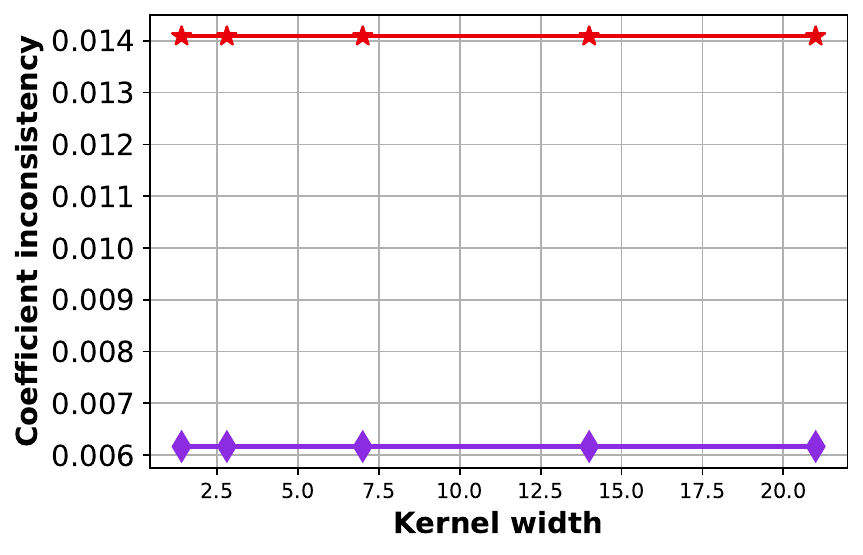}
        \caption{CIFAR10}
        \label{fig:cifar_mp_kernel_width_CI}
    \end{subfigure}
\caption{Coefficient inconsistency (CI) vs. Kernel width.}
\vspace{-0.5cm}
\label{fig:CI_kernel_width_all}
\end{figure}

\begin{figure}[!htb]
\centering
    \begin{subfigure}[b]{0.24\textwidth}
        \vspace{0pt}
        \centering
        \captionsetup{justification=centering}
        \includegraphics[width=\textwidth]{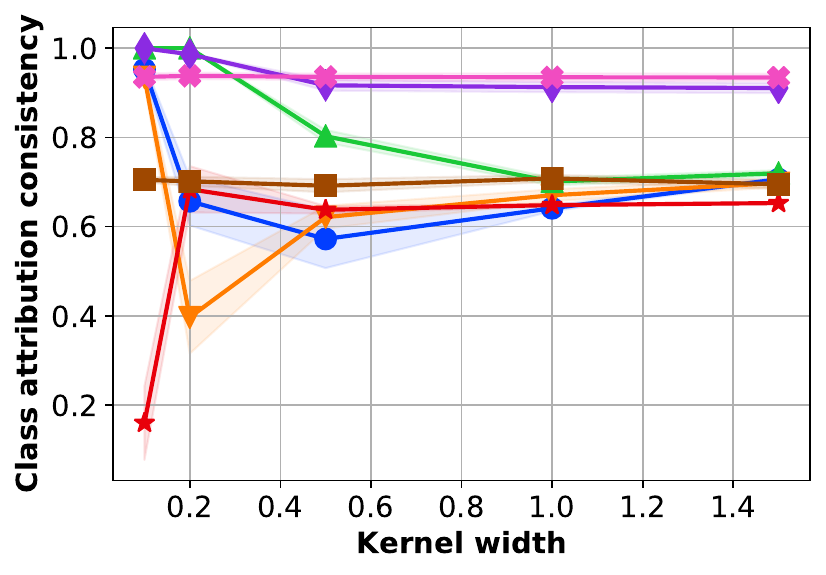}
        \caption{Iris}
        \label{fig:iris_rfc_kernel_width_CAC}
    \end{subfigure}
    \begin{subfigure}[b]{0.24\textwidth}
        \vspace{0pt}
        \centering
        \captionsetup{justification=centering}
        \includegraphics[width=\textwidth]{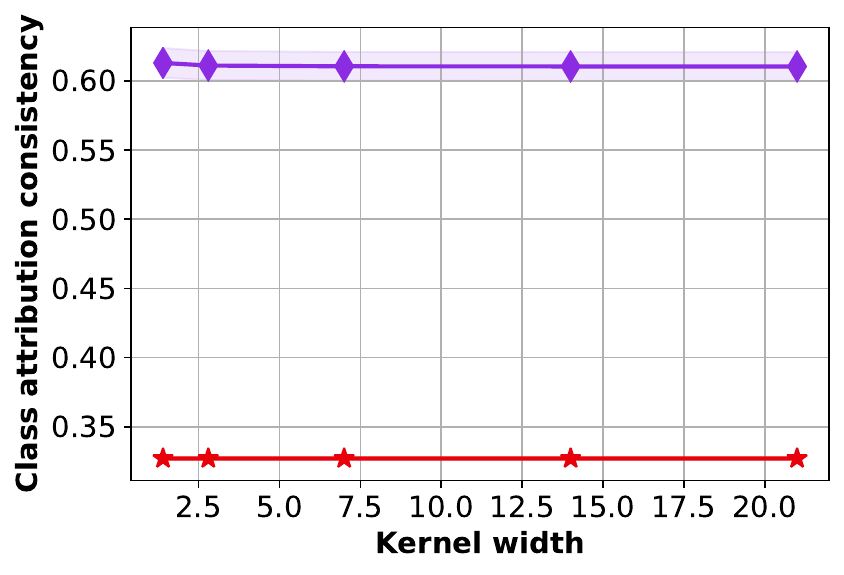}
        \caption{FMNIST}
        \label{fig:fmnist_nn_mp_kernel_width_CAC}
    \end{subfigure}
    \begin{subfigure}[b]{0.24\textwidth}
        \vspace{0pt}
        \centering
        \captionsetup{justification=centering}
        \includegraphics[width=\textwidth]{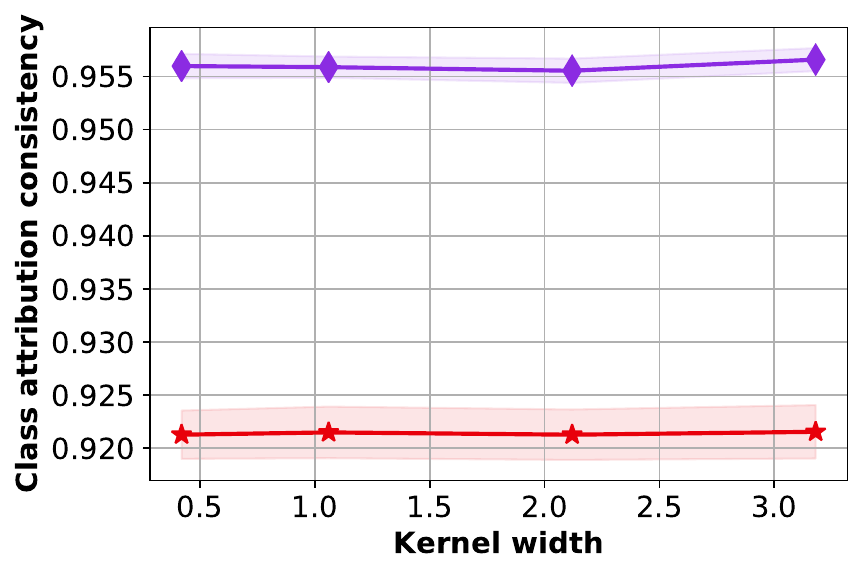}
        \caption{Rotten Tomatoes}
        \label{fig:rotten_mnb_mp_kernel_width_CAC}
    \end{subfigure}
\caption{Class attribution consistency (CAC) vs. Kernel width.}
\vspace{-0.5cm}
\label{fig:CAC_kernel_width_all}
\end{figure}

\begin{figure}[!htb]
\centering
    \begin{subfigure}[b]{0.24\textwidth}
        \vspace{0pt}
        \centering
        \captionsetup{justification=centering}
        \includegraphics[width=\textwidth]{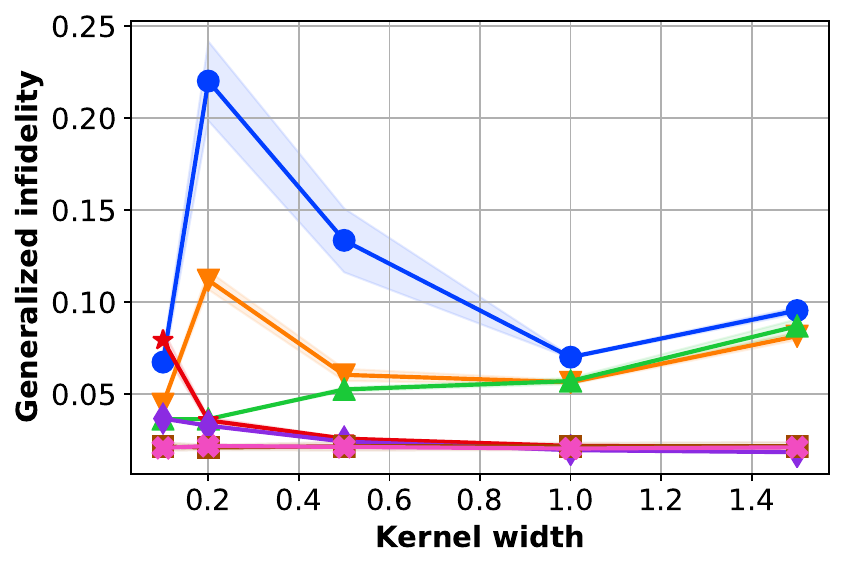}
        \caption{Iris}
        \label{fig:iris_rfc_kernel_width_GI}
    \end{subfigure}
    \begin{subfigure}[b]{0.24\textwidth}
        \vspace{0pt}
        \centering
        \captionsetup{justification=centering}
        \includegraphics[width=\textwidth]{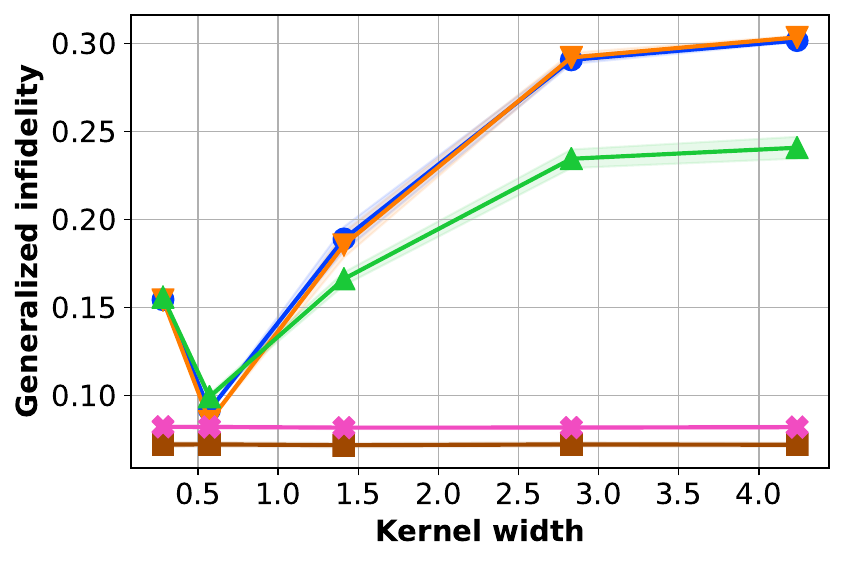}
        \caption{MEPS}
        \label{fig:MEPS_rfr_kernel_width_GI}
    \end{subfigure}
    \begin{subfigure}[b]{0.24\textwidth}
        \vspace{0pt}
        \centering
        \captionsetup{justification=centering}
        \includegraphics[width=\textwidth]{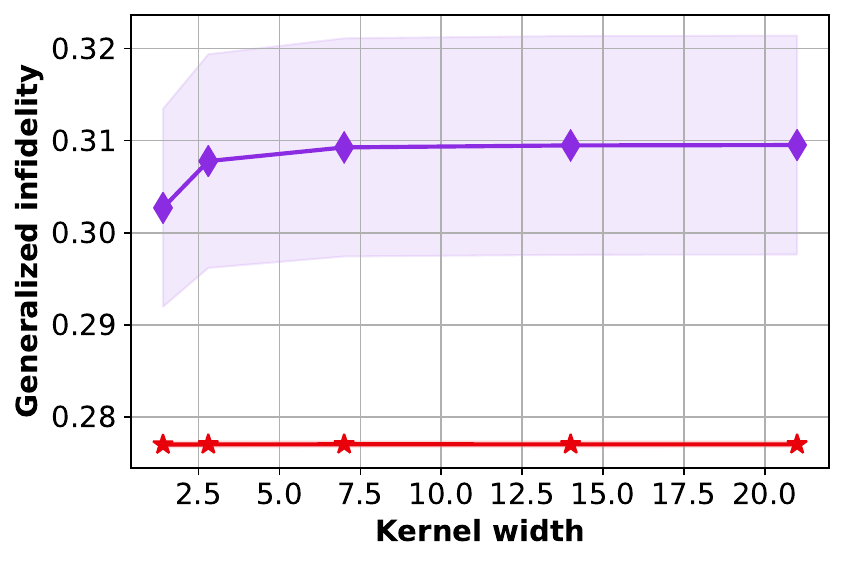}
        \caption{FMNIST}
        \label{fig:fmnist_nn_mp_kernel_width_GI}
    \end{subfigure}
    \begin{subfigure}[b]{0.24\textwidth}
        \vspace{0pt}
        \centering
        \captionsetup{justification=centering}
        \includegraphics[width=\textwidth]{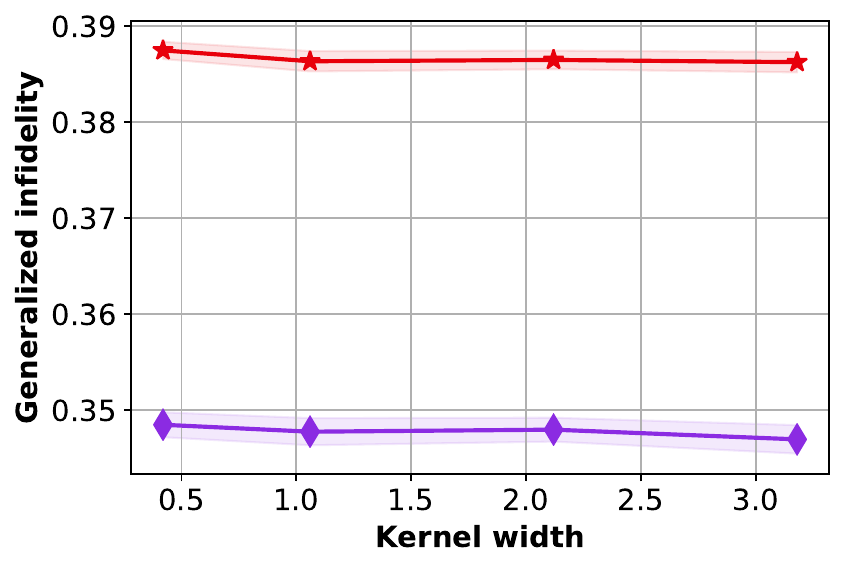}
        \caption{Rotten Tomatoes}
        \label{fig:rotten_mnb_mp_kernel_width_GI}
    \end{subfigure}
    \begin{subfigure}[b]{0.24\textwidth}
        \vspace{0pt}
        \centering
        \captionsetup{justification=centering}
        \includegraphics[width=\textwidth]{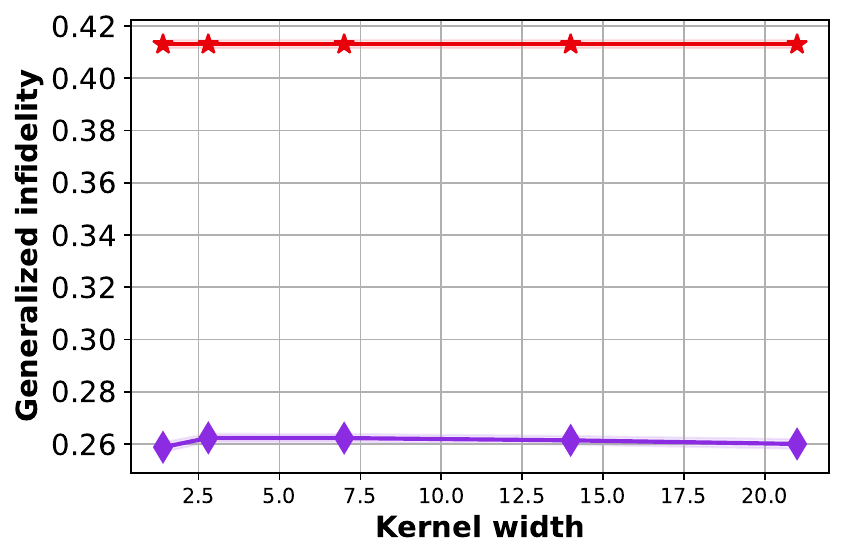}
        \caption{CIFAR10}
        \label{fig:cifar_mp_kernel_width_GI}
    \end{subfigure}
    
\caption{Generalized infidelity (GI) vs. Kernel width.}
\vspace{-0.5cm}
\label{fig:GI_kernel_width_all}
\end{figure}

\begin{figure}[!htb]
\centering
    \begin{subfigure}[b]{0.24\textwidth}
        \vspace{0pt}
        \centering
        \captionsetup{justification=centering}
        \includegraphics[width=\textwidth]{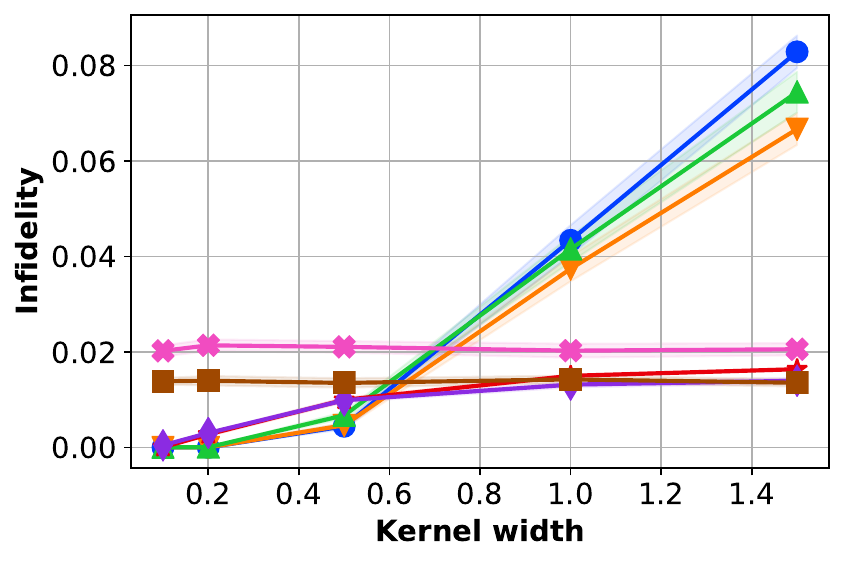}
        \caption{Iris}
        \label{fig:iris_rfc_kernel_width_INFD}
    \end{subfigure}
    \begin{subfigure}[b]{0.24\textwidth}
        \vspace{0pt}
        \centering
        \captionsetup{justification=centering}
        \includegraphics[width=\textwidth]{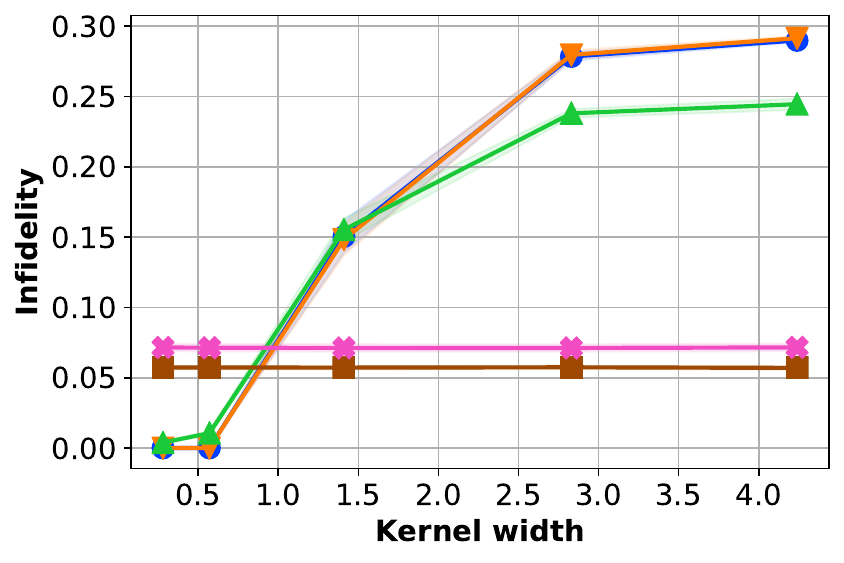}
        \caption{MEPS}
        \label{fig:MEPS_rfr_kernel_width_INFD}
    \end{subfigure}
    \begin{subfigure}[b]{0.24\textwidth}
        \vspace{0pt}
        \centering
        \captionsetup{justification=centering}
        \includegraphics[width=\textwidth]{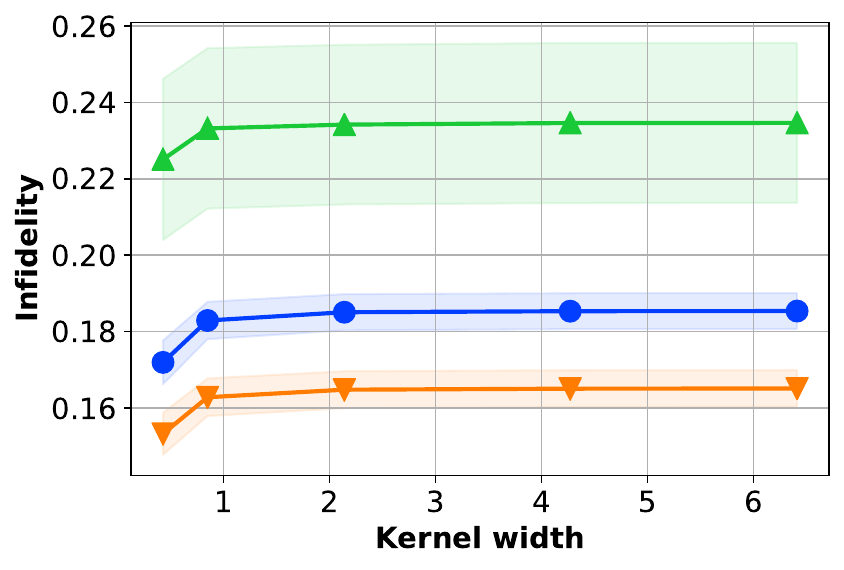}
        \caption{FMNIST (random)}
        \label{fig:fmnist_nn_bp_kernel_width_INFD}
    \end{subfigure}
    \begin{subfigure}[b]{0.24\textwidth}
        \vspace{0pt}
        \centering
        \captionsetup{justification=centering}
        \includegraphics[width=\textwidth]{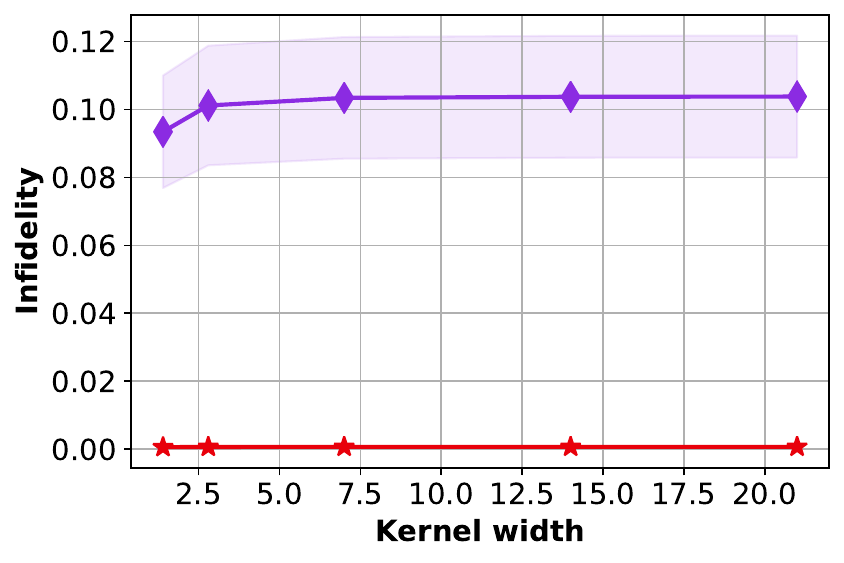}
        \caption{FMNIST (realistic)}
        \label{fig:fmnist_nn_mp_kernel_width_INFD}
    \end{subfigure}
    \begin{subfigure}[b]{0.24\textwidth}
        \vspace{0pt}
        \centering
        \captionsetup{justification=centering}
        \includegraphics[width=\textwidth]{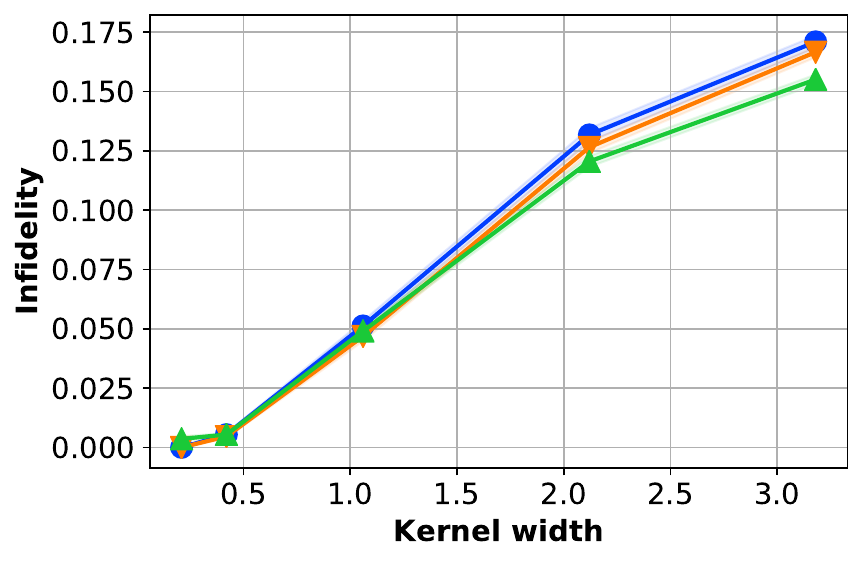}
        \caption{Rotten Tomatoes (random)}
        \label{fig:rotten_mnb_bp_kernel_width_INFD}
    \end{subfigure}
    \begin{subfigure}[b]{0.24\textwidth}
    \vspace{0pt}
    \centering
    \captionsetup{justification=centering}
    \includegraphics[width=\textwidth]{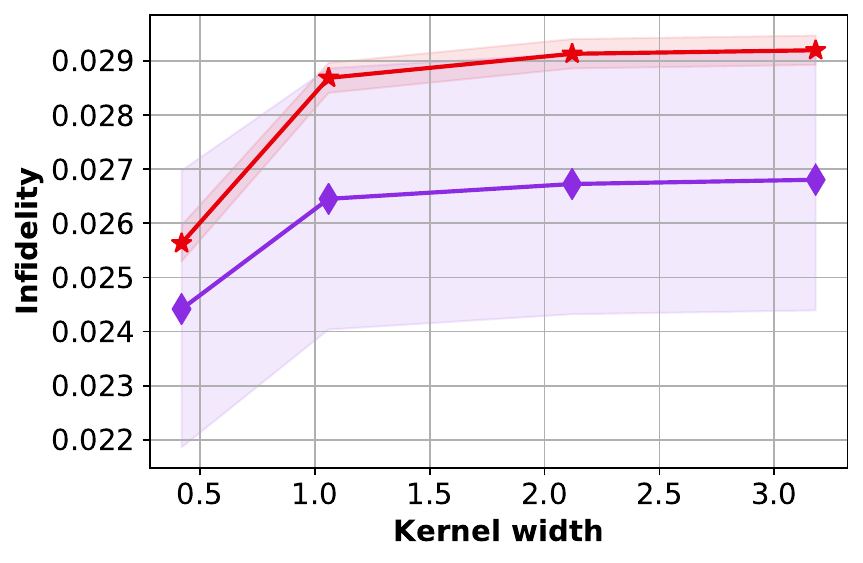}
    \caption{Rotten Tomatoes (realistic)}
    \label{fig:rotten_mnb_mp_kernel_width_INFD}
    \end{subfigure}
    \begin{subfigure}[b]{0.24\textwidth}
    \vspace{0pt}
    \centering
    \captionsetup{justification=centering}
    \includegraphics[width=\textwidth]{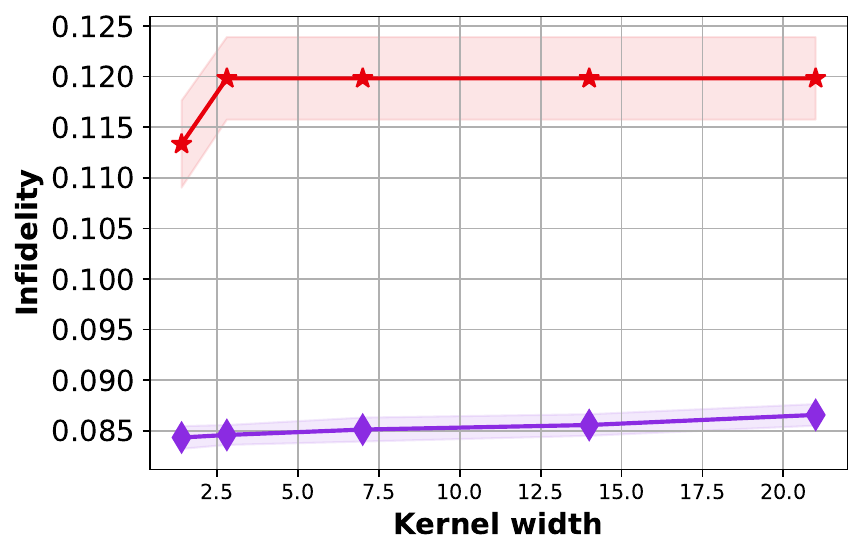}
    \caption{CIFAR10\\\phantom{test}}
    \label{fig:cifar_mp_kernel_width_INFD}
    \end{subfigure}
\caption{Infidelity (INFD) vs. Kernel width.}
\label{fig:INFD_kernel_width_all}
\end{figure}


\section{Example Feature Attributions in Text Data: MeLIME vs LINEX}
\label{app:text_examples}

Below we see sample attributions by the two methods along with the magnitude of the attributions. Attribution magnitudes are printed with a precision of $10^{-3}$ and shown along with the corresponding words in descending order. 

\subsection{Positive Sentiment}
\begin{verbatim}
enticing and often funny documentary .
MeLIME: documentary funny and enticing often
LINEX : documentary funny often enticing and
MeLIME: 0.517 0.446 0.333 0.317 0.311
LINEX : 0.416 0.377 0.342 0.331 0.330

one-of-a-kind near-masterpiece .
MeLIME: kind near masterpiece
LINEX : masterpiece kind one
MeLIME: 0.832 0.695 0.182 
LINEX : 0.712 0.384 0.381 

a fast , funny , highly enjoyable movie .
MeLIME: enjoyable highly funny fast movie
LINEX : enjoyable highly fast funny movie
MeLIME: 0.550 0.432 0.412 0.389 0.198
LINEX : 0.409 0.389 0.372 0.350 0.326

ferrara's strongest and most touching movie of recent years .
MeLIME: touching years most strongest and 
LINEX : touching most recent strongest and
MeLIME: 0.735 0.490 0.450 0.443 0.427 
LINEX : 0.490 0.488 0.450 0.444 0.407 

saved from being merely way-cool by a basic , credible compassion .
MeLIME: cool basic credible merely from
LINEX: cool credible merely compassion from
MeLIME: 1.514 0.050 0.040 0.029 0.026 
LINEX : 0.358 0.308 0.304 0.299 0.293 

really quite funny .
MeLIME: funny quite really
LINEX : funny quite really
MeLIME: 0.559 0.417 0.233
LINEX : 0.462 0.368 0.275

spare yet audacious . . .
MeLIME: spare yet audacious
LINEX : audacious spare yet
MeLIME: 0.626 0.447 0.395
LINEX : 0.501 0.431 0.422

an engrossing and infectiously enthusiastic documentary .
MeLIME: engrossing documentary and enthusiastic an
LINEX : engrossing documentary an enthusiastic and
MeLIME: 0.593 0.455 0.358 0.354 0.333
LINEX : 0.461 0.407 0.374 0.357 0.350

a wildly funny prison caper .
MeLIME: funny caper wildly prison
LINEX : funny caper prison wildly
MeLIME: 0.541 0.364 0.214 0.193
LINEX : 0.403 0.335 0.245 0.239

this charming but slight tale has warmth , wit 
and interesting characters compassionately portrayed .
MeLIME: charming compassionately and interesting portrayed 
LINEX : charming compassionately has tale portrayed 
MeLIME: 0.690 0.507 0.456 0.444 0.424
LINEX : 0.464 0.435 0.431 0.430 0.429

thoughtful , provocative and entertaining .
MeLIME: thoughtful entertaining and provocative
LINEX : thoughtful entertaining and provocative
MeLIME: 0.612 0.517 0.402 0.395
LINEX : 0.505 0.461 0.415 0.404

the film is quiet , threatening and unforgettable .
MeLIME: quiet unforgettable and film the
LINEX : unforgettable quiet film and is
MeLIME: 0.597 0.483 0.412 0.325 0.303 
LINEX : 0.421 0.416 0.388 0.378 0.338 

a moving tale of love and destruction in unexpected places , unexamined lives .
MeLIME: unexpected moving love tale lives
LINEX : moving unexpected places lives in
MeLIME: 0.692 0.662 0.577 0.538 0.499
LINEX : 0.538 0.530 0.521 0.513 0.501

though frodo's quest remains unfulfilled , a hardy group of 
determined new zealanders has proved its creative mettle .
MeLIME: creative group proved has new 
LINEX : creative quest its proved determined 
MeLIME: 0.602 0.441 0.424 0.402 0.393 
LINEX : 0.410 0.392 0.390 0.385 0.381 
\end{verbatim}

\subsection{Negative Sentiment}
\begin{verbatim}
originality is sorely lacking .
MeLIME: lacking sorely is originality
LINEX : lacking sorely originality is
MeLIME: 0.543 0.381 0.296 0.278
LINEX : 0.430 0.356 0.314 0.271

an ugly , pointless , stupid movie .
MeLIME: stupid pointless ugly movie an
LINEX : stupid pointless ugly movie an
MeLIME: 0.543 0.499 0.385 0.365 0.276
LINEX : 0.446 0.411 0.373 0.360 0.350

so devoid of pleasure or sensuality that it cannot even be dubbed hedonistic .
MeLIME: devoid even be dubbed of
LINEX : devoid so dubbed be cannot 
MeLIME: 0.666 0.416 0.413 0.372 0.344 
LINEX : 0.400 0.392 0.387 0.380 0.368 

neither revelatory nor truly edgy--merely crassly flamboyant 
and comedically labored .
MeLIME: edgy neither nor labored revelatory 
LINEX : edgy neither nor labored truly 
MeLIME: 1.256 0.338 0.277 0.204 0.021 
LINEX : 0.439 0.398 0.398 0.369 0.349 

occasionally funny , sometimes inspiring , often boring .
MeLIME: boring occasionally inspiring sometimes often
LINEX : boring occasionally sometimes often inspiring
MeLIME: 0.669 0.242 0.218 0.210 0.182 
LINEX : 0.377 0.266 0.266 0.250 0.236 

a cumbersome and cliche-ridden movie greased
with every emotional device known to man .
MeLIME: cliche every device movie with
LINEX : cliche every man cumbersome emotional
MeLIME: 0.695 0.449 0.327 0.280 0.268 
LINEX : 0.385 0.361 0.354 0.349 0.309 

ponderous , plodding soap opera disguised as a feature film .
MeLIME: plodding soap ponderous opera disguised 
LINEX : plodding soap film ponderous feature 
MeLIME: 0.579 0.522 0.421 0.408 0.382 
LINEX : 0.442 0.440 0.418 0.406 0.377 

kitschy , flashy , overlong soap opera .
MeLIME: soap flashy opera overlong kitschy
LINEX : soap flashy opera overlong kitschy
MeLIME: 0.499 0.397 0.391 0.358 0.230
LINEX : 0.389 0.362 0.360 0.346 0.300

[a] poorly executed comedy .
MeLIME: poorly comedy executed
LINEX : poorly comedy executed
MeLIME: 0.653 0.348 0.257
LINEX : 0.502 0.335 0.309

a bad movie that happened to good actors .
MeLIME: bad happened movie to that 
LINEX : bad happened to movie actors 
MeLIME: 0.692 0.396 0.371 0.367 0.242 
LINEX : 0.442 0.384 0.367 0.361 0.344 

a complete waste of time .
MeLIME: waste complete time of
LINEX : waste complete time of
MeLIME: 0.614 0.425 0.313 0.247
LINEX : 0.480 0.381 0.348 0.278

don't waste your money .
MeLIME: waste money don your
LINEX : waste money don your
MeLIME: 0.592 0.497 0.408 0.309
LINEX : 0.483 0.450 0.411 0.337

witless and utterly pointless .
MeLIME: pointless witless and utterly
LINEX : pointless witless utterly and
MeLIME: 0.652 0.491 0.263 0.245
LINEX : 0.506 0.444 0.311 0.269
\end{verbatim}

\section{Example Feature Attributions in Image Data: MeLIME vs LINEX}
\label{app:image_examples}
We show feature attributions for individual example images with MeLIME and LINEX with MeLIME perturbations in Figure \ref{fig:fmnist_ind_examples2}. In Figure \ref{fig:fmnist_ind_examples} we show class-wise mean feature attributions along with mean images. In Figure \ref{fig:cifar_ind_examples}, we see examples from CIFAR10. LINEX explanations seem to provide more meaningful feature attributions.

\begin{figure}[t]
\includegraphics[width=0.5\textwidth]{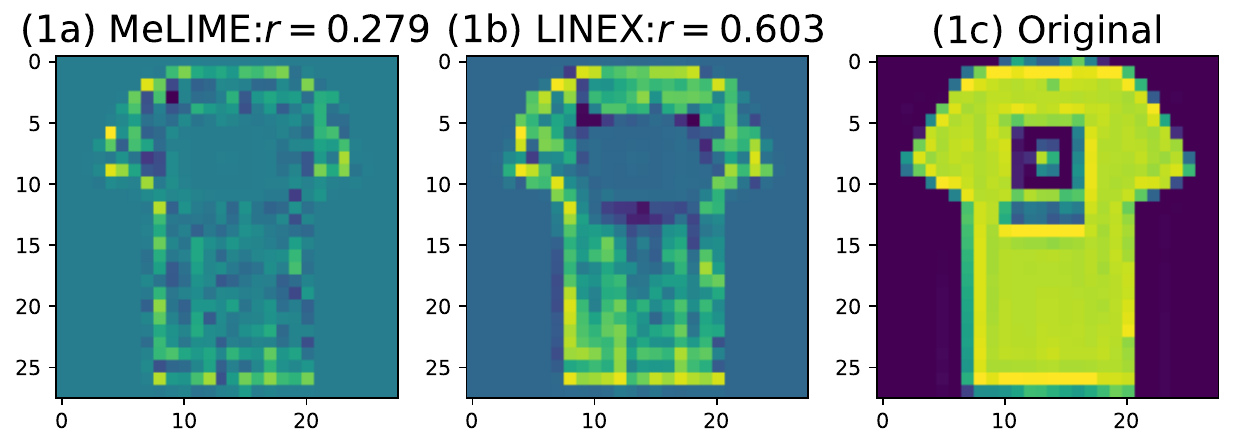}
\includegraphics[width=0.5\textwidth]{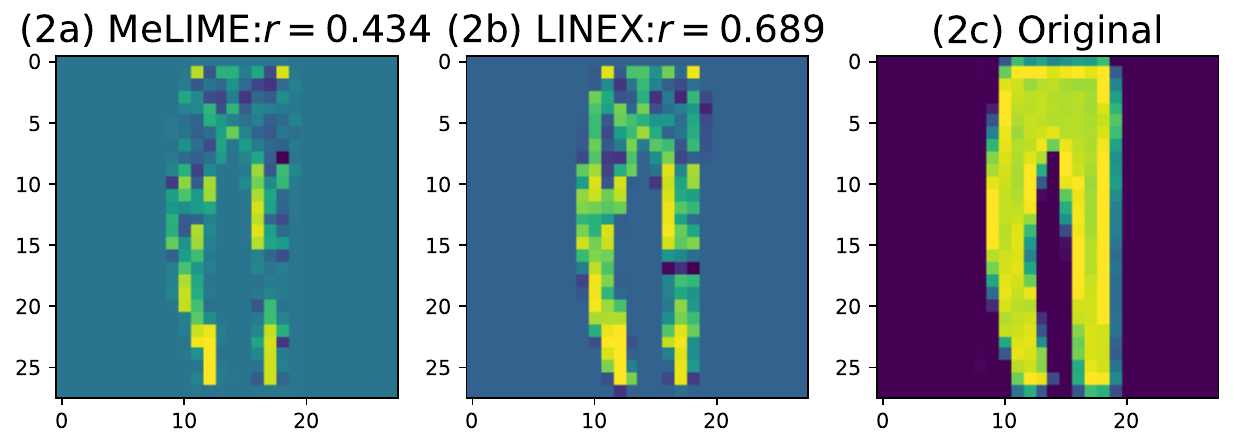}
\includegraphics[width=0.5\textwidth]{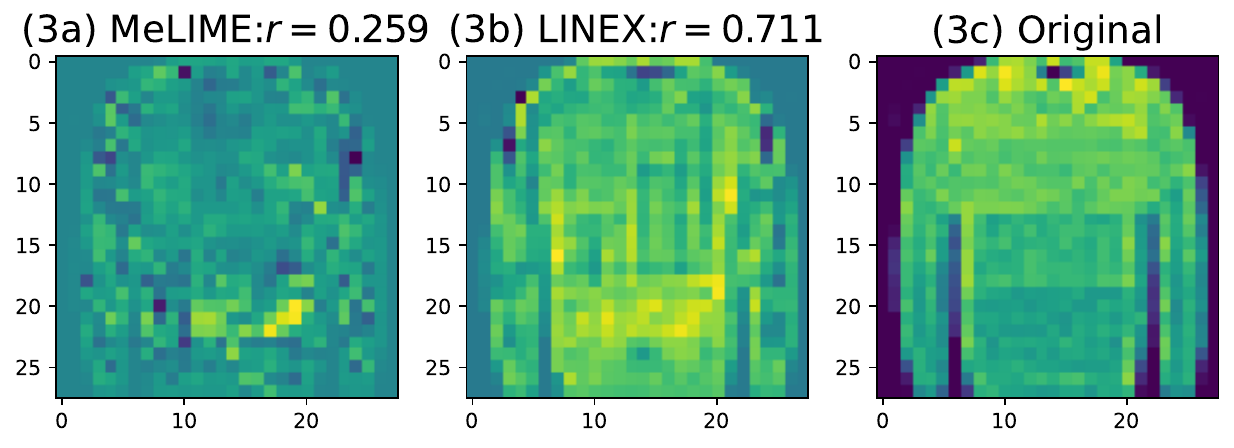}
\includegraphics[width=0.5\textwidth]{figures/fmnist-class-3.pdf}
\includegraphics[width=0.5\textwidth]{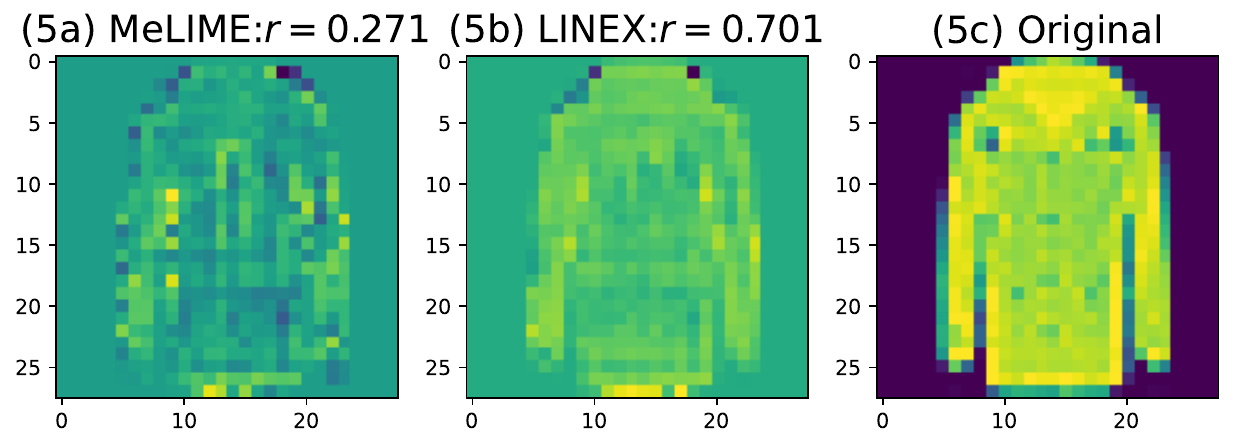}
\includegraphics[width=0.5\textwidth]{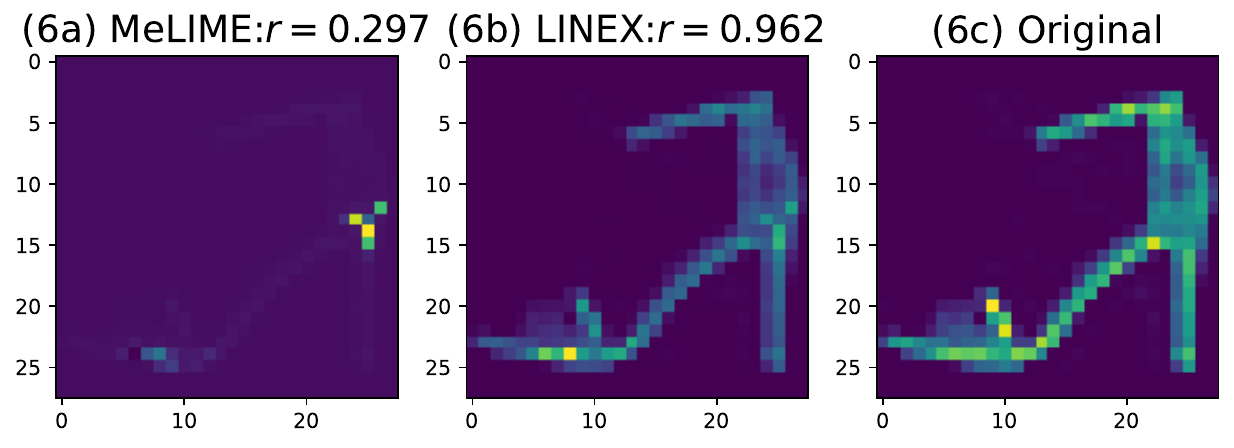}
\includegraphics[width=0.5\textwidth]{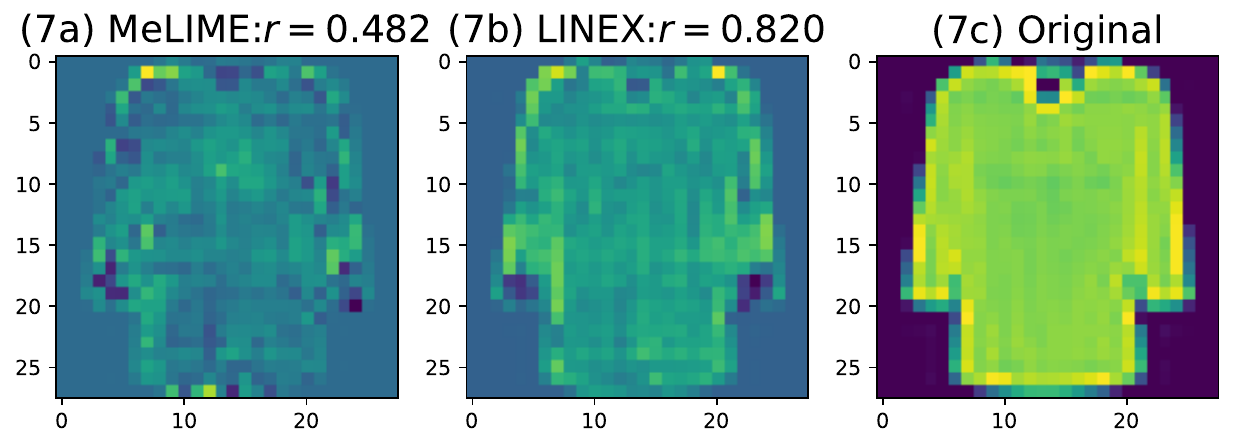}
\includegraphics[width=0.5\textwidth]{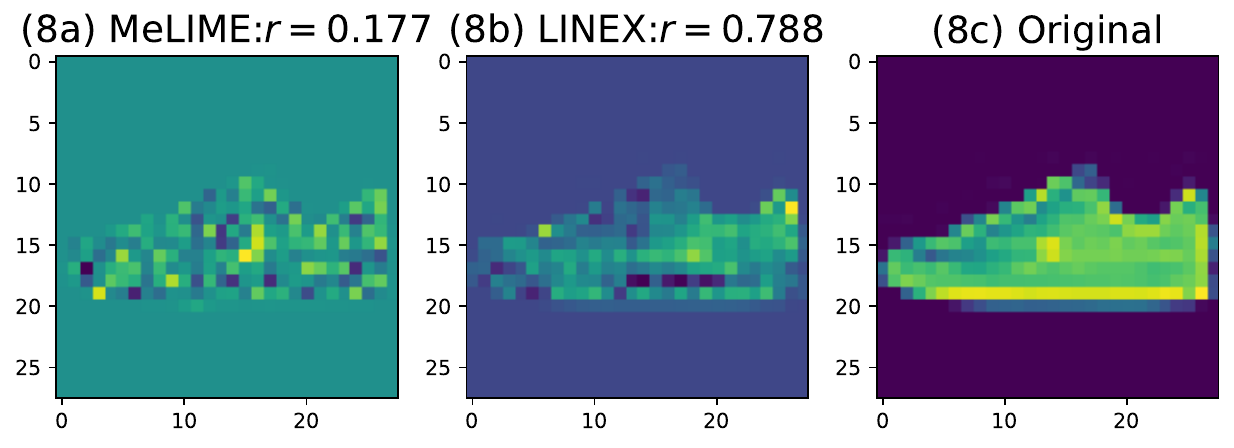}
\includegraphics[width=0.5\textwidth]{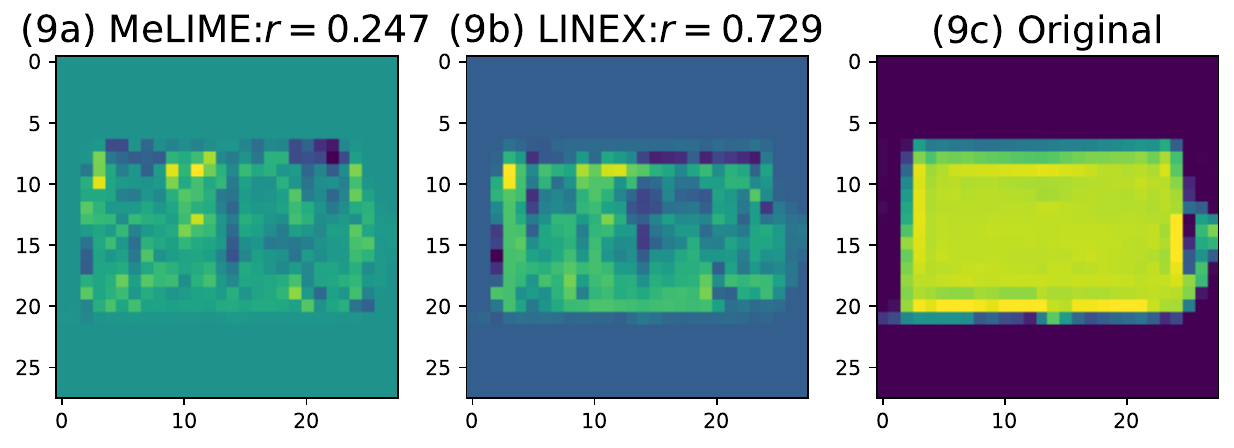}
\includegraphics[width=0.5\textwidth]{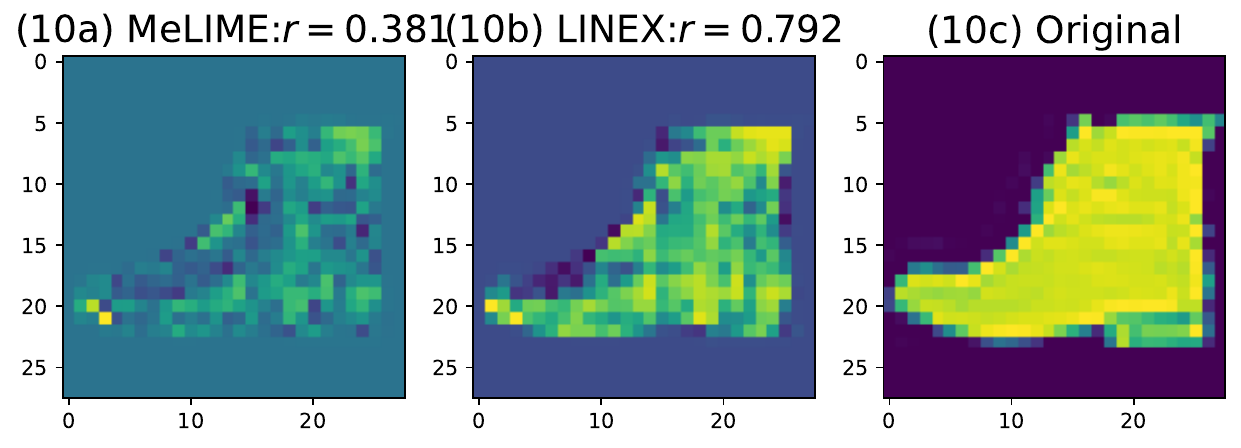}
\caption{Results using individual samples for realistic perturbations for FMNIST dataset for all classes:$1$-$10$ (\emph{T-shirt/top, Trouser, Pullover, Dress, Coat, 
Sandal, Shirt, Sneaker, Bag} and \emph{Ankle boot}). (a) MeLIME feature attributions for an image. (b) LINEX feature attributions for an image. (c) Original image in the class. The $r$ values show Pearson's correlation between feature attributions and the original image from the respective class. We observe that LINEX attributions/explanations exhibit significantly higher correlation with the original image belonging to a particular class (i.e. high CAC).}
\label{fig:fmnist_ind_examples2}
\end{figure}

\begin{figure}[t]
\includegraphics[width=0.5\textwidth]{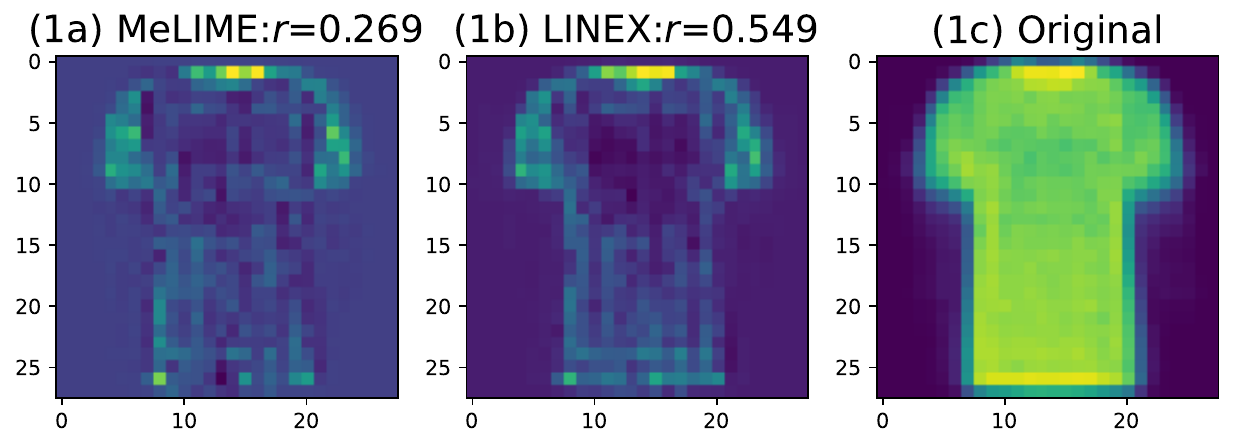}
\includegraphics[width=0.5\textwidth]{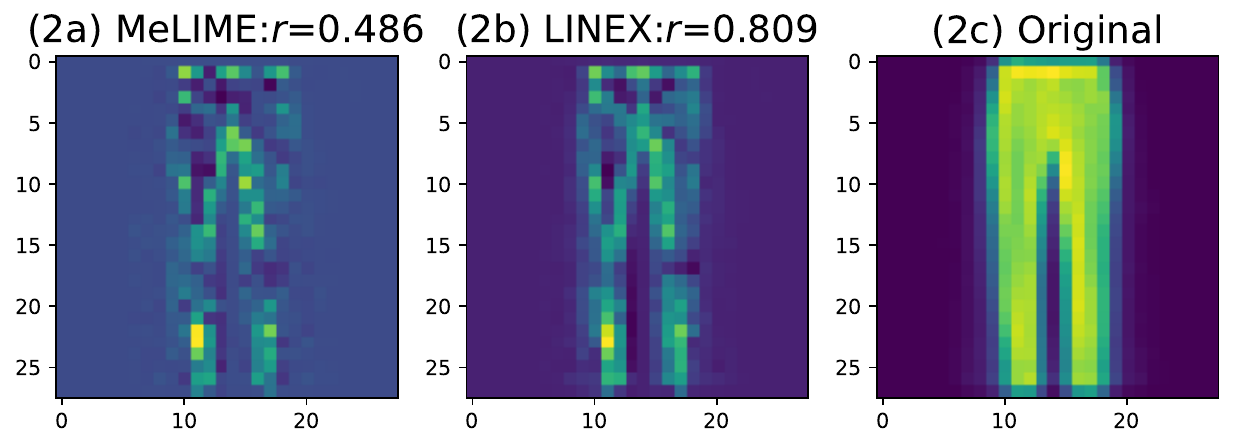}
\includegraphics[width=0.5\textwidth]{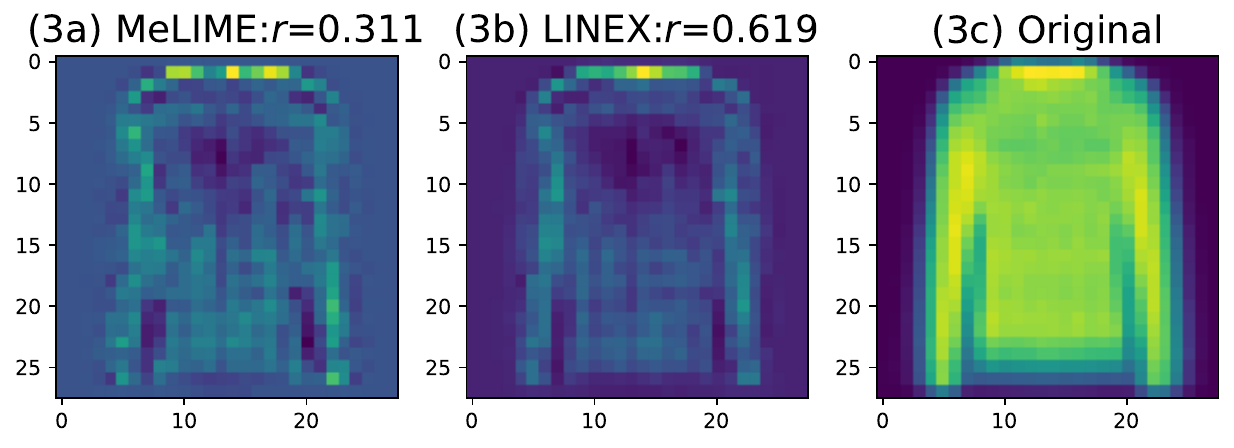}
\includegraphics[width=0.5\textwidth]{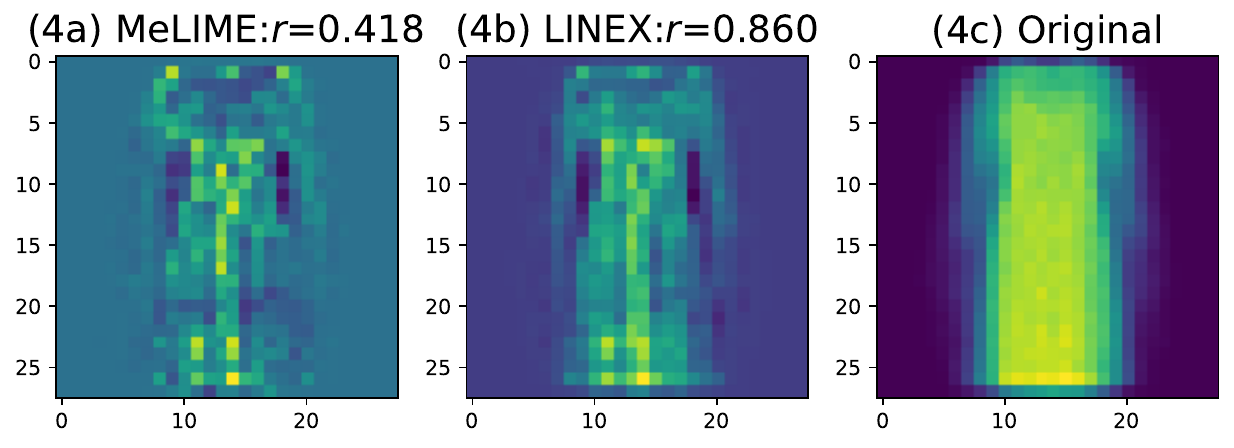}
\includegraphics[width=0.5\textwidth]{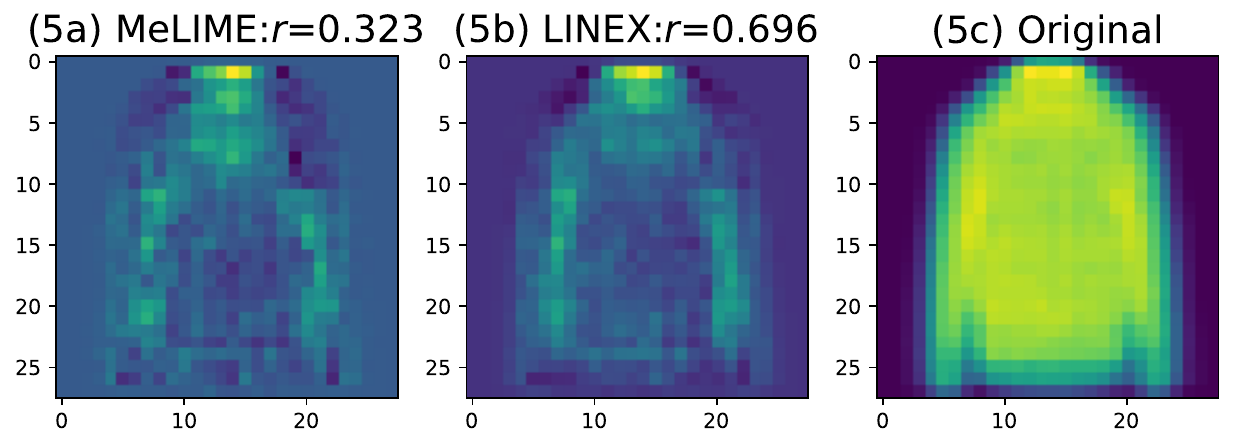}
\includegraphics[width=0.5\textwidth]{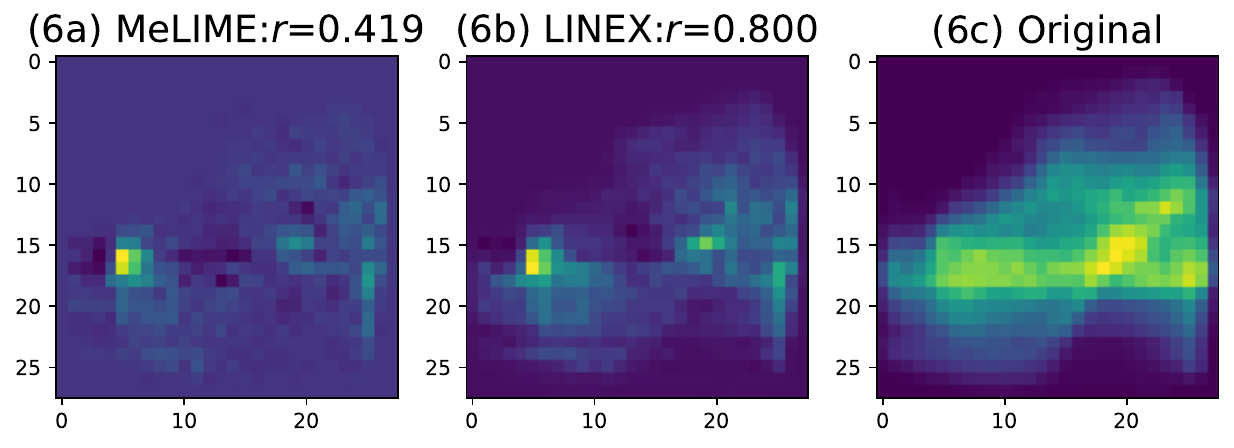}
\includegraphics[width=0.5\textwidth]{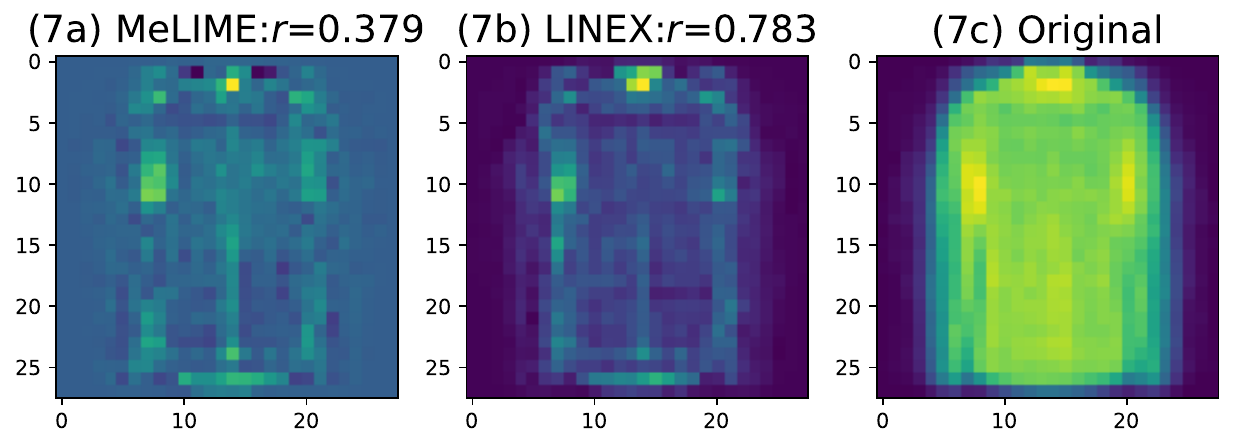}
\includegraphics[width=0.5\textwidth]{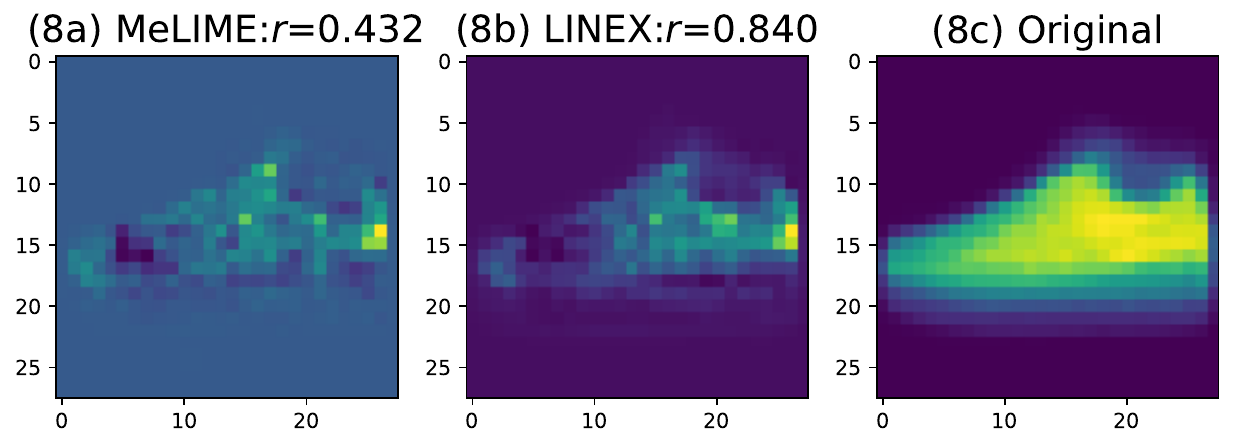}
\includegraphics[width=0.5\textwidth]{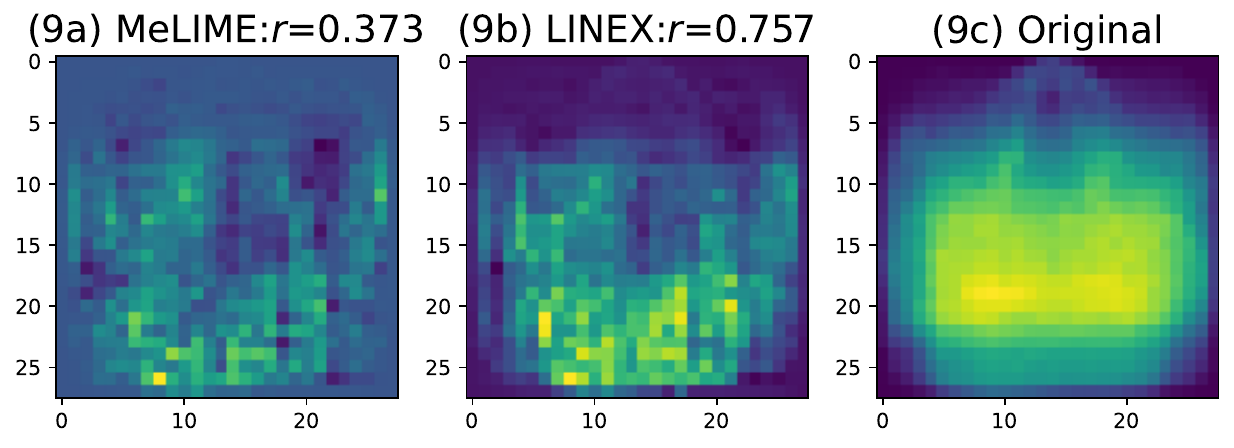}
\includegraphics[width=0.5\textwidth]{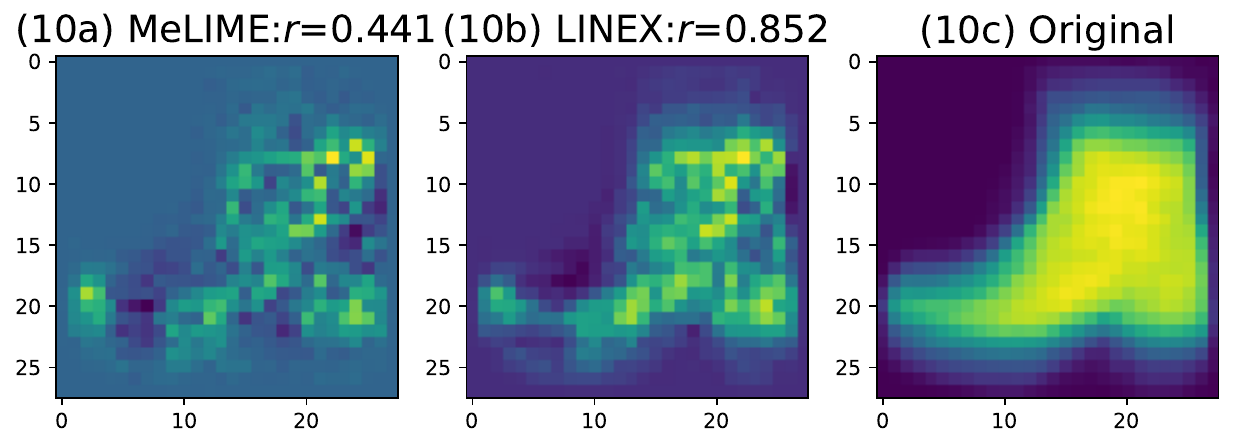}
\caption{Results using realistic perturbations for FMNIST dataset with mean feature importances for all classes:$1$-$10$ (\emph{T-shirt/top, Trouser, Pullover, Dress, Coat, 
Sandal, Shirt, Sneaker, Bag} and \emph{Ankle boot}). (a) Mean feature attributions of all images in the class using MeLIME. (b) Mean feature attributions of all images in the class using LINEX.  (c) Mean of all images in the class. The $r$ values show Pearson's correlation between average feature attributions and mean of the original images from the respective classes. We observe that LINEX explanations/attributions exhibit significantly higher correlation with the original images belonging to a particular class (i.e. high CAC).}
\label{fig:fmnist_ind_examples}
\end{figure}

\begin{figure}[t]
\includegraphics[width=\textwidth]{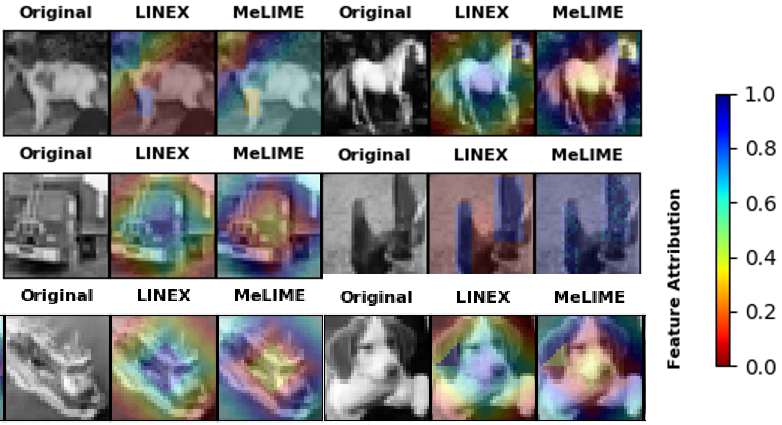}
\caption{Results using realistic perturbations for CIFAR10 dataset. We see above images of a dog, a horse, a truck, a bird, a boat and a dog again randomly selected from CIFAR10. The original images are greyed out here so that the (normalized) attributions are clearly visible. As can be seen LINEX attributions seem to consistently focus on salient features as compared to MeLIME. For example for the first dog image we highlight the head, ears and leg, while MeLIME focuses more on the neck and some of the background. For horse too LINEX focuses on head and body, while MeLIME focuses on the legs and neck. For truck both seem to focus on important features. For bird LINEX hones in on the wings, while MeLIME although giving importance to wings also attributes some of the background. The boat image LINEX focuses on the center of the boat, while Melime on the edges and some of the water around the boat. For the dog face image LINEX focuses on the nose, eyes and ears, while Melime focuses on the ears and neck.}
\label{fig:cifar_ind_examples}
\end{figure}



\section{Results for All Methods Including SHAP}
\label{app:res_shap}
In Table \ref{tab:results2}, we provide the results for SHAP along with all methods for easy comparison. Note that SHAP does not have standard errors since it is computed only once per test point. The INFD values for SHAP are miniscule since SHAP values add up to the predictions by definition. In order to compute GI, CI, $\Upsilon$, CAC, we convert the SHAP values to SHAP attributions \cite{amparore2021trust} first and follow the same approach used by other explanation methods.

\begin{table}[t]
\small
\caption{Comparing the different methods (including SHAP) using metrics infidelity (INFD), generalized infidelity (GI), coefficient inconsistency (CI), class attribution consistency (CAC) and unidirectionality ($\Upsilon$).}
\addtolength{\tabcolsep}{-4pt}
\begin{tabular}{|c | c | c | c | c | c | c |}
\hline
\textit{Dataset} & \textit{Method} & INFD $\downarrow$ & GI $\downarrow$ & CI $\downarrow$ & $\Upsilon$ $\uparrow$ & CAC $\uparrow$ \\
\hline
\hline
\multirow{8}{*}{\textit{IRIS}} 

& {LIME}  &  $0.015 \pm 0.011$ & $0.132 \pm 0.042$ & $0.319 \pm 0.132$ & $0.646 \pm 0.040$ & $0.667 \pm 0.167$ \\
& {S-LIME}  &  $0.015 \pm 0.010$ & $0.077 \pm 0.011$ & $0.143 \pm 0.045$ & $0.704 \pm 0.037$ & $0.878 \pm 0.034$ \\
& {LINEX/rand} &  $0.013 \pm 0.009$ & $\mathbf{0.052 \pm 0.008}$ & $\mathbf{0.044 \pm 0.013}$ & $\mathbf{0.802 \pm 0.043}$ & $\mathbf{0.921 \pm 0.042}$ \\


& {{NB/rand}}  &  {$0.040 \pm 0.010$} & {$0.067 \pm 0.003$} & {$0.319 \pm 0.132$} & {$0.646 \pm 0.040$} & {$0.667 \pm 0.167$} \\

\cline{2-7}

& {MeLIME}  & $0.008 \pm 0.003$ & $0.049 \pm 0.018$ & $0.219 \pm 0.108$ & $0.629 \pm 0.013$ & $0.464 \pm 0.100$\\
& {LINEX/real}  & $0.009 \pm 0.003$ & $\mathbf{0.029 \pm 0.003}$ & $\mathbf{0.024 \pm 0.002}$ & $\mathbf{0.744 \pm 0.044}$ & $\mathbf{0.942 \pm 0.023}$\\

& {{NB/real}}  & {$0.058 \pm 0.022$} & {$0.034 \pm 0.000$} & {$0.219 \pm 0.108$} & {$0.629 \pm 0.013$} & {$0.464 \pm 0.100$}\\

\cline{2-7}
& {MAPLE}  & $0.009 \pm 0.001$ & $0.038 \pm 0.004$ & $0.261 \pm 0.033$ & $0.458 \pm 0.032$ & $0.586 \pm 0.035$\\
& {LINEX/mpl}  & $0.013 \pm 0.000$ & $\mathbf{0.020 \pm 0.000}$ & $\mathbf{0.026 \pm 0.002}$ & $\mathbf{0.694 \pm 0.008}$ & $\mathbf{0.929 \pm 0.004}$\\

\cline{2-7}
& {SHAP}  & $0.007$ & $0.197$ & $0.248$ & $0.664$ & $0.524$\\

\hline
\hline
\multirow{6}{*}{\textit{MEPS}} 
& {LIME} & $0.158 \pm 0.066$ & $0.214 \pm 0.041$ & $0.005 \pm 0.001$ & $0.981 \pm 0.006$ & \multirow{4}{*}{NA}\\
& {S-LIME} & $0.158 \pm 0.066$ & $0.214 \pm 0.042$ & $0.005 \pm 0.001$ & $0.974 \pm 0.008$ & \\
& {LINEX/rand} & $\mathbf{0.130 \pm 0.052}$ & $\mathbf{0.164 \pm 0.021}$ & $0.003 \pm 0.001$ & $0.979 \pm 0.006$ & \\
& {NB/rand} & {$0.275 \pm 0.062$} & {$0.311 \pm 0.079$} & {$0.005 \pm 0.001$} & {$0.981 \pm 0.006$} & \\

\cline{2-7}
& {MAPLE} & $\mathbf{0.063 \pm 0.000}$ & $\mathbf{0.067 \pm 0.000}$ & $0.007 \pm 0.000$ & $0.957 \pm 0.000$ & \multirow{2}{*}{NA}\\
& {LINEX/mpl}  & $0.098 \pm 0.001$ & $0.094 \pm 0.001$ & $0.007 \pm 0.000$ & $0.950 \pm 0.000$	& \\
                            
\cline{2-7}
& {SHAP}  & $0.000$ & $0.091$ & $0.009$ & $0.940$ & NA \\

\hline
\hline
\multirow{4}{*}{\textit{FMNIST}} 
& {LIME}  &  $0.162 \pm 0.003$ & \multirow{4}{*}{NA} & \multirow{4}{*}{NA} &	\multirow{4}{*}{NA} & \multirow{4}{*}{NA} \\
& {S-LIME}  & $0.142 \pm 0.003$ &  &  &	&  \\
& {LINEX/rand} & $0.149 \pm 0.002$ &  &	 &	& \\
& {NB/rand} & {$0.207 \pm 0.000$} &  &	 &	& \\

\cline{2-7}
& {MeLIME}  & $\mathbf{0.001 \pm 0.000}$ & $\mathbf{0.277 \pm 0.000}$ & $0.007 \pm 0.000$ & $0.769 \pm 0.000$ & $0.327 \pm 0.000$ \\
& {LINEX/real}  & $0.100 \pm 0.002$ & $0.304 \pm 0.001$ & $0.002 \pm 0.000$ & $\mathbf{0.780 \pm 0.000}$ & $\mathbf{0.649 \pm 0.001}$ \\
& {NB/real}  & {$0.017 \pm 0.000$} & {$0.446 \pm 0.000$} & {$0.007 \pm 0.000$} & {$0.769 \pm 0.000$} & {$0.327 \pm 0.000$} \\

\cline{2-7}
& {SHAP} & $0.000$ & $1.962$ & $0.589$ & $0.551$ & $0.038$ \\

\hline
\hline

\multirow{5}{*}{\textit{CIFAR10}} 
& {LIME}  &  $0.191 \pm 0.005$ & \multirow{4}{*}{NA} & \multirow{4}{*}{NA} &	\multirow{4}{*}{NA} & \multirow{4}{*}{NA} \\
& {S-LIME}  & $0.185 \pm 0.002$ &  &  &	&  \\
& {LINEX/rand} & $0.186 \pm 0.002$ &  &	 &	& \\
& {NB/rand}  & {$0.208 \pm 0.001$} &   &   &	 &  \\                              
\cline{2-7}
& {MeLIME}  & $0.100 \pm 0.003$ & $0.412 \pm 0.007$ & $0.014 \pm 0.000$ & $0.546 \pm 0.003$ & \multirow{3}{*}{NA} \\
& {LINEX/real}  & $0.090 \pm 0.005$ & $\mathbf{0.279 \pm 0.001}$ & $\mathbf{0.006 \pm 0.000}$ & $\mathbf{0.679 \pm 0.004}$ &  \\
& {NB/real} &  {$0.103 \pm 0.002$} & {$0.398 \pm 0.004$} & {$0.014 \pm 0.000$} & {$0.546 \pm 0.003$} &\\
\cline{2-7}
& {SHAP} & $0.003$ & $1.376$ & $0.398$ & $0.512$ & NA \\
\hline
\hline

\multirow{4}{4em}{\textit{Rotten Tomatoes}}
& {LIME}  & $0.079 \pm 0.036$ & \multirow{4}{*}{NA} & \multirow{4}{*}{NA} &	\multirow{4}{*}{NA} & \multirow{4}{*}{NA} \\
& {S-LIME}  & $0.075 \pm 0.035$ &  &  &	&  \\
& {LINEX/rand} & $0.069 \pm 0.032$ &  &	 &	& \\
& {{NB/rand}} & {$0.241 \pm 0.007$} &  &	 &	& \\

\cline{2-7}
& {MeLIME}  & $\mathbf{0.029 \pm 0.001}$ & $0.391 \pm 0.000$ & $0.000 \pm 0.000$ & $0.999 \pm 0.000$ & $0.909 \pm 0.000$\\
& {LINEX/real}  & $0.053 \pm 0.000$ & $\mathbf{0.361 \pm 0.000}$ & $0.000 \pm 0.000$ & $1.000 \pm 0.000$ & $\mathbf{0.953 \pm 0.001}$ \\
& {NB/real}  & {$0.035 \pm 0.000$} & {$0.535 \pm 0.000$} & $0.000 \pm 0.000$ & $0.999 \pm 0.000$ & $0.909 \pm 0.000$\\

\cline{2-7}
& {SHAP} & $0.000$ & $0.384$ & $0.008$ & $0.999$ & $0.015$   \\
\hline

\end{tabular}
\label{tab:results2}
\vspace{-0.5cm}
\end{table}

\section{Error Analysis of LINEX}
\label{app:err_analysis}

We perform error analysis for LINEX to gain better understanding about the method. We choose FMNIST dataset for doing this since, LINEX/real under performs MeLIME in terms of the INFD measure here (see Table \ref{tab:results}) more heavily compared to other datasets and so we wanted to investigate the reasons for this. This also happens to be one of the higher dimensional datasets that is intuitive to visualize and understand.

We start by observing that even though LINEX/real underperforms in the INFD metric, the gap is not so great in the GI metric, which suggests that MeLIME may be overfitting explanations here. We also note that in terms of CI, $\Upsilon$, and CAC metrics, LINEX/real clearly outperforms MeLIME.

We now choose a sample of images from the dataset where LINEX/real has highest instance-level infidelity numbers and display them in Figure \ref{fig:error_analysis_linex}. Just looking at the explanations and the corresponding original images visually, it is evident that LINEX/real highlights the prominent features like sleeves and collar in a shirt, handles of the bags, outlines of the boots/shoes, even though the infidelity values are high. However, MeLIME misses out on some of these prominent features and focuses only on optimizing the local fit. The fact that LINEX zeroes in on important features also provides additional evidence for the closeness of GI metrics between the two methods, and the better performance of LINEX/real with CI, $\Upsilon$, and CAC metrics.

This conclusion is also verified when we look at the performance of LINEX at a class level. In Figure \ref{fig:error_analysis_agg}, we see two classes one where the infidelity of LINEX is low (i.e. Trousers class) and the other where its infidelity is high (i.e Shirt class). As can be seen since the Trousers class has examples with less superfluous features (viz. varied designs) focusing on which might reduce infidelity but are not critical for determination of the class, LINEX does better in terms of infidelity on the prior. However, although infidelity is higher for the latter Shirt class it does much better on other metrics such as GI, CAC, CI and $\Upsilon$ indicating that LINEX truly focuses on robust features.

\begin{figure}[t]
\centering
\includegraphics[width=0.65\textwidth]{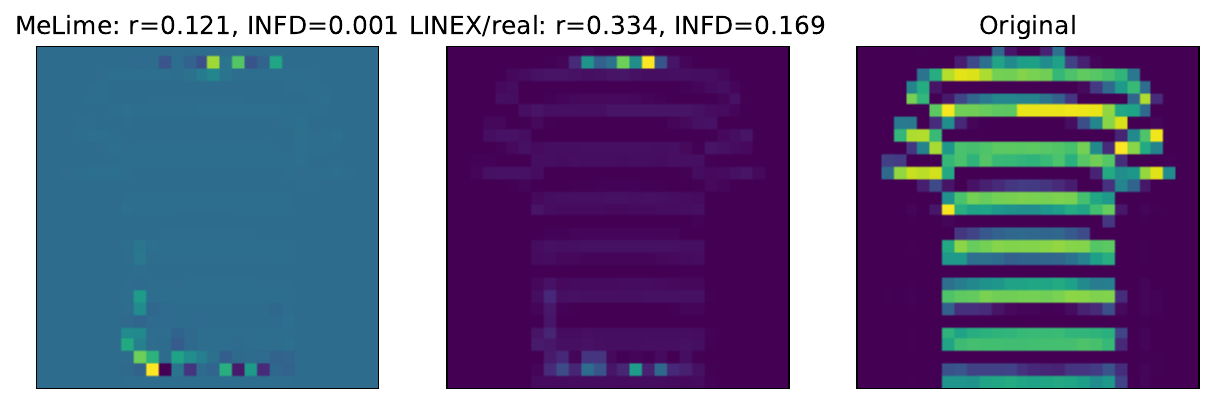}
\includegraphics[width=0.65\textwidth]{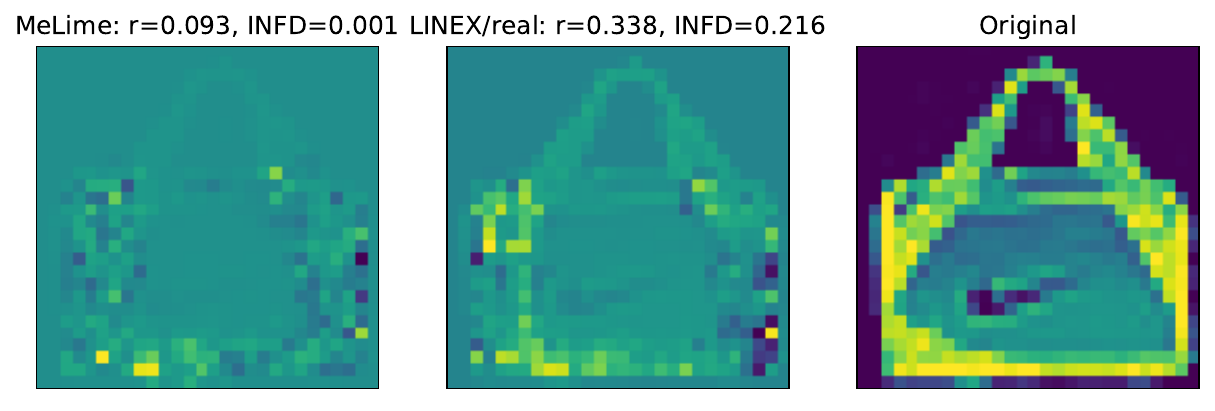}
\includegraphics[width=0.65\textwidth]{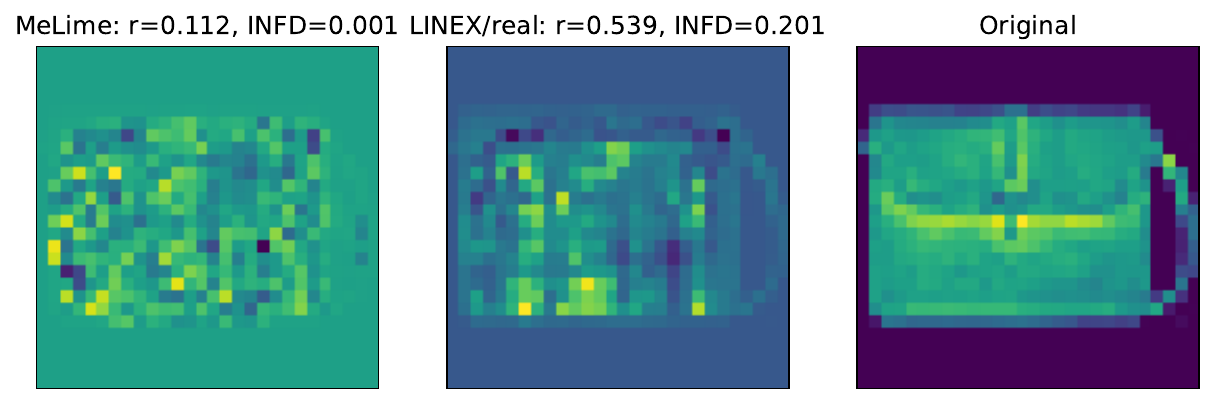}
\includegraphics[width=0.65\textwidth]{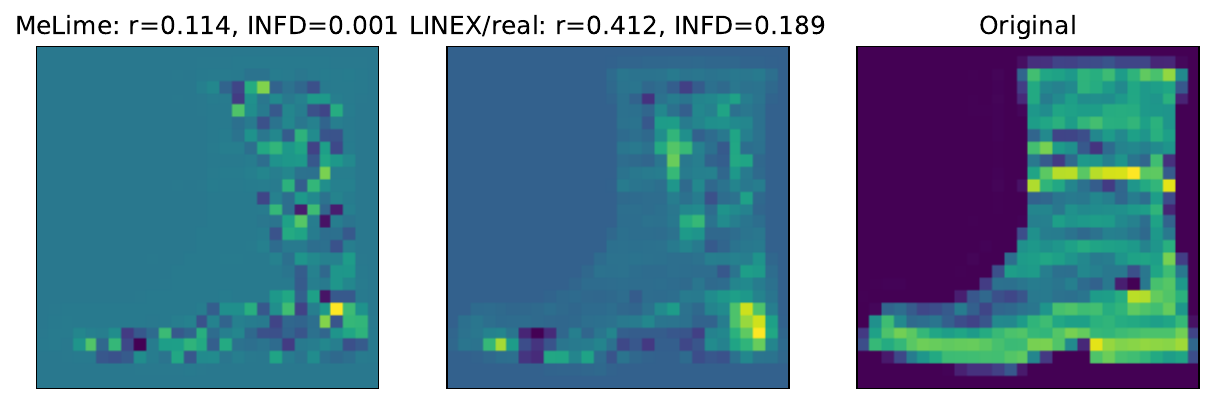}
\includegraphics[width=0.65\textwidth]{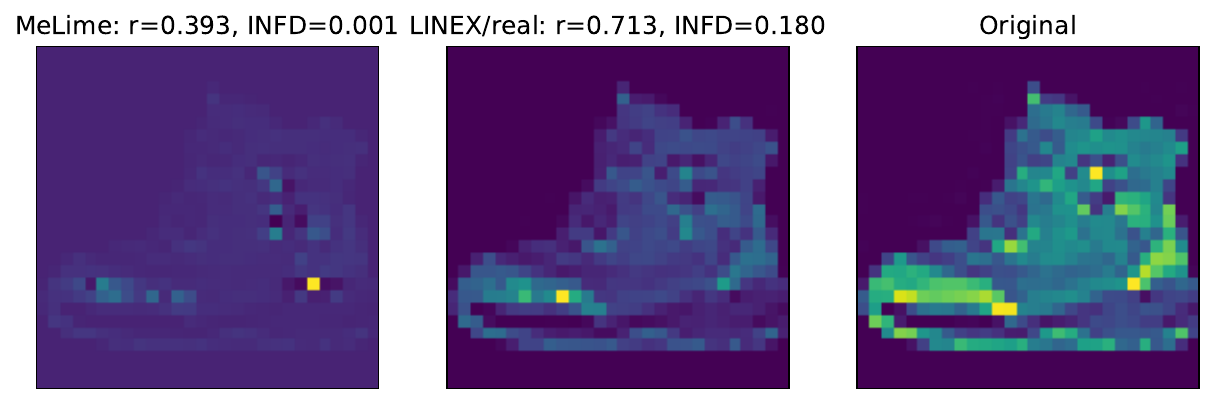}
\includegraphics[width=0.65\textwidth]{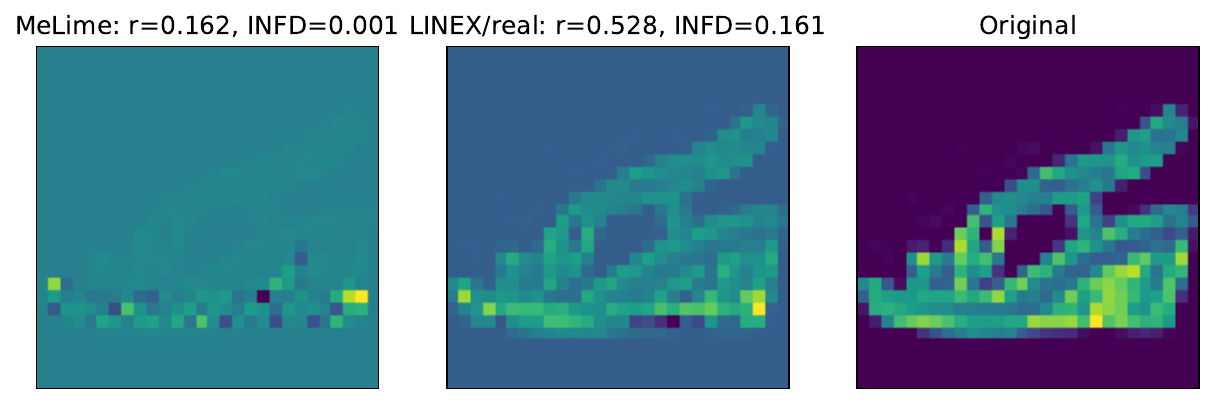}
\caption{Error analysis for a chosen set of examples in FMNIST using MeLIME and LINEX/real methods. The three columns are the MeLIME feature attributions, LINEX/real feature attributions, and the original images. The rows correspond to different examples. We show the Pearson's correlation coefficient between feature attributions and mean of the original images from the respective classes ($r$) and instance-level infidelity (INFD) measures. LINEX seems to highlight important features like stripes in the t-shirt, handles of the bags, outlines of the boots/shoes more prominently, while MeLIME seems to overfit to the data while missing out on highlighting some key features prominently.}
\label{fig:error_analysis_linex}
\end{figure}
\begin{figure}[t]
\centering
\includegraphics[width=\textwidth]{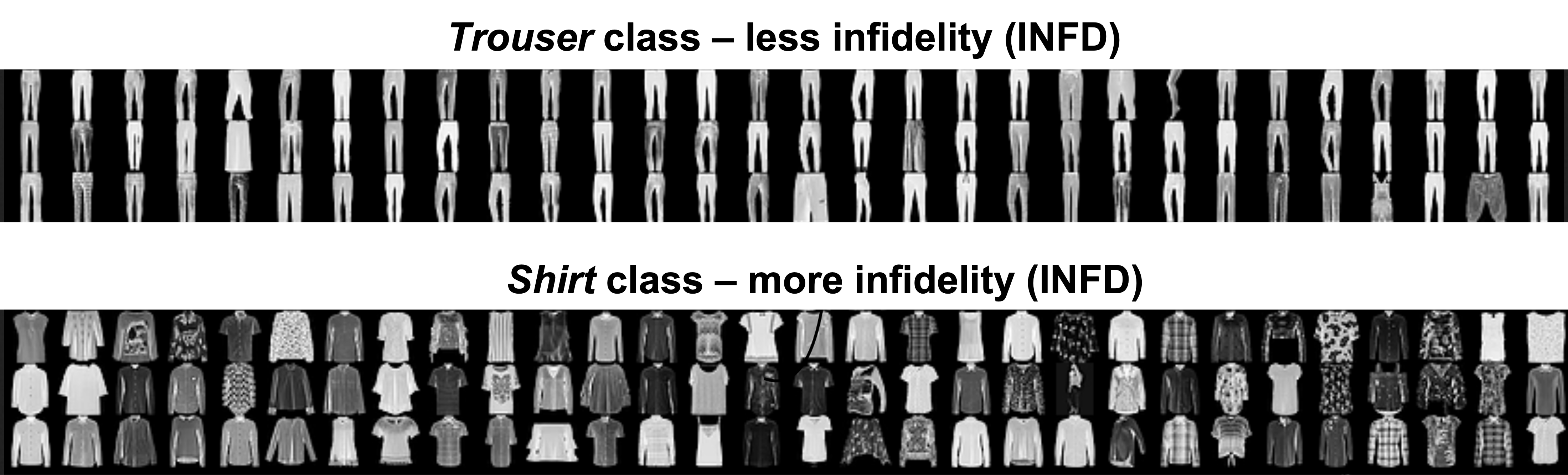}
\caption{We see above that infidelity is lower for Trousers class for LINEX as compared with the Shirts class. A reason for this is that the trousers are more plain with less superfluous features such as the different designs in shirts. Since LINEX focuses on robust features  focusing excessively on the designs is not critical for it to determine a shirt, albeit focusing on these designs might reduce infidelity. Advantage of it relying on robust features is however apparent when we look at other metrics such GI, CAC, CI and $\Upsilon$ as seen in Table \ref{tab:results} where it is much closer to or superior to MeLIME. }
\label{fig:error_analysis_agg}
\end{figure}


\section{Ablation Analysis of Important Features for Various Explanation Methods}
\label{app:ablation_analysis_summary}
We wanted to analyze the most challenging case for us in the reported experiments which is on the FMNIST dataset where we are more worse than MeLIME in terms of INFD than any of the other setups. We thus assess if the features deemed important - those with the largest coefficients - by the explanation methods are indeed important for the black box model to make their predictions. To assess this, we set the we set a fraction of features (pixel values) corresponding to the top coefficients of MeLIME and LINEX/realistic to a baseline value and run the modified images again through the black box model - this is what we mean by ablation here. The baseline value here was chosen to be -1 since that is the  value of the background pixels. We then used two measures to assess the quality of explanations - higher values being better for both. The first measure is mean absolute error between the predicted scores before and after ablation, corresponding to the original predicted class. The second measure is the fraction of images that changed their predicted class after ablation. We see from Figure \ref{fig:ablation_analysis_fmnist} that LINEX/realistic substantially outperforms MeLIME in both these measures, clearly demonstrating the relevance of features chosen by our method to the black box.

\begin{figure}[t]
\centering
\includegraphics[width=0.49\textwidth]{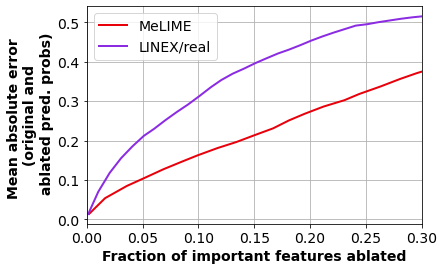}
\includegraphics[width=0.49\textwidth]{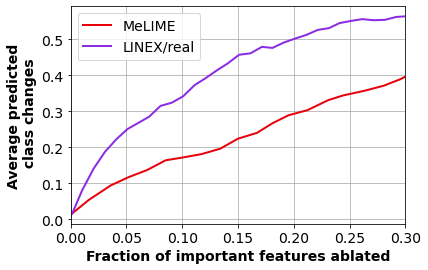}
\caption{Ablation analysis to determine if the features deemed important by the explanation methods are actually considered important for prediction by the black box model. We see that features chosen by LINEX impact the prediction of the black box model much more than those chosen by MeLIME. This is true with respect to both MAE measure (left) between the predicted probabilities before and after ablation for winning (or argmax) class, and the change in predicted classes (right) before and after ablation. Higher values here mean that the features chosen by the explanations are more relevant for the black box to make its predictions. The maximum value of both measures is 1.0.}
\label{fig:ablation_analysis_fmnist}
\end{figure}

\section{Error Analysis of LINEX based on Ablation}
\label{app:err_analysis_ablation}

Highlighting stable features for examples near non-linearities is a key strength of LINEX. However, in some cases for examples near class boundaries it may ignore sensitive features as we show in this demonstration.

In Figure \ref{fig:error_analysis_linex_ablation}, we show 6 examples that are appear to be close to class boundaries. We ablate pixels corresponding to top $15\%$ of important features chosen by MeLIME and LINEX/realistic using the approach discussed in Section \ref{app:ablation_analysis_summary}. Ablation based on MeLIME importances meaningfully changes classes, whereas ablation by LINEX importances does not. The changes in prediction for MeLIME ablation for the six images are respectively from \textit{Dress} to \textit{Trouser}, \textit{Sneaker} to \textit{Sandal}, \textit{Pullover} to \textit{Dress}, \textit{Sneaker} to \textit{Sandal}, \textit{Bag} to \textit{Pullover}, and \textit{Sneaker} to \textit{Sandal}. The new class assignment looks reasonable looking at the ablated images. We also see that the changes in class probabilities for the original class ($p$) are much higher after MeLIME ablation compared to LINEX/realistic ablation.

MeLIME ablated images for the first example has structures that look like trouser legs,  for the second, fourth and sixth examples the area around the heel is more open making the original sneaker look like a sandal, for the third example, there is a hole in the hooded part of the pullover making it resemble a dress. The fifth example is classified as a pullover possibly because of the elongated structures on the sides that look like hands. 

Note that such cases of LINEX under performing are rare though as is confirmed by its superior performance in Figure \ref{fig:ablation_analysis_fmnist}.

\begin{figure}[t]
\centering
\includegraphics[width=\textwidth]{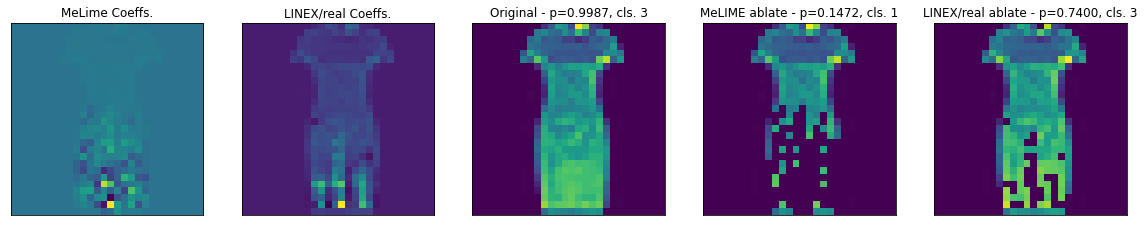}
\includegraphics[width=\textwidth]{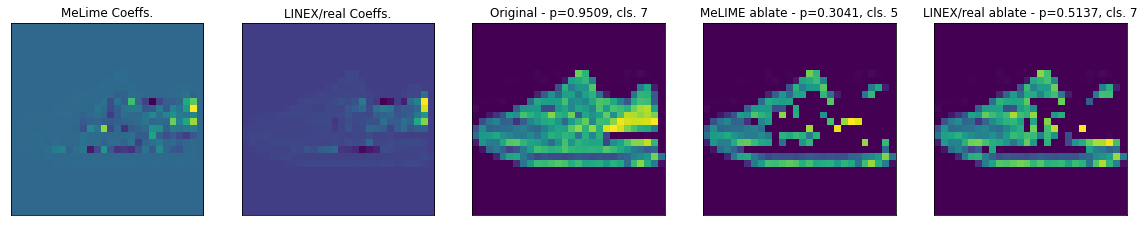}
\includegraphics[width=\textwidth]{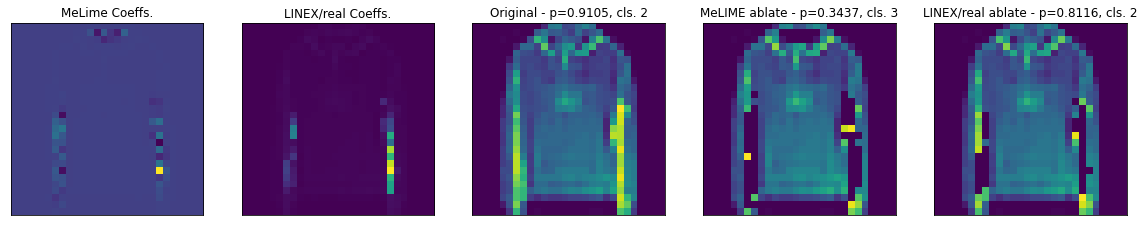}
\includegraphics[width=\textwidth]{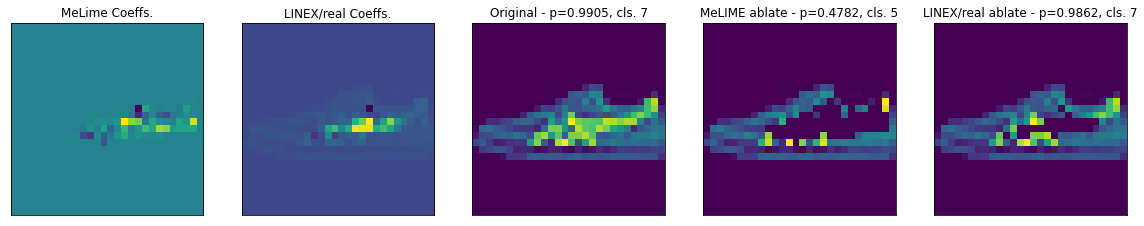}
\includegraphics[width=\textwidth]{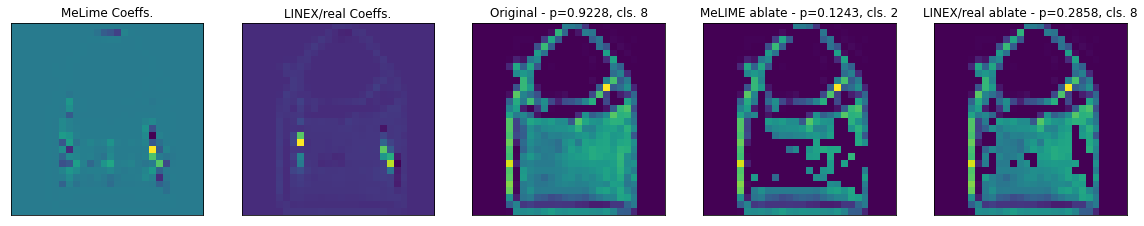}
\includegraphics[width=\textwidth]{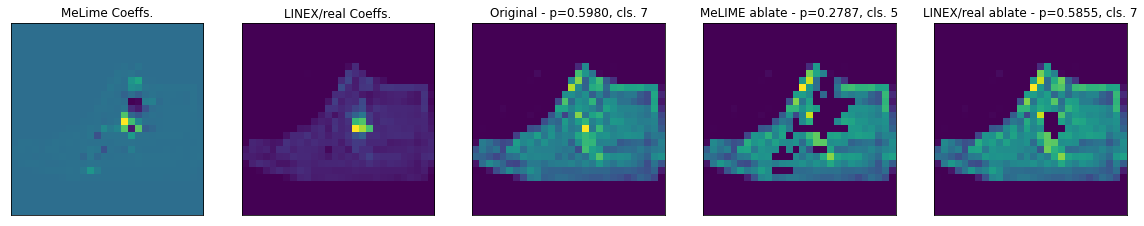}
\caption{
Error analysis for a chosen set of examples in FMNIST using MeLIME and LINEX/realistic methods, using ablation of important features. Each row shows results for a particular image. The columns show the: (a) MeLIME coefficients, (b) LINEX/realistic coefficients, (c) the original image along with its predicted class (cls.) and predicted probability for that class ($p$), (d) the image after MeLIME ablation along with the predicted probability for the original class ($p$) and the new class prediction (cls.), and (e) the image after LINEX/realistic ablation along with the predicted probability for the original class ($p$) and the new class prediction (cls.). The changes in prediction for MeLIME ablation for the six images are respectively from \textit{Dress} to \textit{Trouser}, \textit{Sneaker} to \textit{Sandal}, \textit{Pullover} to \textit{Dress}, \textit{Sneaker} to \textit{Sandal}, \textit{Bag} to \textit{Pullover}, and \textit{Sneaker} to \textit{Sandal}. No changes in classes are seen for LINEX ablation.
}
\label{fig:error_analysis_linex_ablation}
\end{figure}

\section{Understanding Behavior of LIME and LINEX with Synthetic Data}
\label{app:synthetic_data_expt}

We consider explaining the behavior of a function of two variables $x$ and $y$ with Class 1 sandwiched between Class 0 (see Figure \ref{fig:synthetic_data_expt}). The third (or vertical) axis denotes the probability of being in Class 1. Clearly, $x$ is the only important feature here that determines the class label. 

From Figure \ref{fig:synthetic_data_expt} (left), we see that the LIME (here MeLIME would be the same as LIME since the space is flat and all points are realistic) feature attributions at points $a$, $b$, and $c$ will provide importance to $x$ feature for small as well as large kernel width (1 and 2 respectively) neighborhoods. For point $c$, in the interior of the Class 0, the attributions are stable across kernel widths. However for points $a$ and $b$ close to the boundary of classes, the attributions for small kernel width and large kernel width neighborhoods differ significantly along the $x$ direction. This shows the instability of LIME explanations near boundaries of classes for different kernel widths.

In contrast in Figure \ref{fig:synthetic_data_expt} (right), we see that the LINEX explanation constructed for the two kernel widths provides stable feature attributions for all points $a$, $b$, $c$. For $a$ and $b$, LINEX will conservatively pick a smaller feature attribution along the $x$ direction since the function changes rapidly in its neighborhood. As such though LINEX will still pick the feature in the $x$ direction in this scenario.

\begin{figure}[t]
\centering
\includegraphics[width=0.49\textwidth]{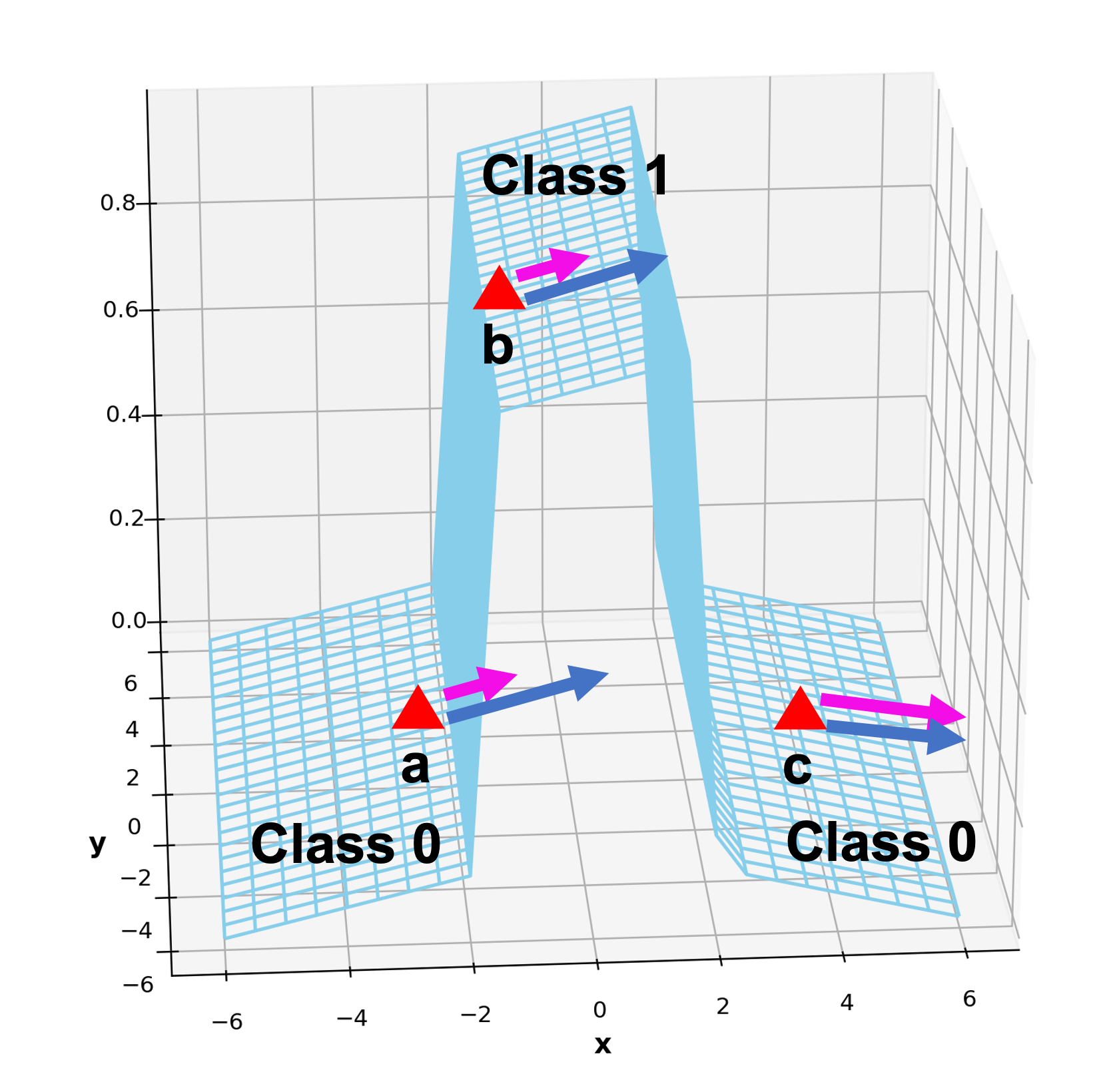}
\includegraphics[width=0.49\textwidth]{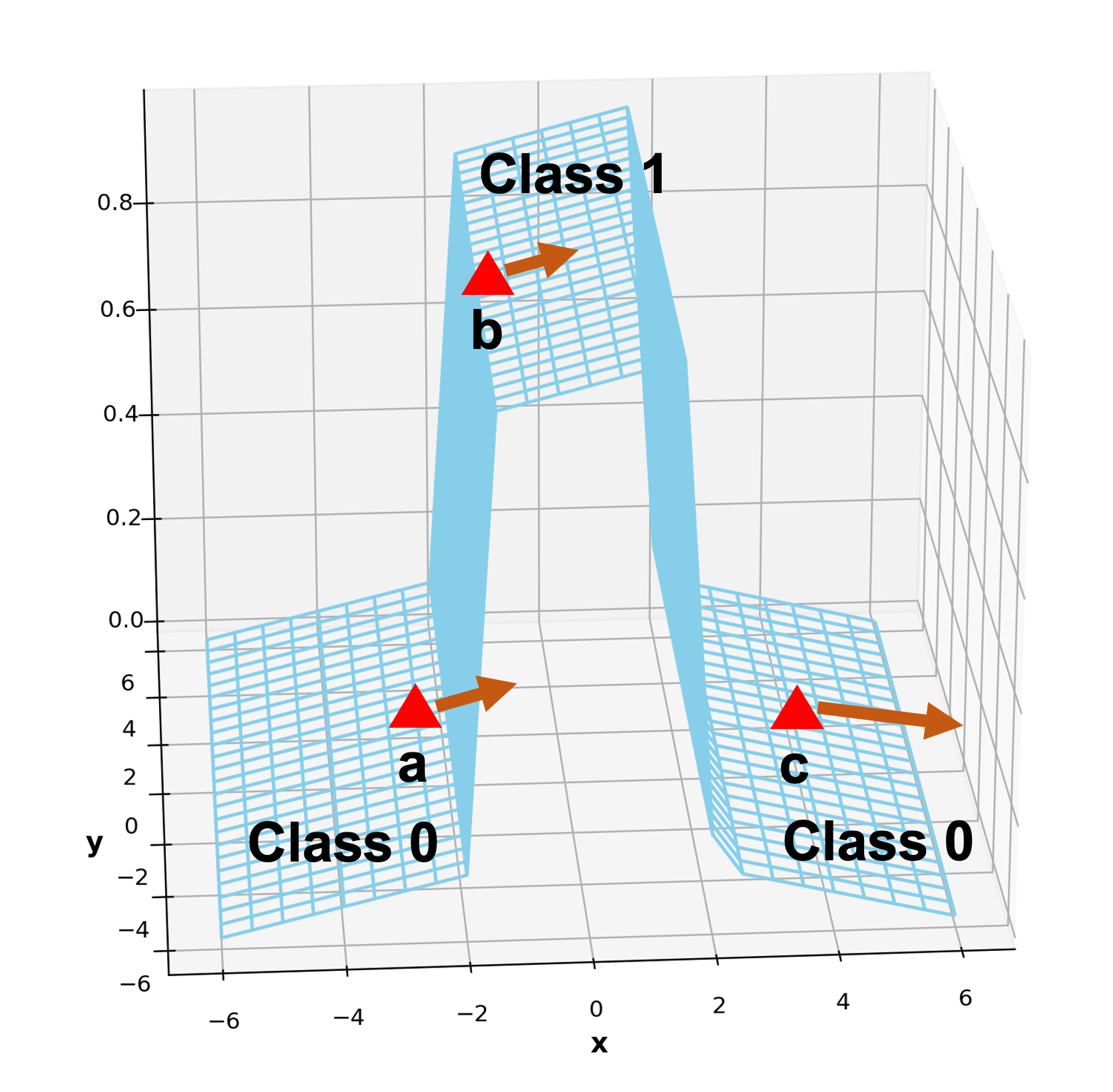}
\caption{
LIME (left) and LINEX (right) feature attributions for three points ($a$, $b$, $c$) for a synthetic data where we have Class 1 sandwiched between Class 0. For LIME, the different colors pink and blue correspond to feature attributions obtained with the small and large kernel width neighborhoods. Note how explanations for LIME change significantly (in magnitude) by kernel widths near the class boundaries, whereas the LINEX explanation remains stable, where it still picks up the important feature.
}
\label{fig:synthetic_data_expt}
\end{figure}

{\section{Variation of feature attributions with $\gamma$}}
\label{sec:feat_variation_gamma}

{Based on the proof of Theorem \ref{thm1}, if for a feature the optimal attributions have opposite sign for each of the two environments, then $\gamma$ can be made arbitrarily small (except 0) or large and the output of Algorithm \ref{algo:LINEX} should still be the same which is $0$ as the Nash Equilibrium is $\pm \gamma$. If the optimal attributions are the same sign then we should still get the same output from Algorithm \ref{algo:LINEX} as long as $\gamma \geq \min(|w_{1i}|, |w_{2i}|)$ since the attribution from our algorithm is the minimum of those values. When $\gamma < \min(|w_{1i} |, |w_{2i} |)$ then the feature attributions will smoothly reduce as $\gamma$ reduces.}

{We demonstrate this behavior in Figure \ref{fig:feat_variation_gamma} using an example from the IRIS dataset with random perturbations using the same setting as in Section \ref{sec:exp}. In the experiments in Section \ref{sec:exp}, we set $\gamma = 0.329$ which is the maximum absolute value based on a linear fit to each environment. As $\gamma$ increases beyond $0.329$, the attributions are unchanged demonstrating robustness. Same holds true while reducing $\gamma$ up to $0.165$ beyond which we see smooth reduction in the attribution values. Qualitatively, similar behavior is seen for other examples too. Because we set $\gamma$ pessimistically (ignoring constraints) to a high value, we can expect our reported performances in the paper to be robust across many values of $\gamma$.}

\begin{figure}[t]
\centering
\includegraphics[width=0.75\textwidth]{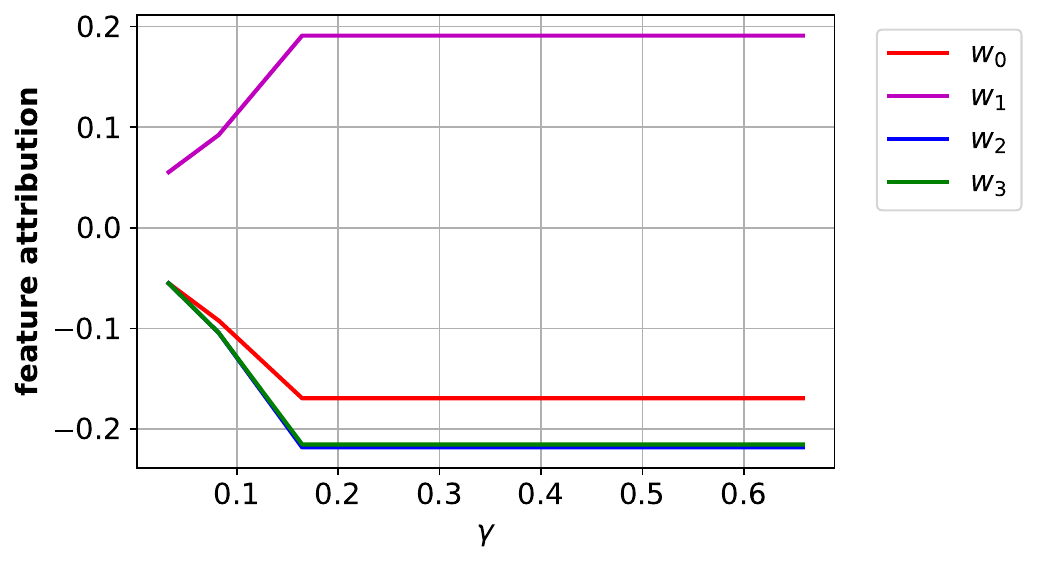}
\caption{{Feature attributions for the four features for an example in the IRIS dataset are shown above when varying $\gamma$. We used the same setting as in Section \ref{sec:exp} for this experiment. The attributions increase smoothly as $\gamma$ increases and stay constant after $\gamma \geq \min(|w_{1i}|, |w_{2i}|) \forall i$.}}
\label{fig:feat_variation_gamma}
\end{figure}

\section{{Convergence of LINEX procedure and comparisons}}
\label{sec:linex_conv}

{We demonstrate based on a synthetic example how Algorithm $\ref{algo:LINEX}$ and provides a  unidirectional explanation. We generate synthetic data using a function in $\mathbb{R}^2$ (Figure \ref{fig:linex_conv}(left)). The function gently rises with increasing $y$ values, and along $x$ it is flat first, then rises abruptly and then falls gradually. We want to obtain robust attributions of this function at the point $x=1.0, y=0.0$, which is close to the end of the rising edge along $x$ direction.}

{As we can imagine, since the slope changes abruptly along $x$ direction near the point, it should be ideally excluded from an explanation intended towards recourse based on a linear proxy. Otherwise, the explanation will not generalize in the neighborhood of this point. On the other hand, the $y$ direction should be included since the function changes smoothly along $y$ throughout.}


{To generate explanations We first create two environments centered at the example to explain with variances $0.5$ and $2.0$. Now independently fitting to these environments leads to feature attributions that are $\{-0.033, 0.098\}$ and $\{0.084, 0.102\}$. Appending the two environments the attributions are $\{0.029, 0.095\}$, whereas with LINEX, the attributions would be $\{0.0, 0.093\}$. Thus, LINEX effectively eliminates the feature with high variability or abrupt changes. The behavior of the coefficients for each environment as LINEX converges is shown in Figure \ref{fig:linex_conv}(right). As such, one can also see the convergence is fast.}


\begin{figure}[t]
\centering
\includegraphics[width=0.33\textwidth]{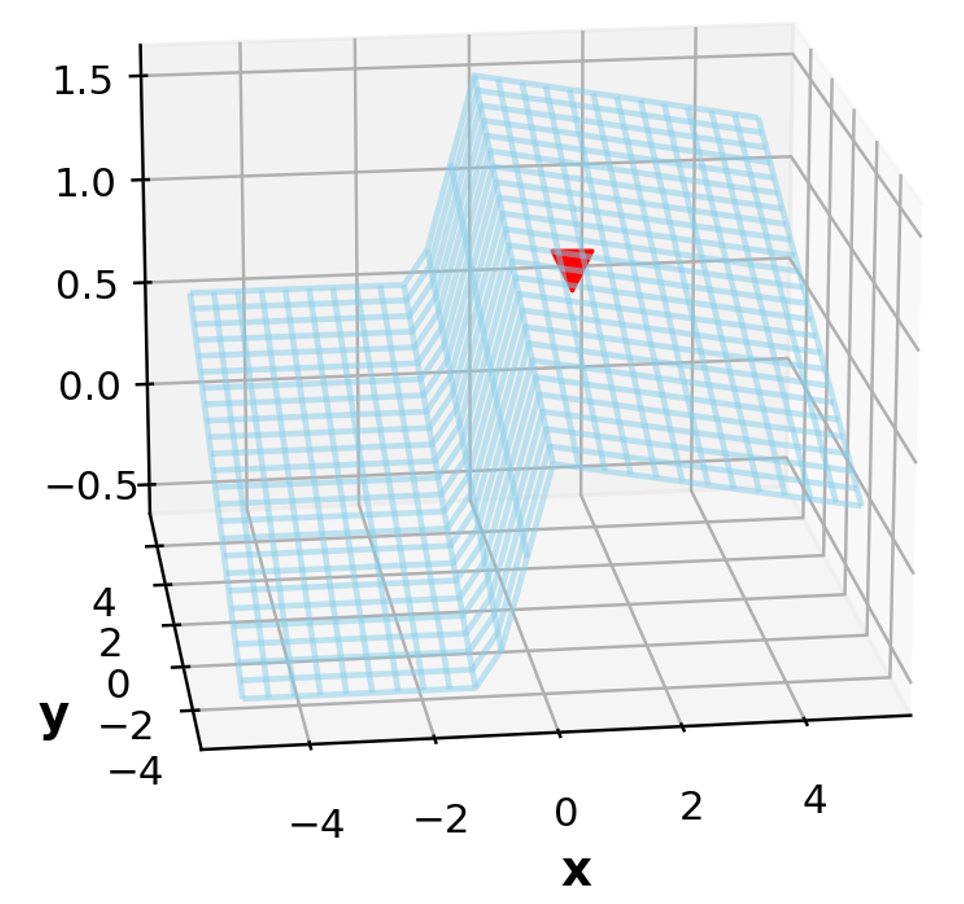}
\includegraphics[width=0.63\textwidth]{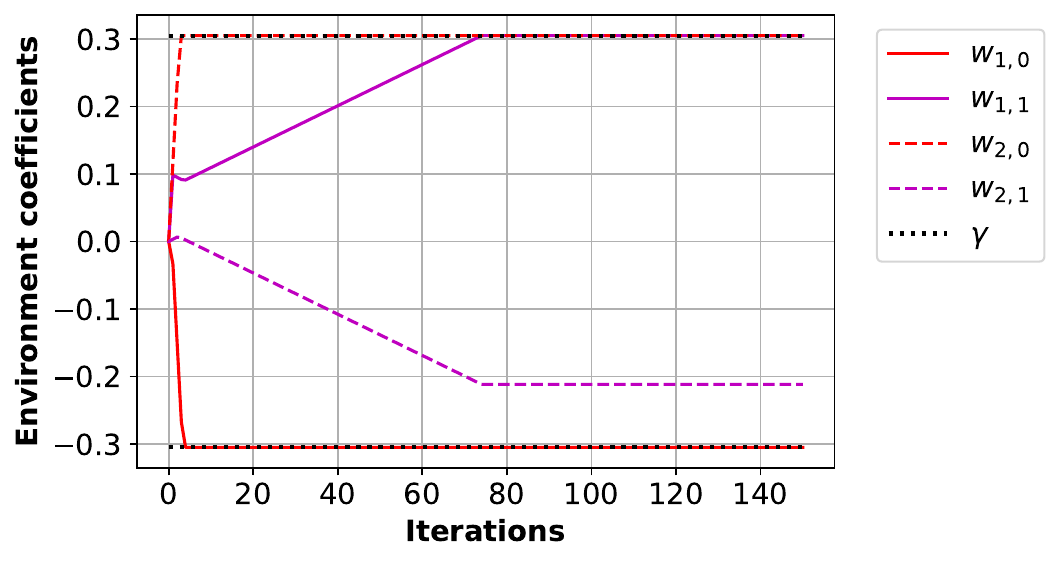}
\caption{
{Left side: Explaining a scalar function in $\mathbb{R}^2$ at the point indicated by the triangle. The point is centered at $x=1.0, y=0.0$. The two environments are created by sampling multivariate normals with variances $0.5$ and $2.0$ respectively (samples not shown) centered at this point. Right side: Convergence of individual environment attributions. The attributions for first feature ($x$), $w_{1,0}$ and $w_{2,0}$, converge to $\gamma$ and $-\gamma$ leading to the optimal attribution of $0$. For the second feature ($y$) the optimal attribution ($w_{1,1} + w_{2,1}$) converges to a positive value.}
}
\label{fig:linex_conv}
\end{figure}

\section{Limitations}

Like any other posthoc explainable AI method there is no way to surely say that LINEX exactly reflects the true reasoning behind a black box classifier in arbitrary applications. It also is somewhat slower than LIME as shown in section A given the game theoretic nature of the algorithm, where its stability and unidirectionality hopefully offsets the additional time required. On the flip side, given its favorable properties in terms of recovering explanations it could be used to violate privacy which may be concerning from a social standpoint.

\newpage

\begin{figure}[!htb]
\centering
    \begin{subfigure}[b]{0.32\textwidth}
        \vspace{0pt}
        \centering
        \captionsetup{justification=centering}
        \includegraphics[width=\textwidth]{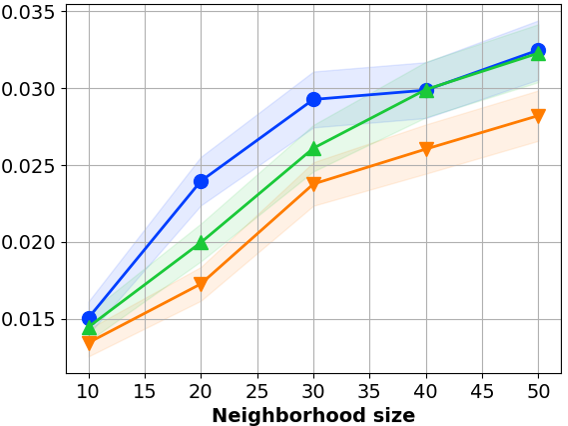}
        \caption{INFD - S-LIME - mean smoothing}
        \label{fig:IRIS_INFD_cnt_lime_smooth_mean}
    \end{subfigure}
    \begin{subfigure}[b]{0.32\textwidth}
        \vspace{0pt}
        \centering
        \captionsetup{justification=centering}
        \includegraphics[width=\textwidth]{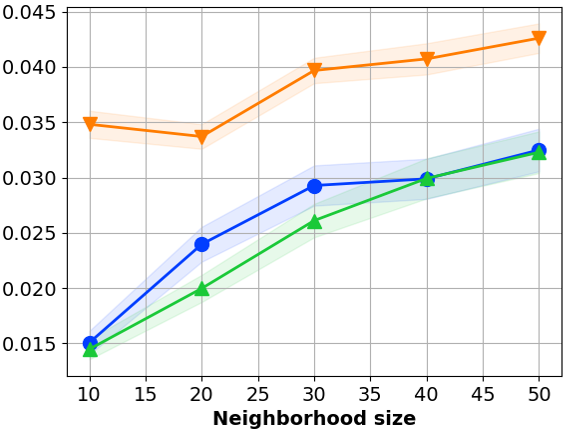}
        \caption{INFD - S-LIME - median smoothing}
        \label{fig:IRIS_INFD_cnt_lime_smooth_median}
    \end{subfigure}
    \begin{subfigure}[b]{0.32\textwidth}
        \vspace{0pt}
        \centering
        \captionsetup{justification=centering}
        \includegraphics[width=\textwidth]{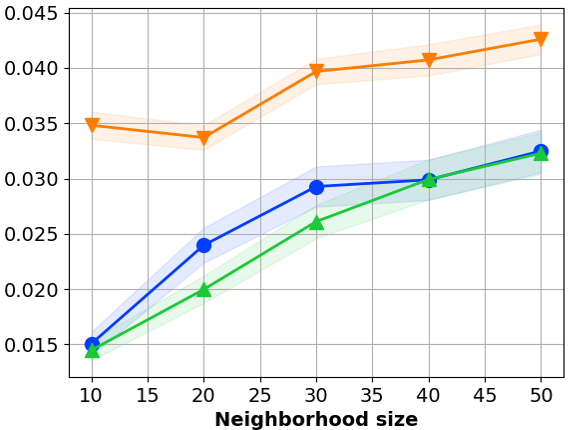}
        \caption{INFD - S-LIME - MoM smoothing}
        \label{fig:IRIS_INFD_cnt_lime_smooth_mom}
    \end{subfigure}

    \begin{subfigure}[b]{0.32\textwidth}
        \vspace{0pt}
        \centering
        \captionsetup{justification=centering}
        \includegraphics[width=\textwidth]{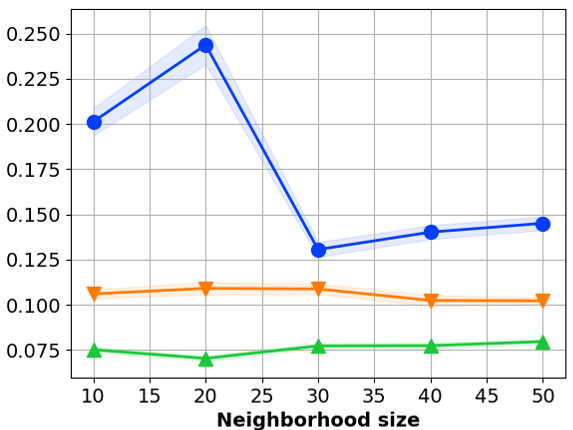}
        \caption{GI - S-LIME - mean smoothing}
        \label{fig:IRIS_GI_cnt_lime_smooth_mean}
    \end{subfigure}
    \begin{subfigure}[b]{0.32\textwidth}
        \vspace{0pt}
        \centering
        \captionsetup{justification=centering}
        \includegraphics[width=\textwidth]{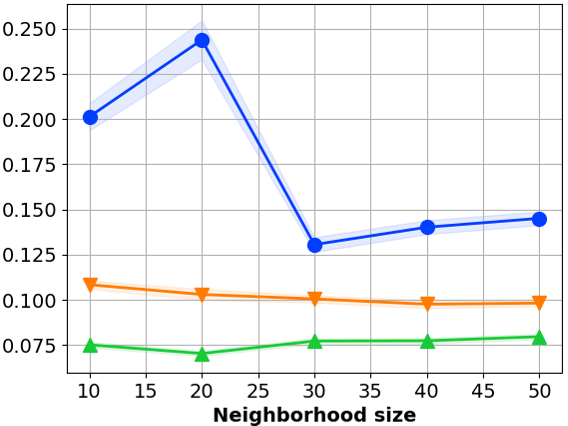}
        \caption{GI - S-LIME - median smoothing}
        \label{fig:IRIS_GI_cnt_lime_smooth_median}
    \end{subfigure}
    \begin{subfigure}[b]{0.32\textwidth}
        \vspace{0pt}
        \centering
        \captionsetup{justification=centering}
        \includegraphics[width=\textwidth]{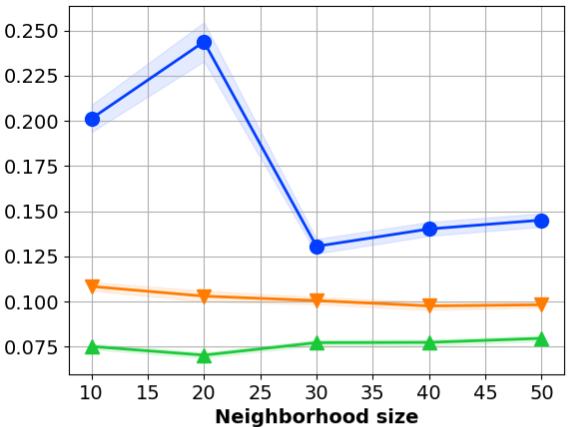}
        \caption{GI - S-LIME - MoM smoothing}
        \label{fig:IRIS_GI_cnt_lime_smooth_mom}
    \end{subfigure}

    \begin{subfigure}[b]{0.32\textwidth}
        \vspace{0pt}
        \centering
        \captionsetup{justification=centering}
        \includegraphics[width=\textwidth]{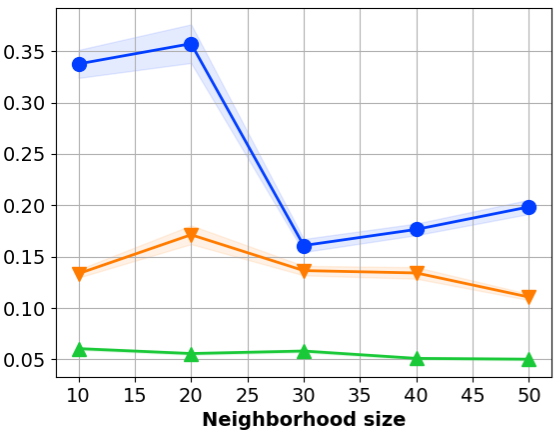}
        \caption{CI - S-LIME - mean smoothing}
        \label{fig:IRIS_CI_cnt_lime_smooth_mean}
    \end{subfigure}
    \begin{subfigure}[b]{0.32\textwidth}
        \vspace{0pt}
        \centering
        \captionsetup{justification=centering}
        \includegraphics[width=\textwidth]{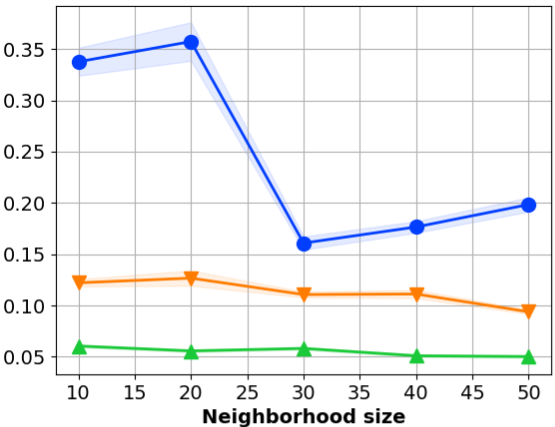}
        \caption{CI - S-LIME - median smoothing}
        \label{fig:IRIS_CI_cnt_lime_smooth_median}
    \end{subfigure}
    \begin{subfigure}[b]{0.32\textwidth}
        \vspace{0pt}
        \centering
        \captionsetup{justification=centering}
        \includegraphics[width=\textwidth]{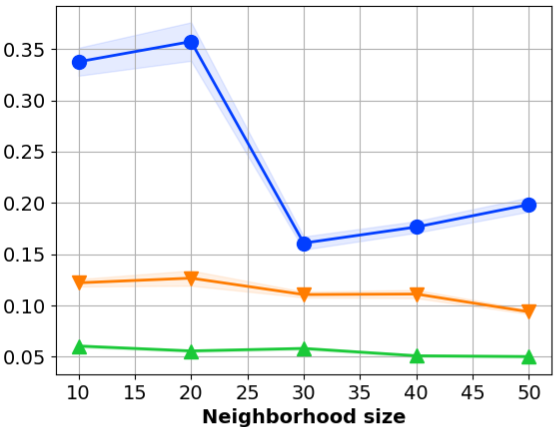}
        \caption{CI - S-LIME - MoM smoothing}
        \label{fig:IRIS_CI_cnt_lime_smooth_mom}
    \end{subfigure}

    \begin{subfigure}[b]{0.32\textwidth}
        \vspace{0pt}
        \centering
        \captionsetup{justification=centering}
        \includegraphics[width=\textwidth]{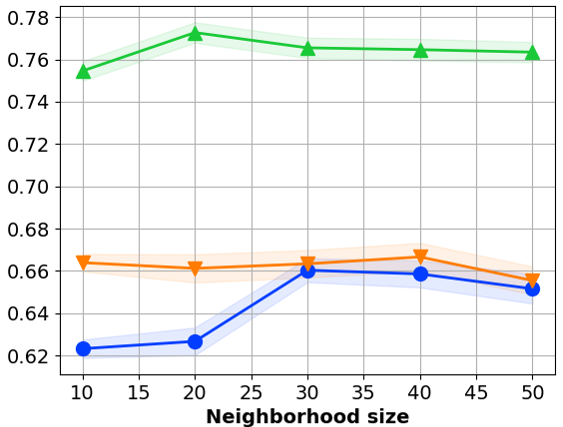}
        \caption{$\Upsilon$ - S-LIME - mean smoothing}
        \label{fig:IRIS_Upsilon_cnt_lime_smooth_mean}
    \end{subfigure}
    \begin{subfigure}[b]{0.32\textwidth}
        \vspace{0pt}
        \centering
        \captionsetup{justification=centering}
        \includegraphics[width=\textwidth]{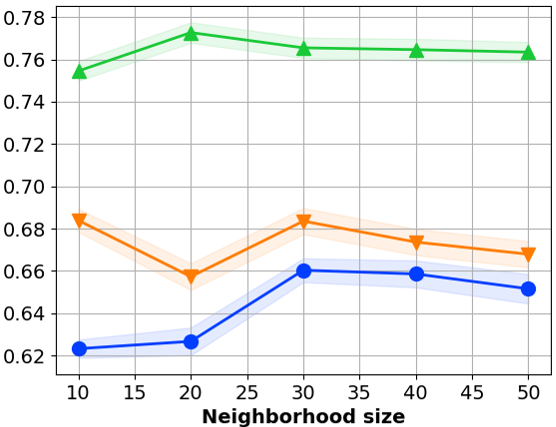}
        \caption{$\Upsilon$ - S-LIME - median smoothing}
        \label{fig:IRIS_Upsilon_cnt_lime_smooth_median}
    \end{subfigure}
    \begin{subfigure}[b]{0.32\textwidth}
        \vspace{0pt}
        \centering
        \captionsetup{justification=centering}
        \includegraphics[width=\textwidth]{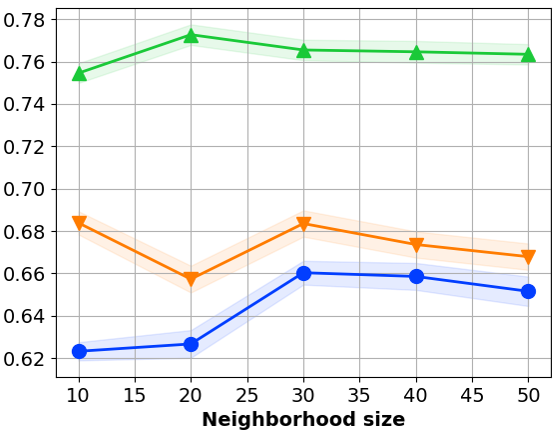}
        \caption{$\Upsilon$ - S-LIME - MoM smoothing}
        \label{fig:IRIS_Upsilon_cnt_lime_smooth_mom}
    \end{subfigure}

    \begin{subfigure}[b]{0.32\textwidth}
        \vspace{0pt}
        \centering
        \captionsetup{justification=centering}
        \includegraphics[width=\textwidth]{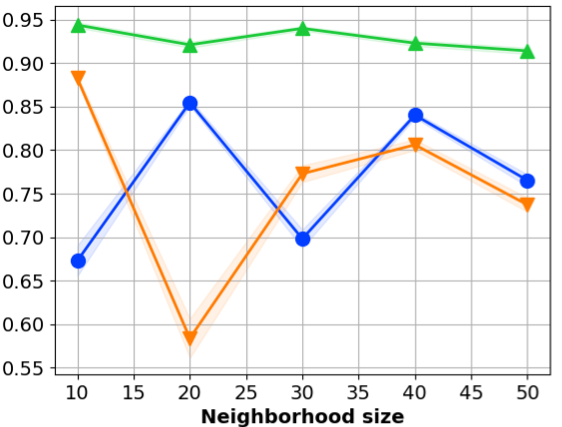}
        \caption{CAC - S-LIME - mean smoothing}
        \label{fig:IRIS_CAC_cnt_lime_smooth_mean}
    \end{subfigure}
    \begin{subfigure}[b]{0.32\textwidth}
        \vspace{0pt}
        \centering
        \captionsetup{justification=centering}
        \includegraphics[width=\textwidth]{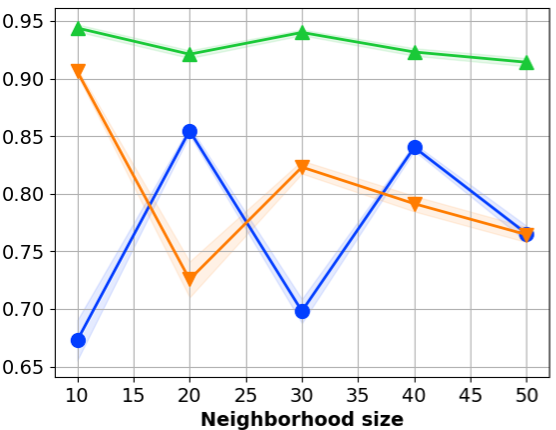}
        \caption{CAC - S-LIME - median smoothing}
        \label{fig:IRIS_CAC_cnt_lime_smooth_median}
    \end{subfigure}
    \begin{subfigure}[b]{0.32\textwidth}
        \vspace{0pt}
        \centering
        \captionsetup{justification=centering}
        \includegraphics[width=\textwidth]{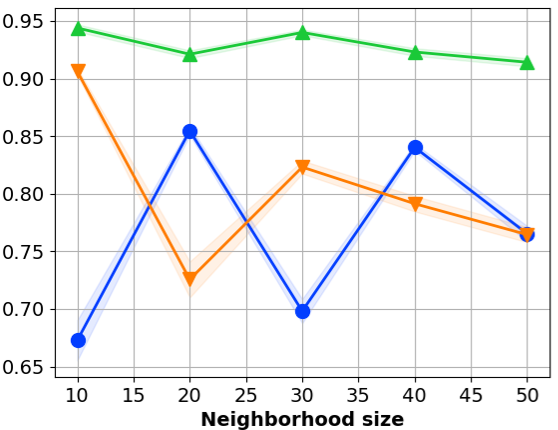}
        \caption{CAC - S-LIME - MoM smoothing}
        \label{fig:IRIS_CAC_cnt_lime_smooth_mom}
    \end{subfigure}

\caption{Effect of various smoothing schemes on S-LIME's performance based on 5 environments (since median-of-means (MoM) is just median for 2) with the same setup described in Suppl. D. for the IRIS dataset. As can be seen median and MoM perform worse than the mean on INFD and similar to it on other metrics. Thus, this does not change the takeaways from the main paper. See Figure 4 in the supplement for legend.}
\vspace{-0.5cm} 
\label{fig:IRIS_mean_median_mom_smoothing}
\end{figure}

\end{document}